%% file: main.tex
\documentclass{article}
\pdfoutput=1 %

\usepackage{iclr2023_conference}

\usepackage[american]{babel}
\usepackage[utf8]{inputenc}
\usepackage[T1]{fontenc} %
\usepackage{textcomp} %
\usepackage{amsmath}
\usepackage{amssymb}
\usepackage{mathrsfs}
\usepackage{amsthm}
\usepackage{cancel}
\usepackage{amsfonts}
\usepackage{refcount}
\usepackage{mathrsfs}
\usepackage{mathtools}
\usepackage{nicefrac}
\usepackage[colorlinks = true,
            linkcolor = blue,
            urlcolor  = blue,
            citecolor = blue,
            anchorcolor = blue]{hyperref}
\usepackage[capitalise,noabbrev]{cleveref} %
\usepackage{url}
\usepackage{booktabs} %
\usepackage{array} %
\usepackage{microtype} %
\usepackage[font=small]{caption}
\usepackage{subcaption}
\usepackage[pdftex]{graphicx}
\usepackage{macros}

\input{math_commands.tex}

\graphicspath{{images/}}

\title{Strong inductive biases provably prevent \\ harmless interpolation}

\author{Michael Aerni\thanks{Equal contribution;
correspondence to \texttt{research@michaelaerni.com}}\, $^{1}$,
Marco Milanta\footnotemark[1]\, $^{1}$,
Konstantin Donhauser$^{1,2}$,
Fanny Yang$^{1}$ \\
$^{1}$Department of Computer Science, ETH Zurich
$^{2}$ETH AI Center\vspace{-0.1in}
}

\iclrfinalcopy %

\ificlrfinal
\else
\RequirePackage{lineno}
\linenumbers
\fi

\begin{document}

\input{figurelayout.tex}

\maketitle

\input{abstract.tex}

\input{mainmatter/introduction}

\input{mainmatter/related_work}

\input{mainmatter/theory}
\input{mainmatter/empirical}
\input{mainmatter/proof_sketch}

\input{mainmatter/conclusion}

\ificlrfinal
   \input{acknowledgments.tex}
\fi

\bibliography{bibliography.bib}
\bibliographystyle{iclr2023_conference}

\clearpage
\appendix
\input{appendix/context_agnostic_proof}
\input{appendix/hypercube}
\input{appendix/matrix_concentration}
\input{appendix/training_error}
\input{appendix/experiment_details}

\end{document}

%% file: math_commands.tex
\newcommand{\f}{\mathbf{f}}

\renewcommand{\c}{c}

\newcommand{\Yf}{\mathcal{Y}}
\newcommand{\T}{T}
\newcommand{\Pf}{\psi}
\newcommand{\Qf}{\mathcal{Q}}
\newcommand{\Gf}{\mathcal{G}}
\newcommand{\B}{\mathcal{B}}
\newcommand{\C}{\mathcal{C}}
\newcommand{\M}{\mathcal{M}}
\newcommand{\Mm}{\mathbf{M}}
\newcommand{\Am}{\mathbf{A}}

\newcommand{\U}{\mathcal{U}}
\newcommand{\filsize}{q}
\newcommand{\dsaw}{\varpi}
\newcommand{\y}{\mathbf{y}}
\newcommand{\reg}{\mathbf{\lambda}}
\newcommand{\diam}{\gamma}
\newcommand{\I}{\mathbf{I}}
\newcommand{\bias}{\textbf{Bias}}
\newcommand{\risk}{\textbf{Risk}}
\newcommand{\var}{\textbf{Variance}}
\newcommand{\Q}{\mathbf{Q}}

\newcommand{\2}{\{-1,1\}}

\renewcommand{\k}{\mathbf{k}}
\newcommand{\kernel}{\mathcal{K}}
\newcommand{\K}{\mathbf{K}}
\newcommand{\mumax}[1]{\mu_{\max}(#1)}

\newcommand{\freg}{\hat{f}_{\reg}}
\newcommand{\fopt}{\hat{f}_{\regopt}}
\newcommand{\regopt}{\reg_{\text{opt}}}
\newcommand{\mlarge}{{\hat{\m}}}
\newcommand{\Lropt}{{\ell_{\regopt}}}
\newcommand{\Lru}{\bar{\ell}}
\newcommand{\Lrmin}{\ell_{\reg_{\text{min}}}}
\newcommand{\fint}{\hat{f}_0}
\newcommand{\fstar}{f^{\star}}

\newcommand{\customsubscript}[2]{{#2}_{#1}}

\newcommand{\Km}{\customsubscript{\leq \m}{\K}}
\newcommand{\KM}{\customsubscript{> \m}{\K}}

\newcommand{\module}[2]{\mathrm{mod}\left(#1,#2\right)}
\newcommand{\rmult}{r}

\renewcommand{\H}{\mathbf{H}}

\newcommand{\HM}{\customsubscript{> \m}{\H}}

\newcommand{\betath}{\beta^*}
\newcommand{\p}{\mathbf{\Psi}}

\let\mathpm\pm
\renewcommand{\pm}{\customsubscript{\leq \m}{\p}}

\renewcommand{\l}{l}
\renewcommand{\l}{l}
\renewcommand{\l}{l}
\newcommand{\overl}{\overline{\l}}
\newcommand{\underl}{\underline{\l}}

\newcommand{\m}{m}

\newcommand{\tm}{\tilde{\m}}
\newcommand{\ratio}{r}
\newcommand{\tmin}{\tau_1}
\newcommand{\tmax}{\tau_2}
\newcommand{\rmin}{\ratio_1}
\newcommand{\rmax}{\ratio_2}
\newcommand{\elem}{\varpi}

\newcommand{\term}{E}
\newcommand{\tterm}{E}

\renewcommand{\i}{k}
\newcommand{\ii}{l}
\newcommand{\iii}{i}

\renewcommand{\L}{L}
\renewcommand{\d}{\delta}
\newcommand{\Ln}{{\ell}}
\newcommand{\Ls}{{\ell_{\sigma}}}
\newcommand{\Lr}{{\ell_{\reg}}}
\newcommand{\Lgt}{{\L^*}}

\newcommand{\bL}{\bar{\L}}
\newcommand{\bd}{\bar{\delta}}
\newcommand{\bm}{\bar{\m}}

\newcommand{\Pm}{\customsubscript{\leq \m}{\P}}

\renewcommand{\P}{\mathbf{\Psi}}

\renewcommand{\a}{a}

\newcommand{\e}{\boldsymbol{\epsilon}}

\newcommand{\D}{\mathbf{D}}
\newcommand{\Dm}{\customsubscript{\leq \m}{\D}}

\newcommand{\KfM}{\customsubscript{> \m}{\Kf}}
\newcommand{\Kfm}{\customsubscript{\leq \m}{\Kf}}
\newcommand{\Kf}{\mathcal{K}}

\newcommand{\SfM}{\customsubscript{> \m}{\Sf}}
\newcommand{\Sfm}{\customsubscript{\leq \m}{\Sf}}
\newcommand{\Sf}{{\mathcal{S}}}
\newcommand{\IndS}{{\mathcal{I}}}

\newcommand{\IndSlarge}{{\hat{\mathcal{I}}}}

\newcommand{\lt}{\tilde{\l}}
\newcommand{\rate}{\eta}
\newcommand{\ratev}{{\eta_v}}
\newcommand{\rateb}{{\eta_b}}
\newcommand{\ratemin}{\eta_{\min}}
\renewcommand{\v}{v}

\renewcommand{\S}{\mathbf{S}}
\newcommand{\Sm}{\customsubscript{\leq \m}{\S}}
\newcommand{\SM}{\customsubscript{> \m}{\S}}

\newcommand{\E}{\mathbb{E}}
\newcommand{\W}{\mathbf{W}}

\renewcommand{\t}{\intercal}
\newcommand{\Tr}{\mathbf{Tr}}
\newcommand{\Proj}{\mathbf{P}}

\newcommand{\innersize}[2]{\langle#1\rangle}

\newtheorem{theorem}{Theorem}
\newtheorem{remark}{Remark}
\newtheorem{lemma}{Lemma}

\newtheorem{proposition}{Proposition}

\newtheorem{assumption}{Assumption}

%% file: figurelayout.tex
\newlength{\figcontentwidth}
\setlength{\figcontentwidth}{397pt}
\newlength{\figgutterwidth}
\setlength{\figgutterwidth}{10pt}
\newlength{\figcolwidth}
\setlength{\figcolwidth}{0.083333\figcontentwidth-0.916667\figgutterwidth}

\newlength{\figtwelvecol}
\setlength{\figtwelvecol}{12\figcolwidth+11\figgutterwidth}
\newlength{\figsixcol}
\setlength{\figsixcol}{6\figcolwidth+5\figgutterwidth}
\newlength{\figfivecol}
\setlength{\figfivecol}{5\figcolwidth+4\figgutterwidth}
\newlength{\figfourcol}
\setlength{\figfourcol}{4\figcolwidth+3\figgutterwidth}
\newlength{\figthreecol}
\setlength{\figthreecol}{3\figcolwidth+2\figgutterwidth}
\newlength{\figonecol}
\setlength{\figonecol}{\figcolwidth}

%% file: abstract.tex
\begin{abstract}
Classical wisdom suggests that estimators should avoid fitting noise to achieve good generalization.
In contrast, modern overparameterized models can yield small test error
despite interpolating noise --- a phenomenon often called ``benign overfitting''
or ``harmless interpolation''.
This paper argues that the degree to which interpolation is harmless hinges upon
the strength of an estimator's inductive bias, i.e.,
how heavily the estimator favors solutions with a certain structure:
while strong inductive biases prevent harmless interpolation,
weak inductive biases can even require fitting noise to generalize well.
Our main theoretical result establishes tight non-asymptotic bounds
for high-dimensional kernel regression that reflect this phenomenon
for convolutional kernels, where the filter size regulates the strength of the inductive bias.
We further provide empirical evidence of the same behavior for deep neural networks
with varying filter sizes and rotational invariance.
\end{abstract}

%% file: mainmatter/introduction.tex
\section{Introduction}

According to classical wisdom (see, e.g., \cite{hastie01statisticallearning}),
an estimator that fits noise suffers from ``overfitting''
and cannot generalize well.
A typical solution is to prevent interpolation, that is,
stopping the estimator from achieving zero training error
and thereby fitting less noise.
For example, one can use ridge regularization or early stopping
for iterative algorithms  %
to obtain a model that has training error close to the noise level.
However, large overparameterized models such as neural networks seem to behave differently:
even on noisy data, they may achieve optimal test performance
at convergence after interpolating the training data %
\citep{Nakkiran_2021, belkin_2019}
--- a phenomenon referred to as \emph{harmless interpolation} \citep{muthukumar_2020} or \emph{benign overfitting} \citep{bartlett_2020} and often discussed in the context of \emph{double descent} \citep{belkin_2019}.

To date, we lack a general understanding of when interpolation is harmless for overparameterized models.
In this paper, we argue that the \emph{strength of an inductive bias}
critically influences whether an estimator exhibits harmless interpolation.
An estimator with a strong inductive bias
heavily favors ``simple'' solutions that structurally  align with the ground truth
(such as sparsity or rotational invariance).
Based on well-established high-probability recovery results of sparse linear regression
\citep{tibshirani_1996,candes_2008,donoho_2006},
we expect that models with a stronger inductive bias
generalize better than ones with a weaker inductive bias, particularly from noiseless data.
In contrast, the effects of inductive bias are much less studied for interpolators of noisy data.

Recently, \citet{donhauser22a} provided a first rigorous analysis
of the effects of inductive bias strength
on the generalization performance
of linear max-$\ell_p$-margin/min-$\ell_p$-norm interpolators.
In particular, the authors prove that
a stronger inductive bias (small $p \to 1$) not solely
enhances a model's ability to generalize on noiseless data,
but also increases a model's sensitivity to noise
--- eventually harming generalization when interpolating noisy data.
As a consequence, their result suggests that interpolation might not be harmless
when the inductive bias is too strong.

In this paper, we confirm the hypothesis and show that strong inductive biases indeed prevent harmless interpolation, while also moving away from sparse linear models.
As one example, we consider data where the true labels nonlinearly only depend on input features in a local neighborhood, and vary the strength of the inductive bias via the filter size of convolutional kernels or shallow convolutional neural networks
--- small filter sizes encourage functions that depend nonlinearly only on local neighborhoods of the input features.
As a second example, we also investigate classification for rotationally invariant data,
where we encourage different degrees of rotational invariance for neural networks.
In particular,
\begin{itemize}
\item we prove a phase transition between harmless and harmful interpolation
    that occurs by varying the strength of the inductive bias via the filter
    size of convolutional kernels for kernel regression in the high-dimensional setting
    (\cref{th:rates}).
\item we further show that, for a weak inductive bias,
    not only is interpolation harmless but partially fitting the observation noise
    is in fact necessary
    (\cref{th:trainerror}).
\item we show the same phase transition experimentally for neural networks
    with two common inductive biases:
    varying convolution filter size,
    and rotational invariance enforced via data augmentation
    (\cref{sec:empirical}).
\end{itemize}

From a practical perspective,
empirical evidence suggests that large neural networks
not necessarily benefit from early stopping.
Our results match those observations for typical networks
with a weak inductive bias;
however, we caution that strongly structured models
must avoid interpolation,
even if they are highly overparameterized.

%% file: mainmatter/related_work.tex
\section{Related work}
\label{sec:related_work}

We now discuss three groups of related work and explain
how their theoretical results cannot reflect
the phase transition between harmless and harmful interpolation for high-dimensional kernel learning.

\pseudoparagraph{Low-dimensional kernel learning}
Many recent works
\citep{bietti2021sample,favero2021locality,bietti2022approximation,wyart}
prove statistical rates for kernel regression
with convolutional kernels in low-dimensional settings,
but crucially rely on ridge regularization.
In general, one cannot expect harmless interpolation
for such kernels in the low-dimensional regime
\citep{rakhlin2019consistency,mallinar2022benign,buchholz2022kernel};
positive results exist only
for very specific adaptive spiked kernels \citep{belkin2019does}.
Furthermore, techniques developed for low-dimensional settings
(see, e.g., \cite{scholkopf18}) usually suffer from a curse of dimensionality,
that is, the bounds become vacuous in high-dimensional settings
where the input dimension grows with the number of samples.

\pseudoparagraph{High-dimensional kernel learning}
One line of research
\citep{liang2020multiple,McRae2022,liang2020just,liu2021kernel}
tackles high-dimensional kernel learning and
proves non-asymptotic bounds
using advanced high-dimensional random matrix concentration tools
from \cite{el2010spectrum}.
However, those results heavily rely on a bounded Hilbert norm assumption.
This assumption is natural in the low-dimensional regime,
but misleading in the high-dimensional regime,
as pointed out in \cite{donhauser21}.
Another line of research
\citep{Ghorbani19,ghorbani2020neural,Mei2021_invariances,ghosh2022the,misiakiewicz2021learning,Mei2021_hypercontractivity}
asymptotically characterizes the precise risk
of kernel regression estimators
in specific settings
with access to a kernel's eigenfunctions and eigenvalues.
However, these asymptotic results are insufficient
to investigate how varying the filter size of a convolutional kernel
affects the risk of a kernel regression estimator.
In contrast to both lines of research,
we prove tight non-asymptotic matching upper and lower bounds for high-dimensional kernel learning which precisely capture the phase transition described in \cref{subsec:main_result}.

\pseudoparagraph{Overfitting of structured interpolators}
Several works question the generality of harmless interpolation %
for models that incorporate strong structural assumptions.
Examples include structures enforced via
data augmentation \citep{Nishi21},
adversarial training \citep{rice2020overfitting,Kamath21,sanyal2021how,donhauser21neurips},
neural network architectures \citep{Li21},
pruning-based sparsity \citep{Chang21},
and sparse linear models \citep{wang2021tight,muthukumar_2020,chatterji2022foolish}.
In this paper, we continue that line of research
and offer a new theoretical perspective
to characterize when interpolation is expected to be harmless.

%% file: mainmatter/theory.tex
\section{Theoretical results}
For convolutional kernels,
a small filter size induces a strong bias towards estimators
that depend nonlinearly on the input features
only via small patches.
This section analyzes the effect of filter size
(as an example inductive bias)
on the degree of harmless interpolation
for kernel ridge regression.
For this purpose, we derive and compare tight non-asymptotic
bias and variance bounds
as a function of filter size
for min-norm interpolators
and optimally ridge-regularized estimators
(\cref{th:rates}).
Furthermore, we prove for large filter sizes that
not only does harmless interpolation occur
(\cref{th:rates}),
but fitting some degree of noise is even necessary
to achieve optimal test performance
(\cref{th:trainerror}).

\subsection{Setting}
\label{subsec:setting}
We study kernel regression with a (cyclic) convolutional kernel
in a high-dimensional setting
where the number of training samples $n$
scales with the dimension of the input data $d$ as $n \in \Theta(d^{\Ln})$.
We use the same setting as in previous works on high-dimensional kernel learning
such as \cite{misiakiewicz2021learning}:
we assume that the training samples
$\{x_i,y_i\}_{i = 1}^n$
are i.i.d.\ draws from the distributions
$x_i \sim \U(\2^d)$,
and $y_i = \fstar(x_i) + \epsilon_i$
with ground truth $\fstar$
and noise $\epsilon \sim \mathcal{N}(0,\sigma^2)$.
For simplicity of exposition,
we further assume that $\fstar(x) = x_1 x_2 \cdots x_\Lgt$,
with $\Lgt$ specified in \cref{th:rates}.

While the assumptions on the noise and ground truth can be easily extended
by following the proof steps in \cref{sec:proof},
generalizing the feature distribution is challenging.
Indeed, existing results that establish precise risk characterizations
(see \cref{sec:related_work})
crucially rely on hypercontractivity of the feature distribution
--- an assumption so far only proven for few high-dimensional distributions,
including the hypersphere \citep{beckner95},
and the discrete hypercube \citep{bc75}
which we use in this paper.
Hypercontractivity is essential to tightly control the empirical kernel matrix
within \cref{lem:combined} in \cref{sec:proof}.
Generalizations beyond this assumption require the development
of new tools in random matrix theory,
which we consider important future work.

We consider (cyclic) convolutional kernels
with filter size $\filsize \in \{1, \dotsc, d\}$ of the form
\begin{equation}\label{eq:conv_kern_def}
    \Kf(x,x') = \frac 1d \sum_{k = 1}^d{\kappa\left(\frac{\innersize{x_{(k,\filsize)},x'_{(k,\filsize)}}{\filsize}}{\filsize}\right)},
\end{equation}
where $x_{(k,\filsize)}\coloneqq[x_{\module kd}\cdots x_{\module{k+\filsize-1}d}]$,
and $\kappa\colon [-1,1]\to \R$ is a nonlinear function that
implies standard regularity assumptions
(see \cref{ass:genericity} in \cref{app:hypercube})
that hold for instance for the exponential function.
Decreasing the filter size $\filsize$
restricts kernel regression solutions to depend nonlinearly only on local neighborhoods  instead of the entire input $x$.

We analyze the kernel ridge regression (KRR) estimator, which is the minimizer of the following convex optimization problem:
\begin{equation}\label{eq:krr_def}
    \hat f_\reg = \argmin_{f \in \mathcal H} \frac 1n\sum_{i = 1}^n (f(x_i)-y_i)^2 + \frac{\reg}n \norm{f}_{\mathcal{H}}^2,
\end{equation}
where $\mathcal H$ is the Reproducing Kernel Hilbert space (RKHS)
over $\{-1,1\}^d$ generated by the convolutional kernel $\Kf$ in \cref{eq:conv_kern_def},
$\norm{\cdot}_{\mathcal H}$ the corresponding norm,
and $\reg>0$ the ridge regularization penalty.\footnote{
Note that previous works show how early-stopped gradient methods on the square loss
behave statistically similarly to kernel ridge regression \citep{Raskutti14,Wei17}.
}
In the interpolation limit
($\reg \to 0$), we obtain the min-RKHS-norm interpolator
\begin{equation}\label{eq:kri_def}
    \hat f_0 = \argmin_{f \in \mathcal H} \|f\|_{\mathcal H} \quad \text{s.t.}\quad \forall i:~f(x_i) = y_i.
\end{equation}
For simplicity, we refer to $\hat f_0$ as the kernel ridge regression estimator with $\reg =0$.
We evaluate all estimators with the expected population risk over the noise, defined as
\begin{align*}
    \risk(\hat f_{\reg})
    &\coloneqq
    \underbrace{
        \E_{x}\left[\left(\E_\epsilon[\hat f_{\reg}(x)]-\fstar (x)\right)^2\right]
        }_{\coloneqq\bias^2(\hat f_{\reg})}
    +
    \underbrace{
        \E_{x,\epsilon}\left[\left(\hat f_{\reg}(x)-\E_\epsilon[\hat f_{\reg}(x)]\right)^2\right]
    }_{\coloneqq\var(\hat f_{\reg})} .
\end{align*}

\subsection{Main result}
\label{subsec:main_result}
We now present tight upper and lower bounds
for the prediction error of kernel regression estimators
in the setting from \cref{subsec:setting}.
The resulting rates hold for the high-dimensional regime, that is,
when both the ambient dimension $d$ and filter size $\filsize$ scale with $n$.\footnote{
We hide positive constants that depend at most on $\Ln$ and $\beta$ (defined in \cref{th:rates})
using the standard Bachmann--Landau notation $\bigO(\cdot)$, $\Omega(\cdot)$, $\Theta(\cdot)$,
as well as $\lesssim$, $\gtrsim$,
and use $\c, \c_1, \dots$ as generic positive constants.
}
We defer the proof to \cref{sec:proof}.

\newpage
\begin{theorem}[Non-asymptotic prediction error rates]
\label{th:rates}
Let $\Ln > 0$, $\beta \in (0, 1)$, $\Ls \in \R$.
Assume a dataset and a kernel
as described in \cref{subsec:setting},
with the kernel satisfying \cref{ass:genericity}.
Assume further
$n \in \Theta(d^{\Ln})$,
the filter size $\filsize \in \Theta\left(d^\beta\right)$,
and $\sigma^2 \in \Theta(d^{-\Ls})$.
Lastly, define $\bd \coloneqq \frac{\Ln - 1}\beta - \floor{\frac{\Ln -1}\beta}$
and $\delta \coloneqq \frac{\Ln-\Lr -1}\beta  - \floor{\frac{\Ln-\Lr -1}\beta}$
for any $\Lr$.
Then, with probability at least $1-cd^{-\beta\min\{\bd,1-\bd\}}$
uniformly over all $\Lr \in [0, \Ln - 1)$,
the KRR estimate $\freg$ in \cref{eq:krr_def}
with $\max\{\reg,1\} \in \Theta(d^{\Lr})$ satisfies
\begin{align*}
    \var(\hat f_\reg)  &\in \Theta\left(n ^{\frac{-\Ls-\Lr}{\Ln} - \frac{\beta}{\Ln} \min\{\delta, 1-\delta\}}\right).
\end{align*}
Further, for a ground truth $\fstar(x) = x_1 x_2 \cdots x_\Lgt$
with $\Lgt \leq \ceil*{\frac{{\Ln-\Lr-1}}\beta}$,
with probability at least $1-cd^{-\beta\min\{\bd,1-\bd\}}$, we have
\begin{equation*}
     \bias^2(\hat f_\reg)  \in \Theta\left(n^{-2-\frac{2}{\Ln}(-\Lr-1-\beta(\Lgt-1))}\right).
\end{equation*}
Finally, by setting $\Lr=0$, both rates also hold for the min-RKHS-norm interpolator $\fint$
in Eq.~\eqref{eq:kri_def}.

\end{theorem}

Note how the theorem reflects the usual intuition
for the effects of noise and ridge regularization strength on bias and variance
via the parameter $\Lr$:
With increasing ridge regularization $\Lr$ (and thus increasing $\lambda$),
the bias increases and the variance decreases.
Similarly, as noise increases (and thus $\Ls$ decreases), the variance increases.

\pseudoparagraph{Phase transition as a function of $\beta$}
In the following, we focus on the impact
of the filter size $\filsize \in \Theta\left(d^\beta\right)$
on the risk (sum of bias and variance)
via the growth rate $\beta$.
Recalling that a small filter size (small $\beta$)
corresponds to a strong inductive bias, and vice versa,
\cref{fig:theory} demonstrates how the strength of the inductive bias affects generalization.
For illustration, we choose the ground truth $\fstar(x) = x_1 x_2$
so that the assumption on $\Lgt$ is satisfied for all $\beta$.
Specifically, \cref{fig:rates_opt_reg} shows the rates
for the min-RKHS-norm interpolator $\hat f_0$
and the optimally ridge-regularized estimator $\fopt$,
where we choose $\regopt$ to minimize the expected population risk $\risk(\fopt)$.
Furthermore, \cref{fig:tradeoff} depicts the (statistical) bias and variance
of the interpolator $\hat f_0$.
At the threshold $\betath \in (0,1)$,
implicitly defined as the $\beta$ at which the rates
of statistical bias and variance in \cref{th:rates} match,
we can observe the following phase transition:

\begin{itemize}
    \item For $\beta < \betath$, that is, for a strong inductive bias,
    the rates in \cref{fig:rates_opt_reg} for the optimally ridge-regularized estimator
    $\fopt$ are strictly better than the ones
    for the corresponding interpolator $\hat f_0$.
    In other words, we are observing \emph{harmful interpolation}.
    \item For $\beta > \betath$, that is, for a weak inductive bias,
    the rates in \cref{fig:rates_opt_reg} of the optimally ridge-regularized estimator
    $\fopt$ and the min-RKHS-norm interpolator $\hat f_0$ match.
    Hence, we observe \emph{harmless interpolation}.
\end{itemize}

In the following theorem, we additionally show that interpolation is not only harmless
for $\beta>\betath$,
but the optimally ridge-regularized estimator $\fopt$
necessarily fits part of the noise
and has a training error strictly below the noise level.
In contrast, we show that when interpolation is harmful in \cref{fig:rates_opt_reg},
that is, when $\beta < \betath$,
the training error of the optimally ridge-regularized model approaches the noise level.
\begin{theorem}[Training error (informal)]
\label{th:trainerror}
Let $\regopt$ be such that the expected population risk~$\risk(\fopt)$ is minimal,
and let $\beta^*$ be the unique threshold\footnote{
See \cref{th:trainerror_formal} for a more general statement
that does not rely on a unique $\betath$.
}
where the bias and variance bounds in \cref{th:rates}
are of the same order for the interpolator $\hat f_0$ (setting $\Lr =0$).
Then, the expected training error converges in probability:
\begin{align*}
    \lim_{n,d\to \infty}{
    \frac{1}{\sigma^2 }
    \E_{\epsilon}{\left[\frac1n\sum_i (\fopt(x_i) - y_i)^2\right]}
    }
    \quad
    \begin{cases}
        \: = 1 &\beta < \betath, \\
        \: \leq c_{\beta} &\beta \geq \betath,
    \end{cases}
\end{align*}
where $c_{\beta} < 1$ for any $\beta > \beta^*$.
\end{theorem}

We refer to \cref{ssec:training_error_proof} for the proof and a more general statement.

\begin{figure}[t!]
    \centering
    \begin{subfigure}[l]{\figsixcol}
        \centering
        \begin{subfigure}[t]{\linewidth}
            \centering
            \includegraphics[width=\linewidth]{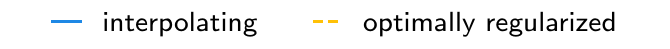}
        \end{subfigure}
        \begin{subfigure}[b]{\linewidth}
            \centering
            \includegraphics[width=\linewidth]{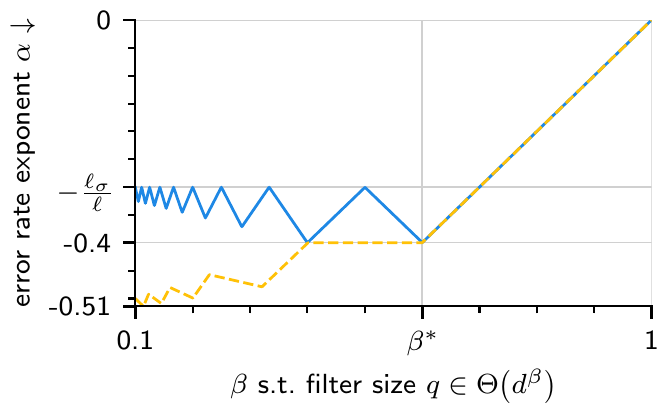}
        \end{subfigure}
        \caption{Interpolating vs.\ regularized model}
        \label{fig:rates_opt_reg}
    \end{subfigure}
    \hfill%
    \begin{subfigure}[r]{\figsixcol}
        \centering
        \begin{subfigure}[t]{\linewidth}
            \centering
            \includegraphics[width=\linewidth]{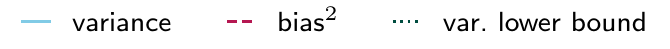}
        \end{subfigure}
        \begin{subfigure}[b]{\linewidth}
            \centering
            \includegraphics[width=\linewidth]{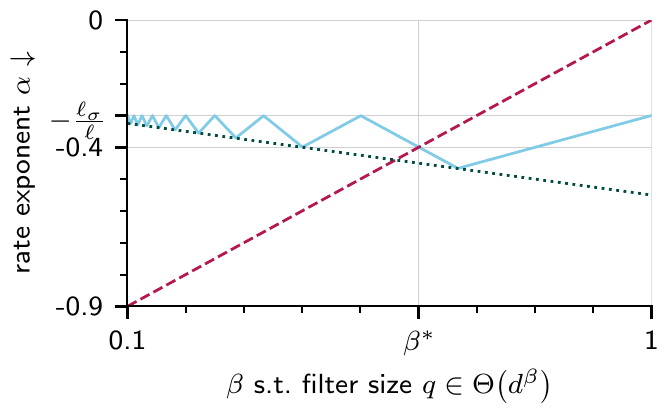}
        \end{subfigure}
        \caption{Variance vs.\ bias for interpolating model}
        \label{fig:tradeoff}
    \end{subfigure}
    \caption{
        Illustration of the rates in \cref{th:rates}
        for high-dimensional kernel ridge regression
        as a function of $\beta$
        --- the rate of the filter size $\filsize \in \Theta(d^\beta)$.
        (a) Rate exponent $\alpha$ of the $\risk \in \Theta(n^{\alpha})$ for
        the interpolator $\hat f_0$
        vs.\ the optimally ridge-regularized estimator $\fopt$.
        (b) Rate exponent of the variance and bias for the interpolator $\hat f_0$.
        For both illustrations,
        we choose $\hat f_0$ with $\Ln = 2$, $\Ls = 0.6$,
        and the ground truth $\fstar(x) = x_1x_2$.
        Lastly, $\betath$ denotes the threshold
        where the bias and variance terms in \cref{th:rates} match,
        and where we observe a phase transition between harmless and harmful interpolation.
        See \cref{app:opt_reg} for technical details.
    }
    \vspace{-0.2in}
    \label{fig:theory}
\end{figure}

\pseudoparagraph{Bias-variance trade-off}
We conclude by discussing how the phase transition
arises from a (statistical) bias and variance trade-off
for the min-RKHS-norm interpolator
as a function of $\beta$,
reflected in \cref{th:rates} when setting $\Lr =0$
and illustrated in \cref{fig:tradeoff}.
While the statistical bias monotonically decreases with decreasing $\beta$
(i.e., increasing strength of the inductive bias),
the variance follows a multiple descent curve
with increasing minima as $\beta$ decreases.
Hence, analogous to the observations in \cite{donhauser22a}
for linear max-$\ell_p$-margin/min-$\ell_p$-norm interpolators,
the interpolator achieves its optimal performance
at  a  $\beta \in (0,1)$,
and therefore at a moderate inductive bias.
Finally, we note that \citet{liang2020multiple}
previously observed a multiple descent curve for the variance,
but as a function of input dimension
and without any connection to structural biases.
\vspace{-0.05in}

%% file: mainmatter/empirical.tex
\section{Experiments}
\label{sec:empirical}
\vspace{-0.05in}
We now empirically study whether the phase transition phenomenon
that we prove for kernel regression
persists for deep neural networks with feature learning.
More precisely, we present controlled experiments to investigate
how the strength of a CNN's inductive bias
influences if interpolating noisy data is harmless.
In practice, the inductive bias of a neural network varies
by way of design choices such as the architecture
(e.g., convolutional vs.\ fully-connected vs.\ graph networks)
or the training procedure (e.g., data augmentation, adversarial training).
We focus on two examples:
convolutional filter size that we vary via the architecture,
and rotational invariance via data augmentation.
To isolate the effects of inductive bias and provide conclusive results,
we use datasets where we know a priori that the ground truth
exhibits a simple structure that matches the networks' inductive bias.
See \cref{sec:experimentdetails} for experimental details.

Analogous to ridge regularization for kernels,
we use early stopping as a mechanism to prevent noise fitting.
Our experiments compare optimally early-stopped CNNs to their interpolating versions.
This highlights a trend that
mirrors our theoretical results:
the stronger the inductive bias of a neural network grows,
the more harmful interpolation becomes.
These results suggest exciting future work:
proving this trend for models with feature learning.
\vspace{-0.05in}

\subsection{Filter size of CNNs on synthetic images}
\label{ssec:empirical_filters}
\vspace{-0.05in}
In a first experiment, we study the impact of filter size on
the generalization of interpolating CNNs.
As a reminder, small filter sizes yield functions that depend
nonlinearly only on local neighborhoods of the input features.
To clearly isolate the effects of filter size,
we choose a special architecture on a synthetic classification problem
such that the true label function is indeed a CNN with small filter size.
Concretely, we generate images of size $32 \times 32$
containing scattered circles (negative class) and crosses (positive class)
with size at most $5 \times 5$.
Thus, decreasing filter size down to $5 \times 5$ corresponds to a stronger inductive bias.
Motivated by our theory, we hypothesize that interpolating noisy data
is harmful with a small filter size,
but harmless when using a large filter size.

\pseudoparagraph{Training setup}
In the experiments, we use CNNs with a single convolutional layer,
followed by global spatial max pooling and two dense layers.
We train those CNNs with different filter sizes on $200$ training samples
(either noiseless or with $20\%$ label flips)
to minimize the logistic loss and achieve zero training error.
We repeat all experiments over $5$ random datasets with
$15$ optimizations per dataset and filter size,
and report the average $0$-$1$-error for $100$k test samples per dataset.
For a detailed discussion on the choice of hyperparameters
and more experimental details, see \cref{ssec:experimentdetails_setup_filters}.

\pseudoparagraph{Results}
First, the noiseless error curves (dashed) in \cref{fig:filters_error_synthetic}
confirm the common intuition
that the strongest inductive bias (matching the ground truth) at size $5$
yields the lowest test error.
More interestingly, for $20\%$ training noise (solid),
\cref{fig:filters_error_synthetic} reveals a similar phase transition as \cref{th:rates}
and confirms our hypothesis:
Models with weak inductive biases (large filter sizes)
exhibit harmless interpolation,
as indicated by the matching test error of interpolating (blue)
and optimally early-stopped (yellow) models.
In contrast, as filter size decreases,
models with a strong inductive bias (small filter sizes)
suffer from an increasing gap in test errors
when interpolating versus using optimal early stopping.
Furthermore, \cref{fig:filters_error_es}
reflects the dual perspective of the phase transition
as presented in \cref{th:trainerror}
under optimal early stopping:
models with a small filter size
entirely avoid fitting training noise,
such that the training error on the noisy training subset equals $100\%$,
while models with a large filter size interpolate the noise.

\begin{figure}[t!]
    \centering
    \begin{subfigure}[l]{\figsixcol}
        \centering
        \begin{subfigure}[t]{\figsixcol}
            \centering
            \includegraphics[width=\linewidth]{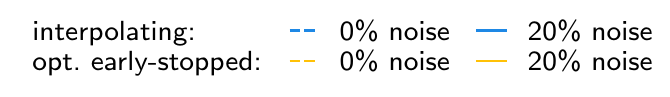}
        \end{subfigure}
        \begin{subfigure}[b]{\figfivecol}
            \centering
            \includegraphics[width=\linewidth]{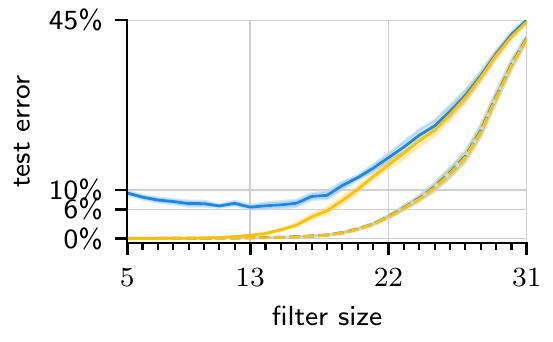}
        \end{subfigure}
        \caption{Test error}
        \label{fig:filters_error_synthetic}
    \end{subfigure}
    \hfill%
    \begin{subfigure}[r]{\figsixcol}
        \centering
        \begin{subfigure}[t]{\figsixcol}
            \centering
            \includegraphics[width=\linewidth]{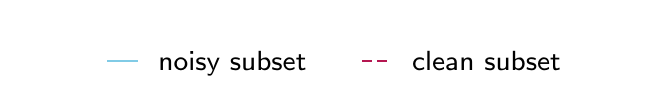}
        \end{subfigure}
        \begin{subfigure}[b]{\figfivecol}
            \centering
            \includegraphics[width=\linewidth]{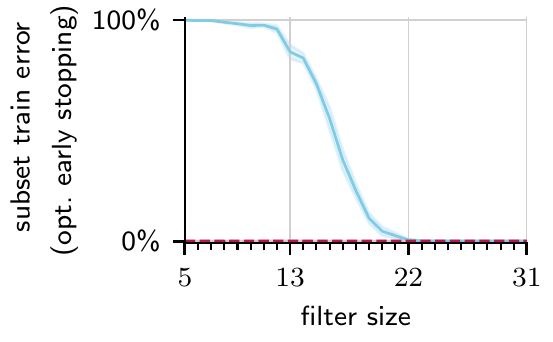}
        \end{subfigure}
        \caption{Training error of optimally early-stopped models}
        \label{fig:filters_error_es}
    \end{subfigure}
    \caption{
        Convolutional neural network experiments with varying filter size on synthetic image data.
        (a) For noisy data, small filter sizes (strong inductive bias) induce a gap between
        the generalization performance of interpolating (blue)
        vs.\ optimally early-stopped models (yellow).
        The gap vanishes
        as the inductive bias decreases (i.e., filter size increases).
        For noiseless data (dashed), interpolation is always harmless.
        (b) Training error of the optimally early-stopped model (optimized for test error)
        on the noisy and clean subsets of a training set with $20\%$ label noise.
        Under optimal early stopping,
        models with a strong inductive bias
        ignore all noisy samples ($100\%$ error on the noisy subset),
        while models with a weak inductive bias fit all noisy training samples
        ($0\%$ error on the noisy subset).
        All lines show the mean over five random datasets,
        and shaded areas the standard error;
        see \cref{ssec:empirical_filters} for the experiment setup.
    }
    \vspace{-0.2in}
    \label{fig:filters_error}
\end{figure}

\pseudoparagraph{Difference to \ddd{}}
One might suspect that our empirical observations simply
reflect another form of \ddd{} \citep{belkin_2019}.
As a CNN's filter size increases (inductive bias becomes weaker),
so does the number of parameters and degree of overparameterization.
Thus, \ddd{} predicts vanishing benefits of regularization due to model size
for weak inductive biases.
Nevertheless, we argue that
the phenomenon we observe here is distinct,
and provide an extended discussion in \cref{ssec:experimentdetails_ddd}.
In short, we choose sufficiently large networks and tune their hyperparameters
to ensure that all models interpolate and yield small training loss,
even for filter size $5$ and $20\%$ training noise.
To justify this approach, we repeat a subset of the experiments
while significantly increasing the convolutional layer width.
As the number of parameters increases for a fixed filter size,
\ddd{} would predict that the benefits of optimal early stopping vanish.
However, we observe that our phenomenon persists.
In particular, for filter size $5$ (strongest inductive bias),
the test error gap between interpolating and optimally early-stopped
models remains large.

\subsection{Rotational invariance of \wrn{}s on satellite images}
\label{ssec:empirical_rotations}
In a second experiment,
we investigate rotational invariance as an inductive bias for CNNs
whenever true labels are independent of an image's rotation.
Our experiments control inductive bias strength
by fitting models on multiple rotated versions
of an original training dataset,
effectively performing varying degrees of data augmentation.\footnote{
Data augmentation techniques can efficiently enforce rotational invariance;
see, e.g., \cite{yang19}.
}
As an example dataset with a rotationally invariant ground truth,
we classify satellite images from the EuroSAT dataset \citep{Helber18}
into $10$ types of land usage.
Because the true labels are independent of image orientation,
we expect rotational invariance to be a particularly fitting
inductive bias for this task.

\pseudoparagraph{Training and test setup}
For computational reasons, we subsample the original EuroSAT training set
to $7680$ raw training and $10$k raw test samples.
In the noisy case, we replace $20\%$ of the raw training labels with
a wrong label chosen uniformly at random.
We then vary the strength of the inductive bias towards rotational invariance
by augmenting the dataset with an increasing number of $k$ rotated versions
of itself.
For each sample, we use $k$ equal-spaced angles spanning $360^\circ$,
plus a random offset.
Note that training noise applies before rotations,
so that all rotated versions of the same image share the same label.
We then center-crop all rotated images
such that they only contain valid pixels.
In all experiments, we fit \wrn{}s \citep{Zagoruyko16}
on the augmented training set
for $5$ different network initializations.
We evaluate the $0$-$1$-error on the randomly rotated test samples
to avoid distribution shift effects from image interpolation.
All random rotations are the same for all experiments
and stay fixed throughout training.
See \cref{ssec:experimentdetails_setup_rotations} for more experimental details.
Lastly, we perform additional experiments with larger models
to differentiate from \ddd{};
see \cref{ssec:experimentdetails_ddd} for the results and further discussions.

\pseudoparagraph{Results}
Similar to the previous subsection,
\cref{fig:rotations_error} corroborates our hypothesis under rotational invariance:
stronger inductive biases result in lower test errors on noiseless data,
but an increased gap between the test errors of interpolating
and optimally early-stopped models.
In contrast to filter size, the phase transition is more abrupt;
invariance to $180^\circ$ rotations
already prevents harmless interpolation.
\Cref{fig:rotations_es} confirms this from a dual perspective,
since all models with some rotational invariance
cannot fit noisy samples for optimal generalization.

\begin{figure}[t!]
    \centering
    \begin{subfigure}[l]{\figsixcol}
        \centering
        \begin{subfigure}[t]{\figsixcol}
            \centering
            \includegraphics[width=\linewidth]{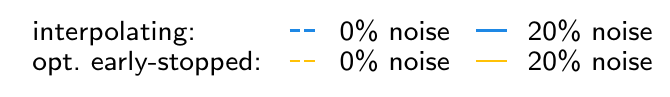}
        \end{subfigure}
        \begin{subfigure}[b]{\figfivecol}
            \centering
            \includegraphics[width=\linewidth]{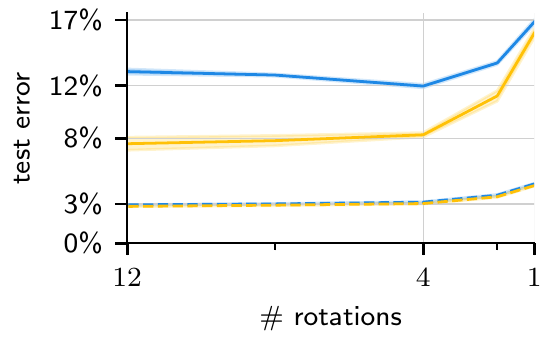}
        \end{subfigure}
        \caption{Test error}
        \label{fig:rotations_error}
    \end{subfigure}
    \hfill%
    \begin{subfigure}[r]{\figsixcol}
        \centering
        \begin{subfigure}[t]{\figsixcol}
            \centering
            \includegraphics[width=\linewidth]{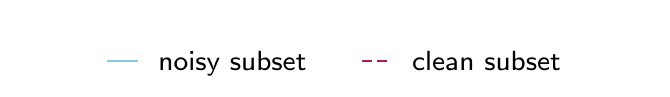}
        \end{subfigure}
        \begin{subfigure}[b]{\figfivecol}
            \centering
            \includegraphics[width=\linewidth]{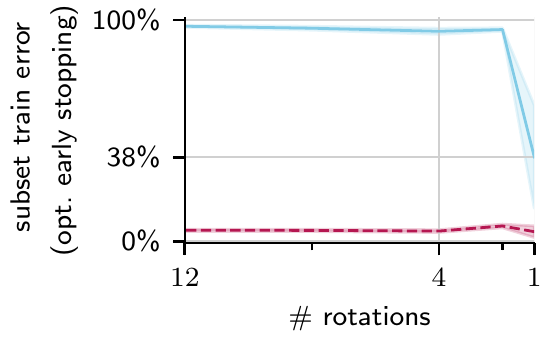}
        \end{subfigure}
        \caption{Training error of optimally early-stopped models}
        \label{fig:rotations_es}
    \end{subfigure}
    \caption{
        Varying degrees of rotational invariance when fitting \wrn{}s on satellite images.
        (a) For noisy data (solid), a strong bias towards rotational invariance
        (via the number of augmented rotations)
        induces a gap between the generalization performance of interpolating (blue)
        vs.\ optimally early-stopped models (yellow).
        The gap decreases as the inductive bias decreases (\# rotations decreases).
        (b) Training error of optimally early-stopped models (w.r.t.\ the test error)
        on the noisy and clean subsets of a training set with $20\%$ label noise:
        For maximum rotational invariance ($12$ rotations),
        optimally early-stopped models avoid fitting noisy data
        (close to $100\%$ training error on noisy subset),
        yet no rotational invariance ($1$ rotation)
        requires fitting noise (less than $50\%$ training error on noisy subset)
        for optimal generalization.
        All lines show the mean over five optimization runs,
        and shaded areas the standard error;
        see \cref{ssec:empirical_rotations} for the experiment setup.
    }\label{fig:rotations}
    \vspace{-0.2in}
\end{figure}

%% file: mainmatter/proof_sketch.tex
\section{Proof of the main result}
\label{sec:proof}
The proof of the main result, \cref{th:rates}, proceeds in two steps:
First, \cref{subsec:proof_s1} presents a fixed-design result
that yields matching upper and lower bounds
for the prediction error of general kernels
under additional conditions.
Second, in \cref{subsec:proof_s2},
we show that the setting of \cref{th:rates}
satisfies those conditions
with high probability over dataset draws.

\paragraph{Notation}
Assuming that inputs are draws
from a data distribution $\nu$ (i.e., $x, x' \sim \nu$),
we can decompose and divide any continuous, positive semi-definite kernel function as
\begin{equation}
\label{eq:kerneldecomp}
    \Kf(x,x') = \sum_{\i = 1}^{\infty} \lambda_{\i}\Pf_\i(x)\Pf_\i(x') =
    \underbrace{
        \sum_{\i = 1}^{\m} \lambda_{\i}\Pf_\i(x)\Pf_\i(x')
    }_{\defeq \Kfm(x,x')}
    + \underbrace{
        \sum_{\i = \m + 1}^{\infty} \lambda_{\i}\Pf_\i(x)\Pf_\i(x')
    }_{\defeq \KfM(x,x')},
\end{equation}
where $\{\Pf_\i\}_{\i \geq 1}$ is an orthonormal basis of the RKHS
induced by $\inner{f,g}_\nu \coloneqq \E_{x\sim\nu}[f(x)g(x)]$
and the eigenvalues $\lambda_{\i}$ are sorted in descending order.
In the following, we write $[\cdot]_{i,j}$ to refer
to the entry in row $i$ and column $j$
of a matrix.
Then, we define the empirical kernel matrix for $\Kf$ as $\K\in \R^{n\times n}$
with $[\K]_{i,j} \coloneqq \Kf(x_i,x_j)$,
and analogously the truncated versions $\Km$ and $\KM$ for $\Kfm$ and $\KfM$, respectively.
Next, we utilize the matrices $\Pm \in \R^{n\times \m}$ with $[\Pm]_{i,l} \coloneqq \Pf_l(x_i)$,
and $\Dm \coloneqq \text{diag}(\lambda_1,\dots,\lambda_\m)$.
We further use the squared kernel $\Sf(x,x') \coloneqq \E_{z \sim \nu}[\Kf(x,z)\Kf(z,x')]$,
its truncated versions $\Sfm$ and $\SfM$,
as well as the corresponding empirical kernel matrices
$\S,\Sm,\SM \in \mathbb R^{n\times n}$.
Next, for a symmetric positive-definite matrix,
we write $\mineig{\cdot}$ and $\mumax{\cdot}$ (or $\norm{\cdot}$)
to indicate the min and max eigenvalue, respectively,
and $\eigval_\iii(\cdot)$ for the $\iii$-th eigenvalue
in decreasing order.
Finally, we use $\innersize{\cdot,\cdot}{d}$
for the Euclidean inner product in $\R^d$.

\subsection{Generalization bound for fixed-design}
\label{subsec:proof_s1}
First, \cref{th:vidya} provides tight fixed-design bounds
for the prediction error.
\begin{theorem}[Generalization bound for fixed-design]
\label{th:vidya}
Let $\Kf$ be a kernel that
under a distribution $\nu$
decomposes as $\Kf(x,x') = \sum_\i \lambda_\i \Pf_\i(x)\Pf_\i(x')$
with
$\E_{x\sim\nu}[\Pf_\i(x)\Pf_{\i'}(x)] = \delta_{\i,\i'}$,
and $\{(x_i,y_i)\}_{i = 1}^n$ be a dataset
with $y_i = \fstar(x_i) + \epsilon_i$
for zero-mean $\sigma^2$-variance i.i.d.\ noise $\epsilon_i$
and ground truth $\fstar$.
Define
$\tmin \defeq \min\left\{\frac{n\lambda_\m}{\max\{\reg,1\}},1\right\}$,
$\tmax \defeq \max\left\{\frac{n\lambda_{\m+1}}{\max\{\reg,1\}},1\right\}$,
$\rmin \defeq \frac{\mineig{\KM}+\reg}{\max\{\reg,1\}}$,
$\rmax \defeq \frac{\norm{\KM}+ \reg}{\max\{\reg,1\}}$.
Then, for any $\m \in \mathbb{N}$ such that $\rmin > 0$ and
\begin{equation}
    \label{eq:ass1}
    \norm*{ \Pm^\t\Pm/n-\I_\m} \leq \frac{1}{2},
\end{equation}
the KRR estimate $\freg$ in \cref{eq:krr_def} for any $\reg \geq 0$ has a variance upper and lower bounded by
\begin{equation*}
    \frac{\rmin^2\tmin^2}{2\rmax^2(1.5+\rmin)^2}
    \frac{\m}{n}
    + \frac{\sum_{\iii = \m+1}^{n}{\eigval_\iii{(\SM)}}}{\rmax^2\max\{\reg,1\}^2}
    \leq \var(\hat f_{\reg})/\sigma^2
    \leq 6\frac{\rmax^2}{\rmin^2}
    \frac{\m}{n}
    + \frac{\Tr\left(\SM\right)}{\rmin^2\max\{\reg,1\}^2} .
\end{equation*}
Furthermore, for any ground truth that can be expressed as
$\fstar(x) = \sum_{\i = 1}^\m \a_\i \Pf_\i(x)$
with $\a \in \R_+^m$ and $\Pf_\i$ as defined in \cref{eq:kerneldecomp},
the bias is upper and lower bounded by
 \begin{equation*}
    \frac{\rmin^2\tmin^2}{(1.5+\rmin)^2}
    \max\{\reg,1\}^2
    \frac{\norm{{\Dm^{-1}\a}}^2}{n^2}
    \leq \bias^2(\hat f_{\reg})
    \leq 4 \left(\rmax^2+1.5\frac{\rmax^3}{\rmin^2}\right) \tmax \max\{\reg,1\}^2 \frac{\norm{{\Dm^{-1}\a}}^2}{n^2}.
\end{equation*}
\end{theorem}

See \cref{app:context_agnostic} for the proof.
Note that this result holds for any fixed-design dataset.
We derive the main ideas from the proofs in
\cite{bartlett_2020,Tsigler2021},
where the authors establish tight bounds for the min-$\ell_2$-norm interpolator
on independent sub-Gaussian features.
\begin{remark}[Comparison with \cite{McRae2022}]
The upper bound for the bias in Theorem~1 from \citet{McRae2022}
depends on the suboptimal term ${\norm{{\Dm^{-1/2}\a}}}/{n}$,
but also applies to more general ground truths.
We improve that upper bound in \cref{th:vidya} and present a matching lower bound.
\end{remark}

\subsection{Proof of \texorpdfstring{\cref{th:rates}}{Theorem~\ref{th:rates}}}
\label{subsec:proof_s2}

Throughout the remainder of the proof,
all statements hold for the setting in \cref{subsec:setting}
and under the  assumptions of \cref{th:rates},
especially \cref{ass:genericity} on $\kernel$,
$n\in \Theta(d^\Ln)$, and $\max\{\reg,1\} \in \Theta(d^{\Lr})$.
Furthermore, as in \cref{th:rates}, we use $\delta,\bd>0$
with $\delta\coloneqq \frac{\Ln-\Lr -1}\beta - \floor*{\frac{\Ln-\Lr -1}\beta}$
and $\bd \coloneqq \frac{\Ln -1}\beta  - \floor*{\frac{\Ln -1}\beta}$.
We show that there exists a particular $\m$ for which the conditions
in \cref{th:vidya} hold and we can control the terms in the bias and variance bounds.

\paragraph{Step 1: Conditions for the bounds in \cref{th:vidya}}
We first derive sufficient conditions on $m$ such that
the conditions on $\Pm$ and $\fstar$ in \cref{th:vidya} hold with high probability.
The following standard concentration result
shows that all $\m \ll n$ satisfy \cref{eq:ass1} with high probability.
\begin{lemma}[Corollary of Theorem~5.44 in \cite{Vershynin2012}]
\label{lem:as_1_verification}
For $d$ large enough, with probability at least $1-\c d^{-\beta\bd}$,
all $\m \in \bigO(n\cdot \filsize^{-\bd})$ satisfy \cref{eq:ass1}.
\end{lemma}
See \cref{ssec:matrix_concentration_early} for the proof.
Simultaneously, to ensure that $\fstar$ is contained
in the span of the first $\m$ eigenfunctions,
$\m$ must be sufficiently large.
We formalize this in the following lemma.
\begin{lemma}[Bias bound condition]
    \label{lem:ground_truth_rate}
    Consider a kernel as in \cref{th:rates} satisfying \cref{ass:genericity},
    and a ground truth $\fstar(x) = \prod_{j = 1}^{\Lgt} x_j$
    with $1 \leq \Lgt \leq \ceil*{\frac{{\Ln-\Lr-1}}\beta}$.
    Then, for any $\m$ with
    $\lambda_{\m} \in o{\left(\frac{1}{d \filsize^{\Lgt-1}}\right)}$
    and $d$ sufficiently large,
    $\fstar$ is in the span of the first $\m$ eigenfunctions
    and $ \norm*{\D_{\leq \m}^{-1}\a} \in \Theta\left(d \filsize^{\Lgt-1}\right)$.
\end{lemma}
See \cref{subsec:proof_ground_truth_rate} for the proof.
Note that $\Lgt \geq 1$
follows from $\Lr < \Ln - 1$ in \cref{th:rates},
and allows us to focus on non-trivial ground truth functions.

\vspace{-0.1in}
\paragraph{Step 2: Concentration of the (squared) kernel matrix}
In a second step,
we show that there exists a set of $\m$ that
satisfy the sufficient conditions,
and for which the spectra
of the kernel matrix $\KM$ and squared kernel matrix $\SM$
concentrate.
\begin{lemma}[Tight bound conditions]
\label{lem:combined}
With probability at least $1-\c d^{-\beta\min\{\bd,1-\bd\}}$
uniformly over all $\Lr \in [0, \Ln - 1)$,
for $d$ sufficiently large and
any $\reg \in \R$, $m \in \N$
with $\max\{\reg,1\} \in \Theta(d^{\Lr})$,
$n\lambda_{\m} \in \Theta(\max\{\reg,1\})$, and
$m \in \bigO\left(\frac{n\filsize^{-\delta}}{\max\{\reg,1\}}\right)$,
we have
\begin{align*}
 \vspace{-0.1in}
    \rmin,\rmax \in \Theta(1),
    \quad
    \Tr(\SM) \in \bigO{\left( d^{\Lr}\filsize^{-\delta} +d^{\Lr}\filsize^{-(1-\delta)} \right)},
    \quad
    \sum_{\iii = \m+1}^{n}{\eigval_\iii{(\SM)}}
    \in \Omega{\left( d^{\Lr}\filsize^{-(1-\delta)} \right)}.
\end{align*}
\end{lemma}
We refer to \cref{ssec:proof_lemcombined} for the proof,
which heavily relies on the feature distribution.
Note that the results for $\m$ in \cref{lem:combined}
also imply $\tmin,\tmax \in \Theta(1)$.

\vspace{-0.1in}
\paragraph{Step 3: Completing the proof}

Finally, we complete the proof
by showing the existence of  a particular $\m$
that simultaneously satisfies all conditions of \crefrange{lem:as_1_verification}{lem:combined}.
\begin{lemma}[Eigendecay]
\label{lem:eigendecay_main}
There exists an $\m$ such that
\begin{equation*}
    n\lambda_{\m} \in \Theta(\max\{\reg,1\})
    \qquad \text{and}
    \qquad \m \in \Theta\left(\frac{n\filsize^{-\delta}}{\max\{\reg,1\}}\right)
        \subseteq \bigO(n \cdot \filsize^{-\bd}).
\end{equation*}
Furthermore, assuming $\Lgt \leq \ceil*{\frac{{\Ln-\Lr-1}}\beta}$,
we have $\lambda_{\m} \in o{\left(\frac{1}{d\cdot \filsize^{\Lgt-1}}\right)}$.
\end{lemma}
We refer to \cref{subsec:proof_eigendecay_main} for the proof. As a result, we can use \crefrange{lem:as_1_verification}{lem:eigendecay_main}
to instantiate \cref{th:vidya} for the setting in \cref{th:rates},
resulting in the following tight high-probability bounds
for variance and bias:
\begin{gather*}
    d^{-\Lr}\filsize^{-\delta} +d^{-\Lr}\filsize^{-(1-\delta)}
    \lesssim\; \var(\freg)/\sigma^2
    \;\lesssim d^{-\Lr}\filsize^{-\delta} +d^{-\Lr}\filsize^{-(1-\delta)}, \\
    d^{-2(\Ln-\Lr-1-\beta(\Lgt-1))}
    \lesssim\; \bias^2(\freg)
    \;\lesssim d^{-2(\Ln-\Lr-1-\beta(\Lgt-1))}.
\end{gather*}
Reformulating the bounds in terms of $n$
then concludes the proof of \cref{th:rates}.

%% file: mainmatter/conclusion.tex
\section{Summary and outlook}
\vspace{-0.1in}
In this paper,
we highlight how the strength of an inductive bias impacts generalization.
Concretely, we study when the gap in test error
between interpolating models and their
optimally ridge-regularized or early-stopped counterparts is zero,
that is, when interpolation is harmless.
In particular, we prove a phase transition for kernel regression
using convolutional kernels with different filter sizes:
a weak inductive bias (large filter size) yields harmless interpolation,
and even requires fitting noise for optimal test performance,
whereas a strong inductive bias (small filter size)
suffers from suboptimal generalization when interpolating noise.
Intuitively, this phenomenon arises from a bias-variance trade-off,
captured by our main result in \cref{th:rates}:
with increasing inductive bias,
the risk on noiseless training samples (bias) decreases,
while the sensitivity to noise in the training data (variance) increases.
Our empirical results on neural networks suggest that this phenomenon
extends to models with feature learning,
which opens up an avenue for exciting future work.

%% file: acknowledgments.tex
\subsubsection*{Acknowledgments}
K.D.\ is supported by the ETH AI Center
and the ETH Foundations of Data Science.
We would further like to thank Afonso Bandeira
for insightful discussions.

%% file: appendix/context_agnostic_proof.tex
\section{Generalization bound for fixed-design}
We prove \cref{th:vidya} by deriving a closed-form expression for the bias and variance,
and bounding them individually.
Hence, the proof does not rely on any matrix concentration results.

\subsection{Proof of \texorpdfstring{\cref{th:vidya}}{Theorem~\ref{th:vidya}}}
\label{app:context_agnostic}
It is well-known
that the KRR problem defined in \cref{eq:krr_def}
yields the estimator
\begin{equation*}
    \freg(x) = \y^\t \H^{-1}\k(x) ,
\end{equation*}
where
$x, x_\iii \sim \nu$,
$\y \defeq \f +\epsilon = [\fstar(x_1) + \epsilon_1,\dots, \fstar(x_n) + \epsilon_n]^\t$,
$\k(x) \coloneqq [\Kf(x_1,x),\dots,\Kf(x_n,x)]^\t$,
and $\H \defeq \K+\reg \I_n$.
For this estimator, both bias and variance exhibit a closed-form expression:
\begin{align}
    \bias^2(\freg)
    &= \mathbb{E}_{x\sim\nu}[
        \left(\fstar(x)- \mathbb{E}_\epsilon[(\f +\epsilon)^\t\H^{-1}\k(x)]\right)^2
    ] \nonumber \\
    &= \mathbb{E}_{x\sim\nu}[\left(\fstar(x)- \f^\t\H^{-1}\k(x)\right)^2] \nonumber \\
    &= \mathbb{E}_{x\sim\nu}\left[\fstar(x)^2 \right]
    - 2\f^\t\H^{-1}\mathbb{E}_{x\sim\nu}[\fstar(x)\k(x)]
    + \f^\t\H^{-1}\mathbb{E}_{x\sim\nu}[\k(x)\k(x)^\t]\H^{-1}\f \nonumber \\
    &\overset{(i)}{=} \a^\t\a-2\a^\t\Pm^\t\H^{-1}\Pm\Dm\a
    + \a^\t\Pm^\t\H^{-1}\S\H^{-1}\Pm\a , \label{eq:closedform_bias} \\
    \var(\freg)/\sigma^2
    &= \frac{1}{\sigma^2}\mathbb{E}_{x\sim\nu}\mathbb{E}_{\epsilon}[\left(
        (\f +\epsilon)^\t\H^{-1}\k(x)- \mathbb{E}_{\epsilon}[(\f +\epsilon)^\t\H^{-1}\k(x)]
    \right)^2] \nonumber \\
    &= \frac{1}{\sigma^2}\mathbb{E}_{x\sim\nu}\mathbb{E}_{\epsilon}[
        \epsilon^\t\H^{-1}\k(x)\k(x)^\t\H^{-1}\epsilon
    ] \nonumber \\
    &= \frac{1}{\sigma^2} \sum_{i,j}\mathbb{E}[\epsilon_i\epsilon_j]\left[
        \H^{-1}\mathbb{E}_{x\sim\nu}[\k(x)\k(x)^\t]\H^{-1}
    \right]_{i,j} \nonumber \\
    &\overset{(i)}{=} \mathbf{\Tr}\left(\H^{-1}\S\H^{-1}\right) . \label{eq:closedform_var}
\end{align}
Step $(i)$ uses the definition of the squared kernel $\Sf$,
that $\fstar(x) = \sum_{\i = 1}^\m \a_\i \Pf_\i(x)$
as $\fstar$ is in the span of the first $\m$ eigenfunctions,
and the following consequences of the eigenfunctions' orthonormality:
\begin{align*}
    \mathbb{E}_{x\sim\nu}\left[\fstar(x)^2 \right] &= \sum_{\i,\i' = 1}^\m \a_\i\a_{\i '}\mathbb{E}_{x\sim\nu}\left[\Pf_\i(x)\Pf_{\i'}(x)\right] = \sum_{\i,\i' = 1}^\m \a_\i\a_{\i '}\delta_{\i,\i'} = \norm{\a}_2^2,\\
    \mathbb{E}_{x\sim\nu}\left[[\fstar(x)\k(x)]_\iii\right] &= \sum_{\i = 1}^\m\sum_{\i' = 1}^\infty  \a_\i\lambda_{\i'} \mathbb{E}_{x\sim\nu}\left[\Pf_\i(x)\Pf_{\i'}(x)\right]\Pf_{\i'}(x_\iii) \\
    &= \sum_{\i = 1}^\m \a_\i\lambda_{\i}\Pf_{\i}(x_\iii) = \left[\Pm\Dm\a\right]_\iii.
\end{align*}
We now bound the closed-form expressions of bias and variance individually.

\paragraph*{Bias}
\Cref{lem:squared_kernel} below
yields $\S = \Pm\Dm^2\Pm^\t + \SM$.
Hence, the bias decomposes into
\begin{align*}
    \bias^2(\freg) &= \a^\t\a-2\a^\t\Pm^\t\H^{-1}\Pm\Dm\a + \a^\t\Pm^\t\H^{-1}\Pm\Dm^2\Pm^\t\H^{-1}\Pm\a \\
    &+ \a^\t\Pm^\t\H^{-1}\SM\H^{-1}\Pm\a\\
    &= \underbrace{\norm*{\left(\I_\m - \Dm\Pm^\t\H^{-1}\Pm\right)\a}^2}_{\defeq B_1}
    + \underbrace{\a^\t\Pm^\t\H^{-1}\SM\H^{-1}\Pm\a}_{\defeq B_2} .
\end{align*}

First, we rewrite $B_1$ as
\begin{align}
    B_1 &= \norm*{\left(\I_\m - \Dm\Pm^\t\H^{-1}\Pm \right)\a}^2   \notag \\
    &\overset{(i)}= \norm*{\left(\Dm - \Dm\Pm^\t\left(\Pm\Dm\Pm^\t + \HM\right)^{-1}\Pm\Dm \right)\Dm^{-1}\a}^2   \notag \\
    &\overset{(ii)}= \norm*{\left(\Dm^{-1} + \Pm^\t\HM^{-1}\Pm \right)^{-1}\Dm^{-1}\a}^2 \label{eq:B1equivalence} ,
\end{align}
where $(i)$ uses the decomposition $\H = \Km+\HM$ with $\HM \defeq \KM + \reg\I_n$,
and $(ii)$ applies the Woodbury matrix identity.

Second, we can upper-bound $B_2$ as follows:
\begin{align*}
    B_2
    &\leq \norm{\SM} \a^\t\Pm^\t\H^{-2}\Pm\a
    = \norm{\SM} \left(\Dm^{-1}\a\right)^\t\Dm\Pm^\t\H^{-2}\Pm\Dm\left(\Dm^{-1}\a\right) \\
    &\overset{(i)}
    = \norm*{\SM} \a^\t\Dm^{-1} \left(\Dm^{-1} + \Pm^\t\HM^{-1}\Pm\right)^{-1}\Pm^\t\HM^{-2}\Pm \\
    &\qquad\qquad\qquad\;\;\;\;\;\;\, %
    \left(\Dm^{-1} + \Pm^\t\HM^{-1}\Pm\right)^{-1}\Dm^{-1}\a \\
    &\leq  \underbrace{\norm*{\SM} \norm*{\Pm^\t\HM^{-2}\Pm}}_{\defeq C_1}
    \underbrace{\norm*{\left(\Dm^{-1} + \Pm^\t\HM^{-1}\Pm\right)^{-1}\Dm^{-1}\a}^2}_{B_1},
\end{align*}
where $(i)$ uses \cref{lem:20_from_bartlett}.
Thus, the bias can be bounded by $B_1(1+C_1)$ with $C_1 \geq 0$.

We proceed by upper bounding $C_1$:
\begin{align}
    1+C_1 &\leq 1+n\norm*{\SM} \norm*{\frac{\Pm^\t\Pm}{n}} \frac{1}{\mineig{\HM}^2} \notag\\
    &\overset{(i)}{\leq} 1+1.5 \frac{
        n \lambda_{\m + 1}\norm*{\KM}
    }{
        \left(\mineig{\KM}+\reg\right)^2
    }
    \leq 1+1.5 \frac{
        n \lambda_{\m +1}\left(\norm*{\KM}+\reg\right)
    }{
        \left(\mineig{\KM}+\reg\right)^2
    } \notag\\
    &\leq 1
        + 1.5 \frac{n\lambda_{\m+1}}{\max\{\reg,1\}}
        \frac{\max\{\reg,1\}^2}{(\mineig{\KM}+\reg)^2}
        \frac{\norm*{\KM}+\reg}{\max\{\reg,1\}} \notag \\
    &\leq 1
        + 1.5 \frac{\rmax}{\rmin^2}
        \frac{n\lambda_{\m+1}}{\max\{\reg,1\}} \notag \\
    &\overset{(ii)}\leq
        \left(1+1.5\frac{\rmax}{\rmin^2}\right)
        \max\left\{\frac{n\lambda_{\m+1}}{\max\{\reg,1\}},1\right\} \nonumber \\
    &= \left(1+1.5\frac{\rmax}{\rmin^2}\right) \tmax, \label{eq:C1}
\end{align}
where $(i)$ uses \cref{eq:ass1}
to bound $\norm{\Pm^\t\Pm/n}$
and \cref{lem:squared_kernel_tail} to bound $\norm{\SM}$,
and $(ii)$ follows from $cx+1\leq (c+1)\max\{x,1\}$ for $c \geq 0$.

Hence, to conclude the bias bound,
we need to bound $B_1$ in \cref{eq:B1equivalence} from above and below.

\pseudoparagraph{Upper bound}
\begin{align}
    B_1
    &\leq \norm*{\left(\frac{\Dm^{-1}}n + \frac{\Pm^\t\HM^{-1}\Pm}n\right)^{-1}}^2
    \frac{\norm{{\Dm^{-1}\a}}^2}{n^2} \notag\\
    &\leq \frac{
        \norm{\HM}^2
    }{
        \mineig{\frac{\Pm^\t\Pm}n}^2
    }
    \frac{\norm{{\Dm^{-1}\a}}^2}{n^2} \notag\\
    &\overset{(i)}{\leq} 4 \left(\norm{\KM}+\reg\right)^2
    \frac{\norm{{\Dm^{-1}\a}}^2}{n^2}\notag\\
    &=4\rmax^2\max\{\reg,1\}^2\frac{\norm{{\Dm^{-1}\a}}^2}{n^2}\label{eq:B1}
\end{align}
where $(i)$ follows from \cref{eq:ass1}.

Combining \cref{eq:C1} in $(i)$ and \cref{eq:B1} in $(ii)$
yields the desired upper bound on the bias:
\begin{equation*}
    \bias^2 \leq B_1 +B_2 \leq (1+C_1) B_1
    \overset{(i)} \leq \left(1+1.5\frac{\rmax}{\rmin^2}\right) \tmax B_1
    \overset{(ii)} \leq
        4 \left(1+1.5\frac{\rmax}{\rmin^2}\right) \rmax^2 \tmax
        \max\{\reg,1\}^2 \frac{\norm{{\Dm^{-1}\a}}^2}{n^2}.
\end{equation*}

\pseudoparagraph{Lower bound}
\begin{align*}
    B_1
    &\geq \mineig{\left(
        \frac{\Dm^{-1}}n + \frac{\Pm^\t\HM^{-1}\Pm}n
    \right)^{-1}}^2
    \frac{\norm{{\Dm^{-1}\a}}^2}{n^2} \\
    & \geq\frac{
        1
    }{
        \left({\frac{1}{n\lambda_\m} +\frac{1}{\mineig{\HM}}\norm*{\frac{\Pm^\t\Pm}n}}\right)^2
    }
    \frac{\norm{{\Dm^{-1}\a}}^2}{n^2} \\
    &\overset{(i)}{\geq} \left(
        \frac{
            \frac{n\lambda_\m}{\max\{\reg,1\}}
        }{
            1.5\frac{\max\{\reg,1\}}{\mineig{\KM}+\reg}\frac{n\lambda_\m}{\max\{\reg,1\}} + 1
        }
    \right)^2
    \max\{\reg,1\}^2
    \frac{\norm{{\Dm^{-1}\a}}^2}{n^2}\\
    & \overset{(ii)}\geq \left(\frac{\rmin}{1.5+\rmin}\right)^2
    \min\left\{\frac{n\lambda_\m}{\max\{\reg,1\}},1\right\}^2
    \max\{\reg,1\}^2
    \frac{\norm{{\Dm^{-1}\a}}^2}{n^2} ,\\
    &= \left(\frac{\rmin}{1.5+\rmin}\right)^2 \tmin^2
    \max\{\reg,1\}^2
    \frac{\norm{{\Dm^{-1}\a}}^2}{n^2} ,
\end{align*}
where $(i)$ follows from \cref{eq:ass1},
and $(ii)$ from the fact that $\frac{x}{cx + 1} \geq \frac{1}{1+c}\min\{x,1\}$
for $x, c \geq 0$.
Since $\bias^2 \geq B_1$, this concludes the lower bound for the bias.

\paragraph*{Variance}
As for the bias bound,
we first apply \cref{lem:squared_kernel}
to write $\S = \Pm\Dm^2\Pm^\t + \SM$
and decompose the variance in \cref{eq:closedform_var} into
\begin{equation}
    \var(\freg)/\sigma^2
    = \underbrace{\Tr\left(\H^{-1}\Pm\Dm^2\Pm^\t\H^{-1}\right)}_{\defeq V_1}
    + \underbrace{\Tr\left(\H^{-1}\SM\H^{-1}\right)}_{\defeq V_2}.
\end{equation}

Next, we rewrite $V_1$ as follows:
\begin{align*}
    V_1
    &= \Tr\left(\H^{-1}\Pm\Dm^2\Pm^\t\H^{-1}\right)
        = \Tr\left(\Dm\Pm^\t\H^{-2}\Pm\Dm\right) \\
    &\overset{(i)}{=}
        \Tr\left(
            \left(\Dm^{-1} + \Pm^\t\HM^{-1}\Pm\right)^{-1}
            \Pm^\t\HM^{-2}\Pm
            \left(\Dm^{-1} + \Pm^\t\HM^{-1}\Pm\right)^{-1}
        \right) \\
    &= \frac{1}n \Tr\left(
        \left(\frac{\Dm^{-1}}n + \frac{\Pm^\t\HM^{-1}\Pm}n\right)^{-1}
        \frac{\Pm^\t\HM^{-2}\Pm}n
        \left(\frac{\Dm^{-1}}n + \frac{\Pm^\t\HM^{-1}\Pm}n\right)^{-1}
    \right),
\end{align*}
where $(i)$ follows from \cref{lem:20_from_bartlett}.

To bound $V_1$ and $V_2$,
we use the fact that the trace is the sum of all eigenvalues.
Therefore, the trace is bounded from above and below
by the size of the matrix
times the largest and smallest eigenvalue, respectively.
This yields the following bounds for $V_1$:
\begin{align*}
    V_1
    &\leq \frac{\m}{n} \norm*{\left(\frac{\Dm^{-1}}n + \frac{\Pm^\t\HM^{-1}\Pm}n\right)^{-1}\frac{\Pm^\t\HM^{-2}\Pm}n\left(\frac{\Dm^{-1}}n + \frac{\Pm^\t\HM^{-1}\Pm}n\right)^{-1}
    } \\
    & \overset{(i)}{\leq}
        \frac{\m}{n} 4\rmax^2\max\{\reg,1\}^2
        \frac{\norm{\Pm^\t\Pm/n}}{\mineig{\HM}^2}
    \overset{(ii)} \leq 6\frac\m n \rmax^2 \frac{\max\{\reg,1\}^2}{\left(\mineig{\KM}+\reg\right)^2}  \\
    & = 6\left(\frac{\rmax}{\rmin}\right)^2\frac \m n ,
\end{align*}
and
\begin{align*}
    V_1 &\geq \frac{\m}{n} \mineig{\left(\frac{\Dm^{-1}}n + \frac{\Pm^\t\HM^{-1}\Pm}n\right)^{-1}\frac{\Pm^\t\HM^{-2}\Pm}n\left(\frac{\Dm^{-1}}n + \frac{\Pm^\t\HM^{-1}\Pm}n\right)^{-1}} \\
    &\overset{(iii)}{\geq} \frac{\m}{n} \left(\frac{\rmin}{1.5+\rmin}\right)^2 \tmin^2
    \max\{\reg,1\}^2  \frac{\mineig{\Pm^\t\Pm/n}}{\left(\norm{\KM}+\reg\right)^2} \\
    &\overset{(iv)}\geq \frac{1}{2}\frac{\m}{n} \left(\frac{\rmin}{1.5+\rmin}\right)^2 \tmin^2
    \frac{\max\{\reg,1\}^2}{\left(\norm{\KM}+\reg\right)^2}
    = \frac{1}{2}\frac{\m}{n} \left(\frac{\rmin}{\rmax(1.5+\rmin)}\right)^2 \tmin^2 .
\end{align*}
For $(i)$ and $(iii)$, we bound the terms analogously
to the upper and lower bound of $B_1$,
and use \cref{eq:ass1} in $(ii)$ and $(iv)$.

Next, the upper bound on $V_2$ follows from a special case of Hölder's inequality as follows:
\begin{align*}
    V_2 &= \Tr\left(\H^{-1}\SM\H^{-1}\right)
        \leq \Tr\left(\SM\right)\norm{\H^{-2}} \\
    &= \frac{\Tr\left(\SM\right)}{\mineig{\Km+\HM}^2}
    \leq \frac{\Tr\left(\SM\right)}{\mineig{\HM}^2}
    = \frac{\Tr\left(\SM\right)}{\rmin^2\max\{\reg,1\}^2}.
\end{align*}
For the lower bound, we need a more accurate analysis.
First, we apply the identity
\begin{equation}\label{eq:HMdec}
    \H^{-1} = (\HM + \Km)^{-1}
    = \HM^{-1} - \underbrace{\HM^{-1}\Km(\HM + \Km)^{-1}}_{\defeq \Am},
\end{equation}
which is valid since $\H$ and $\HM$ are full rank.

Next, note that the rank of $\Am$ is bounded by the rank of $\Km$,
which can be written as $\Pm\Dm\Pm^\t$ and therefore has itself rank at most $\m$.
Furthermore, \cref{eq:ass1} implies that $\Pm^\t\Pm$ has full rank,
and hence $\m \leq n$.

Let now $\{\v_1,\dots,\v_{\rank(\Am)}\}$ be an orthonormal basis
of $\col(\Am)$,
and let $\{\v_{\rank(\Am)+1},\dots,\v_{n}\}$ be an orthonormal basis
of $\col(\Am)^\perp$.
Hence, $\{\v_1,\dots,\v_n\}$ is an orthonormal basis of $\R^n$,
and similarity invariance of the trace yields
\begin{align*}
    V_2 &= \Tr\left(\H^{-1}\SM\H^{-1}\right)
        = \sum_{\iii = 1}^n \v_\iii^\t\H^{-1}\SM\H^{-1}\v_\iii \\
    &\geq \sum_{\iii = \rank(\Am)+1}^n \v_\iii^\t\H^{-1}\SM\H^{-1}\v_\iii \\
    &\overset{(i)}= \sum_{\iii = \rank(\Am)+1}^n{
        \v_\iii^\t(\HM^{-1} - \Am)\SM(\HM^{-1} - \Am)\v_\iii
    } \\
    & \overset{(ii)}{=} \sum_{\iii = \rank(\Am)+1}^n \v_\iii^\t \HM^{-1}\SM\HM^{-1}\v_\iii \\
    &{\geq} \frac{1}{\norm{\HM}^2}\sum_{\iii = \rank(\Am)+1}^n \v_\iii^\t\SM\v_\iii , \\
    &= \frac{1}{\norm{\HM}^2}\Tr\left(\Proj^\t\SM\Proj\right),
\end{align*}
where $\Proj$ is the projection matrix of $\R^n$ onto $\col(\Am)^\perp$,
and $(i)$ follows from \cref{eq:HMdec}.
Step $(ii)$ uses that, for all $i>\rank(\Am)$,
$\v_\iii$ is orthogonal to the column space of $\Am$,
and hence
\begin{align*}
    &\v_\iii^\t(\HM^{-1} - \Am)\SM(\HM^{-1} - \Am)\v_\iii \\
    = &\v_\iii^\t \HM^{-1}\SM\HM^{-1}\v_\iii
    - \underbrace{ \v_\iii^\t \Am}_{=0}\SM\HM^{-1}\v_\iii
    - \v_\iii^\t \HM^{-1}\SM\underbrace{\Am \v_\iii}_{=0}
    + \underbrace{\v_\iii^\t\Am\SM \Am\v_\iii}_{=0} .
\end{align*}

Finally, let $\eigval_\iii{(\cdot)}$ be the $\iii$-th eigenvalue
of its argument with respect to a decreasing order.
Then, the Cauchy interlacing theorem yields
$\eigval_{\rank(\Am) + \iii}{(\Sm)} \leq \eigval_\iii{(\Proj^\t\SM\Proj)}$
for all $\iii = 1,\dots, n - \rank(\Am)$.
This implies
\begin{align*}
    \Tr\left(\Proj^\t\SM\Proj\right)
    &= \sum_{\iii = 1}^{n-\rank(\Am)}{\eigval_\iii{(\Proj^\t\SM\Proj)}}
    \geq \sum_{\iii = 1}^{n-\rank(\Am)}{\eigval_{\rank(\Am) + \iii}{(\SM)}}
    \\
    &= \sum_{\iii = 1+\rank(\Am)}^{n}{\eigval_{\iii}{(\SM)}}
    \overset{(i)}{\geq} \sum_{\iii = 1+\m}^{n}{\eigval_\iii{(\SM)}},
\end{align*}
where $(i)$ uses that the rank of $\Am$ is bounded by $\m$.
This concludes the lower bound on $V_2$ as follows:
\begin{equation*}
    V_2 = \Tr\left(\H^{-1}\SM\H^{-1}\right)
    \geq \frac{\sum_{\iii = \m+1}^{n}{\eigval_\iii{(\Sm)}}}{\norm{\HM}^2}
    = \frac{\sum_{\iii = \m+1}^{n}{\eigval_\iii{(\Sm)}}}{\rmax^2\max\{\reg,1\}^2}.
\end{equation*}

\subsection{Technical lemmas}
\begin{lemma}[Squared kernel decomposition]
\label{lem:squared_kernel}
Let $\Kf$ be a kernel function that under a distribution $\nu$
can be decomposed as $\Kf(x,x') = \sum_\i \lambda_\i \Pf_\i(x)\Pf_\i(x')$,
where $\E_{x\sim\nu}[\Pf_\i(x)\Pf_{\i'}(x)]= \delta_{\i,\i'}$.
Then, the squared kernel $\Sf(x,x') = \E_{z\sim\nu}[\Kf(x,z)\Kf(x',z)]$ can be written as
\begin{equation*}
    \Sf(x,x') = \sum_\i \lambda_\i^2 \Pf_\i(x)\Pf_\i(x'),
\end{equation*}
and for any $\m > 0$,
the corresponding kernel matrix
can be written as $\S = \Pm\Dm^2\Pm^\t + \SM$.
\end{lemma}
\begin{proof}
    The statement simply follows from
    \begin{equation*}
        \Sf(x,x') = \sum_{\i,\i'}{
            \lambda_\i\lambda_{\i'}
            \Pf_{\i}(x)
            \underbrace{\E_{z\sim \nu}\left[\Pf_\i(z)\Pf_{\i'}(z)\right]}_{\delta_{\i,\i'}}
            \Pf_{\i'}(x')
        }
        = \sum_\i \lambda_\i^2 \Pf_\i(x)\Pf_\i(x').
    \end{equation*}
\end{proof}

\begin{lemma}[Corollary of Lemma~20 in \cite{bartlett_2020}]
\label{lem:20_from_bartlett}
    \begin{align*}
        \Dm\Pm^\t\H^{-2}\Pm\Dm = &\left(\Dm^{-1} + \Pm^\t\HM^{-1}\Pm\right)^{-1}\Pm^\t\HM^{-2}\Pm\\
        &\left(\Dm^{-1} + \Pm^\t\HM^{-1}\Pm\right)^{-1}
    \end{align*}
\end{lemma}
\begin{proof}
    \begin{align*}
        &\Dm\Pm^\t\H^{-2}\Pm\Dm = \Dm^{1/2}\Dm^{1/2}\Pm^\t\H^{-2}\Pm\Dm^{1/2}\Dm^{1/2}\\
        &\overset{(i)}{=} \Dm^{1/2}\left(\I_\m + \Dm^{1/2}\Pm^\t\HM^{-1}\Pm\Dm^{1/2}\right)^{-1}\Dm^{1/2}\Pm^\t\HM^{-2}\Pm\Dm^{1/2}\\
        &\qquad\qquad\left(\I_\m + \Dm^{1/2}\Pm^\t\HM^{-1}\Pm\Dm^{1/2}\right)^{-1}\Dm^{1/2}\\
        &=\left(\Dm^{-1} + \Pm^\t\HM^{-1}\Pm\right)^{-1}\Pm^\t\HM^{-2}\Pm\left(\Dm^{-1} + \Pm^\t\HM^{-1}\Pm\right)^{-1},
    \end{align*}
    where $(i)$ applies Lemma~20 from \cite{bartlett_2020}.
\end{proof}

\begin{lemma}[Squared kernel tail]
    \label{lem:squared_kernel_tail}
    For $\m > 0$, let $\SM$ be the kernel matrix of the truncated squared kernel
    $\SfM = \sum_{\i > m}{\lambda_\i^2 \Pf_\i(x)\Pf_\i(x')}$,
    and let $\KM$ be the kernel matrix of the truncated original kernel
    $\KfM = \sum_{\i > m}{\lambda_\i \Pf_\i(x)\Pf_\i(x')}$.
    Then,
    \begin{equation*}
        \norm*{\SM} \leq \lambda_{\m+1} \norm*{\KM} .
    \end{equation*}
\end{lemma}
\begin{proof}
    We show that for any vector $v$, $v^\t\SM v\leq \lambda_{\m+1}v^\t\KM v$,
    which implies the claim.
    To do so, we define $\P_\i \in \R^{n\times n}$
    with $[\P_\i]_{i,j} = \Pf_\i(x_i)\Pf_\i(x_j)$ for all $\i > \m$.
    Then we can write $\KM = \sum_{\i > \m} \lambda_\i\P_\i$
    and $\SM = \sum_{\i > \m} \lambda_\i^2\P_\i$.
    Since the eigenvalues are in decreasing order,
    we have $\lambda_\i \leq \lambda_{\m+1}$ for any $\i>m$, and thus
    \begin{equation*}
        v^\t\SM v
        = \sum_{\i > \m} \lambda_\i^2 v^\t \P_\i v
        \leq \lambda_{\m+1} \sum_{\i > \m} \lambda_\i v^\t \P_\i v
        = \lambda_{\m+1}v^\t\KM v .
    \end{equation*}
\end{proof}

%% file: appendix/hypercube.tex
\section{Convolutional kernels on the hypercube}
\label{app:hypercube}

First, \cref{ssec:hypercube_hypercube}
provides a way to decompose general functions for features
distributed uniformly on the hypercube.
Next, \cref{ssec:hypercube_conv} uses those results
to characterize the eigenfunctions and eigenvalues
of cyclic convolutional kernels.
Finally, \cref{subsec:proof_ground_truth_rate,subsec:proof_eigendecay_main}
apply this characterization to prove
\cref{lem:ground_truth_rate,lem:eigendecay_main},
respectively.

\subsection{General functions on the hypercube}
\label{ssec:hypercube_hypercube}
This subsection focuses on the main setting in our paper:
the hypercube domain $\{-1,1\}^d$
together with the uniform probability distribution,
previously studied in \cite{misiakiewicz2021learning}.
For any $S \subseteq \{1,\dots,d\}$, we define
the polynomial
\begin{equation}\label{eq:Y_def}
    \Yf_S(x) \coloneqq \prod_{j \in S}[x]_j
\end{equation}
of degree $\abs{S}$,
where $[x]_j$ is the $j$-th entry of $x \in \2^d$.
It is easy to see that $\{\Yf_S\}_{S \subseteq \{1,\dots,d\}}$
is set of orthonormal functions
with respect to the inner product
$\inner{f,g}_{\2^d} \coloneqq \E_{x\sim \U(\2^d)}[f(x)g(x)]$.
Those functions
play a key role in the remainder of our proof;
as it turns out, they are the eigenfunctions
of the kernel in \cref{subsec:setting}.
Towards a formal statement,
define the polynomials
\begin{align}
    \Gf_{\ii}^{(d)}\left(
        \frac{\innersize{x,x'}{}}{\sqrt{d}}
    \right) &\coloneqq
    \frac{1}{{\B(\ii, d)}}
    \sum_{\abs{S} = \ii}{
        \Yf_S(x)\Yf_S(x')
    }, \label{eq:G_def} \\
    \text{where} \; \B(\ii,d) &\coloneqq \abs{\left\{S \subseteq \{1,\dots,d\}\mid \abs{S} = \ii\right\}} = \binom{d}{\ii}.\label{eq:b_formula}
\end{align}
Note that $\Gf_{\ii}^{(d)}$ only depends on the (Euclidean) inner product of $x$ and $x'$.
Furthermore, $\{\Gf_{\ii}^{(d)}\}_{\ii = 0}^d$ is a set of orthonormal polynomials with respect to the distribution
of $\innersize{x,x'}{}/{\sqrt{d}}$.
The following lemma shows how such polynomials form an eigenbasis
for functions that only depend on the inner product between points
in the unit hypercube.
\begin{lemma}[Local kernel decomposition]
\label{lem:loc_kern_decomposition}
Let $\kappa\colon \R \to \R$ be any function
and $d \in \N_{>0}$.
Then, for any $x,x' \in \2^d$,
we can decompose $\kappa(\innersize{x,x'}{}/d)$ as
\begin{equation}
    \label{eq:kappa_decomposition}
    \kappa\left(\frac{\innersize{x,x'}{}}{d}\right)
    = \sum_{\ii = 0}^d{
        \xi_\ii^{(d)}
        \frac{1}{\B(\ii,d)}
        \sum_{\abs{S}=\ii}{\Yf_S(x)\Yf_S(x')}
    }.
\end{equation}
\end{lemma}
\begin{proof}
    Note that the decomposition only needs to hold
    at the evaluation of $\kappa$ in the values
    that $\innersize{x,x'}{} / d$ can take,
    that is, $\kappa$
    computed in $\{-1, -1 + 2/d, \dots, -2/d + 1, 1\}$.
    Since that set has a cardinality $d+1$,
    we can write $\kappa$ as a linear combination
    of $d+1$ uncorrelated functions.
    In particular, $\{\Gf_{\ii}^{(d)}\}_{\ii = 0}^d$ is a set of such functions
    with respect to the distribution of $\innersize{x,x'}{} / \sqrt{d}$,
    and hence
    \begin{equation*}
    \kappa\left(\frac{\innersize{x,x'}{}}{d}\right)
    = \sum_{\ii = 0}^{d}
    \c_\ii \Gf_{\ii}^{(d)}\left(
        \frac{\innersize{x,x'}{}}{\sqrt{d}}
    \right)
    \end{equation*}
    for some (unknown) coefficients $\c_\ii$.
    Finally, the proof follows by expanding the definition of $\Gf_{\ii}^{(d)}$
    in \cref{eq:G_def}
    and choosing $\xi_\ii^{(d)} = \c_\ii$.
\end{proof}

\subsection{Convolutional kernels on the hypercube}
\label{ssec:hypercube_conv}
While the previous subsection considers general functions on the hypercube,
we now focus on convolutional kernels and their eigenvalues.
This yields the tools to prove \cref{lem:ground_truth_rate,lem:eigendecay_main}.
In order to characterize eigenvalues,
we first follow existing literature such as \cite{misiakiewicz2021learning}
and introduce useful quantities.

Let $S \subseteq \{1,\dots,d\}$.
The diameter of $S$ is
\begin{equation*}
    \diam(S) \coloneqq \max_{i,j \in S}{
        \min{\{
            \module{j-i}{d}+1,\module{i-j}{d}+1
        \}}
    }
\end{equation*}
for $S \neq \emptyset$, and $\diam(\emptyset) = 0$.
Furthermore, we define
\begin{equation}
    \label{eq:C_def}
    \C(\ii, \filsize, d)
    \coloneqq \abs*{
        \left\{S \subseteq \{1,\dots,d\}
        \mid \abs{S} = \ii, \diam(S) \leq \filsize
        \right\}
    }.
\end{equation}
Intuitively, the diameter of $S$
is the smallest number of contiguous feature indices
that fully contain $S$.
The following lemma yields an explicit formula for $\C(\ii, \filsize, d)$,
that is,
the number of sets of size $\ii$
with diameter at most $\filsize$.

\begin{lemma}[Number of overlapping sets]
    \label{lem:n_sets}
    Let $\ii, \filsize, d \in \N$
    with $\ii \leq \filsize < d/2$.
    Then,
    \begin{equation*}
        \C(\ii, \filsize, d) = \begin{cases}
            d\binom{\filsize-1}{\ii-1} &\ii > 0, \\
            1 &\ii = 0.
        \end{cases}
    \end{equation*}
\end{lemma}
\begin{proof}
    Since the result holds trivially
    for $\ii = 0$ and $\ii = 1$,
    we henceforth focus on $\ii \geq 2$.
    Let $\tilde\C(\ii, \diam, d)$
    be the number of subsets $S \subseteq \{1,\dots,d\}$
    of cardinality $\abs{S} = \ii$
    with diameter exactly $\diam(S) = \diam$.
    First, consider $\tilde\C(2, \diam, d)$.
    For each set, we can choose the first element $i$ from $d$ different values,
    and the second as $\module{(i - 1) \mathpm (\diam - 1)}{d} + 1$.
    In this way, since $\filsize < d/2$, we count each set exactly twice.
    Thus,
    \begin{equation*}
        \tilde\C(2, \diam, d) = \frac{d\cdot 2}2 = d.
    \end{equation*}

    Next, consider $\ii > 2$.
    We can build all the possible sets by
    starting with one of the $\tilde\C(2, \diam, d) = d$ sets,
    and adding the remaining $\ii-2$ elements
    from $\diam - 2$ possible indices.
    Hence, every fixed set of size $2$ and diameter $\diam$
    yields $\binom{\diam-2}{\ii-2}$ different sets of size $\ii$.
    Furthermore, by construction,
    every set of size $\ii$ and diameter $\diam$ results from exactly one
    set of size $2$ and diameter $\diam$.
    Therefore,
    \begin{equation*}
        \tilde\C(\ii, \diam, d) = d\binom{\diam-2}{\ii-2}.
    \end{equation*}

    The result for $\ii \geq 2$
    then follows from summing $\tilde\C(\ii, \diam, d)$
    over all diameters $\diam \leq \filsize$:
    \begin{equation*}
        \C(\ii, \filsize, d)
        = d\sum_{\diam = \ii}^\filsize{\binom{\diam-2}{\ii-2}}
        \overset{(i)}{=} d\binom{\filsize-1}{\ii-1},
    \end{equation*}
    where $(i)$ follows from the hockey-stick identity.
\end{proof}

Now, we focus on cyclic convolutional kernels $\Kf$ as in \cref{eq:conv_kern_def}.
First, we restate Proposition~1 from \cite{misiakiewicz2021learning}.
This proposition establishes that $\Yf_S$ for $S \subseteq \{1, \dotsc, d\}$
are indeed the eigenfunctions of $\Kf$,
and yields closed-form eigenvalues $\lambda_{S}$
up to factors $\xi^{(\filsize)}_{\abs{S}}$ that depend on
the inner nonlinearity $\kappa$.
Next, \cref{lem:eigendecay} uses additional regularity assumptions
on $\kappa$ to eliminate the dependency on $\xi^{(\filsize)}_{\abs{S}}$.
This characterization of the eigenvalues then enables
the proof \cref{lem:ground_truth_rate,lem:eigendecay_main}.

\begin{proposition}[Proposition~1 from \cite{misiakiewicz2021learning}]
    \label{prop:kern_eig}
    Let $\Kf$ be a cyclic convolutional kernel over the unit hypercube
    as defined in \cref{eq:conv_kern_def}.
    Then,
    \begin{equation*}
        \Kf(x,x')
        \coloneqq \frac{1}{d}
        \sum_{k = 1}^d{
            \kappa\left(
                \frac{\innersize{x_{(k, \filsize)},{x'}_{(k, \filsize)}}{\filsize}}{\filsize}
            \right)
        }
        =\sum_{\ii = 0}^\filsize{
            \sum_{\substack{\diam(S)\leq \filsize \\ \abs{S} = \ii}}{
                \lambda_S \Yf_S(x)\Yf_S(x')
            }
        },
    \end{equation*}
    with
    \begin{equation*}
        \lambda_S = \xi^{(\filsize)}_{\abs{S}}
        \frac{\filsize + 1 - \diam(S)}{d\B(\abs{S},\filsize)}
    \end{equation*}
    where $\xi^{(\filsize)}_{\abs{S}}$ are the coefficients
    of the $\kappa$-decomposition (\cref{eq:kappa_decomposition})
    over $\2^\filsize$.
    Alternatively,
    \begin{equation*}
        \Kf(x,x') = \sum_{\i}{\lambda_{S_\i} \Yf_{S_\i}(x)\Yf_{S_\i}(x')}
    \end{equation*}
    where we order all $S_\i \subseteq \{1,\dots,d\}$ with $\diam(S_\i) \leq \filsize$
    such that $\lambda_{S_\i} \geq \lambda_{S_{\i+1}}$.
    In particular, $\lambda_{S_\i} = \lambda_{\i}$.
\end{proposition}
We refer to Proposition~1 from \cite{misiakiewicz2021learning}
for a formal proof.
Intuitively, the result follows from applying \cref{lem:loc_kern_decomposition}
for each subset $S \subseteq \{1, \dotsc, d\}$ of contiguous elements.
This is possible because crucially,
any subset of $\filsize$ features
is again distributed uniformly on the $\filsize$-dimensional unit hypercube.
Lastly, the factor $\filsize + 1 - \diam(S)$
stems from the fact that each eigenfunction $\Yf_S$ appears as many times
as there are contiguous index sets of size $\diam(S)$
supported in a fixed contiguous index set of size $\filsize$.
In other words, the term is the number of shifted instances of $S$
supported in a contiguous subset of $\filsize$ features.

As mentioned before, \cref{prop:kern_eig} characterizes the eigenvalues
of cyclic convolutional kernels $\Kf$
up to factors $\xi^{(\filsize)}_{\abs{S}}$ that depend on
the inner nonlinearity $\kappa$.
To avoid the additional factors,
we require the following regularity assumptions:
\begin{assumption}[Regularity]
    \label{ass:genericity}
    Let $\T \defeq \ceil*{4+\frac{4\Ln}{\beta}}$.
    A cyclic convolutional kernel
    $
        \Kf(x,x')
        =\frac{1}{d}
        \sum_{k = 1}^d{
            \kappa\left(
                \frac{\innersize{x_{(k, \filsize)},{x'}_{(k, \filsize)}}{\filsize}}{\filsize}
            \right)
        }
    $
    from the setting of \cref{subsec:setting}
    with inner function $\kappa$ satisfies the regularity assumption
    if there exist constants $\c \geq \T$, $\c', \c'' > 0$
    and a series of constants $\{\c_\ii > 0\}_{\ii=0}^{\T}$ such that,
    for any $\filsize \geq \c$,
    the decomposition
    \begin{equation*}
        \Kf(x,x') = \sum_{\ii = 0}^\filsize\sum_{\substack{\diam(S)\leq \filsize \\ \abs{S} = \ii}}\xi^{(\filsize)}_{\abs{S}}
    \frac{\filsize + 1 - \diam(S)}{d\B(\abs{S},\filsize)} \Yf_S(x)\Yf_S(x')
    \end{equation*}
    from \cref{prop:kern_eig} over inputs $x,x' \in \2^d$
    satisfies
    \begin{align}
        \xi_\ii^{(\filsize)} &\geq \c_{\ii} \quad
        \forall\ii \in \{0,\dots, \T \}\label{eq:ass_lb},\\
        \xi_\ii^{(\filsize)} &\geq 0 \quad
        \forall\ii > \T \label{eq:ass_psd}, \\
        \xi_{\filsize-\ii}^{(\filsize)} &\leq \frac{\c'}{\filsize^{\T-\ii +1}} \quad
        \forall\ii \in \{0,\dots, \T \}, \label{eq:ass_ub_1}\\
        \sum_{\ii \geq 0} \xi_\ii^{(\filsize)}&\leq \c''. \label{eq:ass_ub_2}
    \end{align}
\end{assumption}

For sufficiently high-dimensional inputs $x, x'$,
\cref{eq:ass_lb,eq:ass_psd} ensure that the convolutional kernel
$\Kf(x,x')$ in \cref{eq:conv_kern_def}
is a valid kernel,
and that it can learn polynomials of degree up to $\T$.
Indeed, if $\xi_\ii^{(\filsize)} = 0$ for some $\l$,
then there are no polynomials of degree $\ii$ among the eigenfunctions of $\Kf$.
Furthermore, \cref{eq:ass_lb,eq:ass_ub_1,eq:ass_ub_2} guarantee
that the eigenvalue tail is sufficiently bounded.
This allows us to bound $\norm{\KM}$ and $\mineig{\KM}$ in \cref{app:matrix_concentration}.

Our assumption resembles Assumption~1 by \citet{misiakiewicz2021learning}:
For one, \cref{eq:ass_ub_1,eq:ass_ub_2} are equivalent to Equations~43~and~44
in \cite{misiakiewicz2021learning}.
Furthermore, \cref{eq:ass_lb} above is a slightly stronger version
of Equation~42,
where strengthening is necessary due to the non-asymptotic nature of our results.

We still argue that many standard $\kappa$,
for example, the Gaussian kernel,
satisfy \cref{ass:genericity} with our convolution kernel $\Kf$.
Because such $\kappa$ satisfy Assumption~1 in \cite{misiakiewicz2021learning},
we only need to check that they additionally satisfy our \cref{eq:ass_lb}.
If $\kappa$ is a smooth function,
we have $\xi_{\ii}^{(\filsize)} = \kappa^{(\ii)}(0) + o(1)$ for all $\ii\leq \T$,
where $\kappa^{(\ii)}$ is the $\ii$-th derivative of $\kappa$.
In particular, all derivatives of the exponential function at $0$ are strictly positive,
implying \cref{eq:ass_lb} for the Gaussian kernel
if $d$ is large enough.

The final lemma of this section is a corollary that
characterizes the eigenvalues solely in terms of $\abs{S}$,
and further shows that, for $d$ large enough,
the eigenvalues decay as $\abs{S}$ grows.

\begin{lemma}[Corollary of \cref{prop:kern_eig}]
\label{lem:eigendecay}
Consider a cyclic convolutional kernel as in \cref{prop:kern_eig}
that satisfies \cref{ass:genericity} with $\filsize \in \Theta(d^\beta)$ for some $\beta \in (0,1)$.
Then, for any $S\subseteq \{1,\dots,d\}$ such that $\diam(S) \leq \filsize$ and $\abs{S} < \T$,
the eigenvalue $\lambda_S$ corresponding to the eigenfunction $\Yf_S(x)$ satisfies
\begin{equation}\label{eq:beforeT}
    \lambda_{S} \in \Omega\left(\frac{1}{d\cdot \filsize^{\abs{S}}}\right)
    \quad \text{and} \quad
    \lambda_{S} \in \bigO\left(\frac{1}{d\cdot \filsize^{\abs{S}-1}}\right).
\end{equation}
Furthermore,
\begin{equation}\label{eq:afterT}
     \max_{\substack{\abs{S} \geq \T \\ \gamma(S)\leq \filsize}}\lambda_{S}
    \in \bigO\left(\frac{1}{d\cdot \filsize^{\T-1}}\right).
\end{equation}

\end{lemma}
\begin{proof}
Without loss of generality,
assume $d$ is large enough such that
$d > \filsize / 2 \geq \c / 2$,
where $\c$ is a constant from \cref{ass:genericity}.

Let $S\subseteq \{1,\dots,d\}$ with $\diam(S) \leq \filsize$
be arbitrary and define
$\ii \defeq \abs{S}$
and $\rmult \defeq \filsize + 1 - \diam(S)$.
Since $\ii\leq \diam(S) \leq \filsize$, we have
\begin{equation*}
    1 \leq \rmult \leq \filsize + 1 - \ii.
\end{equation*}
Furthermore, since $\B(\ii, \filsize) = \binom{\filsize}{\ii}$,
we use the following classical bound on the binomial coefficient
throughout the proof:
\begin{equation*}
    \left(\frac{1}{\ii}\right)^\ii \filsize^\ii
    \leq \B(\ii, \filsize)
    \leq \left(\frac{e}{\ii}\right)^\ii \filsize^\ii .
\end{equation*}

For the first part of the lemma, assume $\abs{S} < \T$.
Then, using \cref{ass:genericity}, we have
\begin{align*}
    \lambda_{S}
    = \frac{\xi_\ii^{(\filsize)}\rmult}{d\B(\ii, \filsize)}
    &\overset{(i)}{\leq} \frac{\ii^\ii}{\filsize^\ii} \frac{\c'' \rmult}{d}
    \leq \c_{\ii,1} \frac{1}{dq^{\ii - 1}}
    \leq \c_{\T,1} \frac{1}{dq^{\ii - 1}},\\
    &\overset{(ii)}{\geq}
    \frac{\ii^\ii}{\filsize^\ii e^\ii}
    \frac{\c_\ii \rmult}{d}
    \geq
    \c_{\ii,2} \frac{1}{dq^{\ii}}
    \geq \c_{\T,2} \frac{1}{dq^{\ii}}
\end{align*}
for some positive constants $\c_{\ii,1}, \c_{\ii,2}$
that depend on $\ii$,
$\c_{\T,1} \defeq \max_{\ii \in \{0, \dotsc, \T - 1\}}{\c_{\ii,1}}$,
and
$\c_{\T,2} \defeq \min_{\ii \in \{0, \dotsc, \T - 1\}}{\c_{\ii,2}}$.
Step $(i)$ follows from the upper bound in \cref{eq:ass_ub_2}
with non-negativity in \cref{eq:ass_lb,eq:ass_psd},
and $(ii)$ follows from the lower bound in \cref{eq:ass_lb}.
Since $\c_{\T,1}$ and $\c_{\T,2}$ do not depend on $\ii$,
this concludes the first part of the proof.

For the second part of the proof,
we consider two cases
depending on whether $\abs{S} \in [\T, \filsize-\T]$
or $\abs{S} > \filsize-\T$.

Hence, first assume $\T\leq\abs{S} \leq \filsize-\T$. Then,
\begin{equation}\label{eq:mid_part}
    \lambda_{S}
    \overset{(i)}{=} \xi_{\abs{S}}\frac{\filsize+1-\gamma(S)}{d\binom{\filsize}{\abs{S}}}
    \overset{(ii)}\leq \xi_{\abs{S}}\frac{\filsize+1-\gamma(S)}{d\binom{\filsize}{\T}}
    \overset{(iii)}\leq \c'' \frac{\filsize}{d \left(\frac \filsize \T\right)^\T}
    = \frac{\c'' \T^\T}{d\filsize^{\T-1}},
\end{equation}
where $(i)$ follows from \cref{prop:kern_eig}.
In step $(ii)$, we use that $\binom{q}{\abs{S}}$ is minimized
when $\abs{S}$ is has the largest difference to $\filsize/2$.
Lastly, step $(iii)$ applies the upper bound from \cref{eq:ass_ub_2}
together with non-negativity of $\xi_{\abs{S}}$ in \cref{ass:genericity},
and the classical bound on the binomial coefficient.

Now assume$\abs{S} > \filsize-\T$. Then,
\begin{align}
    \lambda_S&\overset{(i)}{=} \xi_{\abs{S}}\frac{\filsize+1-\gamma(S)}{d\binom{\filsize}{\abs{S}}}
    \leq \xi_{\filsize-(\filsize-\abs{S})}\frac{\filsize}{d\binom{\filsize}{\filsize-(\filsize-\abs{S})}} \nonumber \\
    &\overset{(ii)}{=} \xi_{\filsize-(\filsize-\abs{S})}\frac{\filsize}{d\binom{\filsize}{\filsize-\abs{S}}}\notag\\
    &\overset{(iii)}{\leq}
        \xi_{\filsize-(\filsize-\abs{S})}
    \frac{
        (\filsize-\abs{S})^{\filsize-\abs{S}} \filsize
    }{
        d \filsize^{\filsize-\abs{S}}
    } \nonumber \\
    &\overset{(iv)}\leq \c' \frac{(\filsize-\abs{S})^{\filsize-\abs{S}}\filsize}{d \filsize^{\T-(\filsize-\abs{S})+1}\filsize^{\filsize-\abs{S}}}
    \overset{(v)}\leq \frac{\c' \T^\T}{d\filsize^{\T}},\label{eq:late_part}
\end{align}
where $(i)$ follows from \cref{prop:kern_eig},
$(ii)$ from the fact that $\binom{n}{n-k} = \binom{n}{k}$,
$(iii)$ from the classical bound for the binomial coefficient,
$(iv)$ from \cref{eq:ass_ub_1} in \cref{ass:genericity},
and $(v)$ from the fact that $\filsize-\abs{S} < \T$ in the current case.

Combining \cref{eq:mid_part,eq:late_part} from the two cases finally yields
\begin{align*}
    \max_{\substack{\abs{S}\geq \T\\\gamma(S)\leq q}}\lambda_{S}
     &\leq \max\left\{
        \max_{\substack{\T\leq\abs{S} \leq \filsize-\T\\\gamma(S)\leq q}}\lambda_{S},
        \max_{\substack{\abs{S} > \filsize-\T\\\gamma(S)\leq q}}\lambda_{S}
    \right\}\\
    &\leq \max\left\{\frac{\c'' \T^\T}{d\filsize^{\T-1}}, \frac{\c' \T^\T}{d\filsize^{\T}}\right\}
    \in \bigO\left(\frac{1}{d\filsize^{\T-1}}\right),
\end{align*}
which concludes the second part of the proof.
\end{proof}

\subsection{Proof of \texorpdfstring{\cref{lem:ground_truth_rate}}{Lemma~\ref{lem:ground_truth_rate}}}
\label{subsec:proof_ground_truth_rate}
    First, note that $\fstar = \Yf_{S^*}$
    for $S^* = \{1, \dots, \Lgt\}$.
    Since $\abs{S^*} = \diam(S^*) = \Lgt$,
    \cref{prop:kern_eig} yields
    \begin{equation*}
        \lambda_{S^*}
        = {\xi_{\Lgt}^{(\filsize)}} \frac{\filsize +1 - \Lgt}{d\B(\Lgt,\filsize)}
        \overset{(i)}\in \Theta\left(\frac{\filsize}{d \filsize^{\Lgt}}\right)
        = \Theta\left(\frac{1}{d \filsize^{\Lgt-1}}\right)
        \overset{(ii)}{\subseteq} \omega(\lambda_\m),
    \end{equation*}
    where $(i)$ uses \cref{eq:ass_lb,eq:ass_ub_2} in \cref{ass:genericity}
    for $\Lgt \leq \T$ and $d$ large enough
    to get $\xi_{\Lgt}^{(\filsize)}\in \Theta(1)$,
    and $(ii)$ uses $\lambda_\m\in o\left(\frac{1}{d \filsize^{\Lgt-1}}\right)$.
    Hence, for $d$ sufficiently large,
    $\lambda_{S^*}>\lambda_\m$.
    Since the eigenvalues are in decreasing order,
    this implies that $\fstar$ is in the span of the first $\m$ eigenfunctions.
    This further yields
    \begin{equation*}
        \norm{\Dm^{-1}\a}
        = \lambda_{S^*}^{-1}
        \in \Theta\left(d \filsize^{\Lgt-1}\right),
    \end{equation*}
    since the entry of $\a$ corresponding to $\Yf_{S^*}$ is $1$ while all others are $0$.

\subsection{Proof of \texorpdfstring{\cref{lem:eigendecay_main}}{Lemma~\ref{lem:eigendecay_main}}}
\label{subsec:proof_eigendecay_main}

Before proving \cref{lem:eigendecay_main},
we introduce the following quantity:
\begin{equation}
    \label{eq:L}
    \L \defeq\floor*{\frac{\Ln-\Lr-1}{\beta}}.
\end{equation}
Intuitively, $\L$ corresponds to the degree of the largest polynomial
that a cyclic convolutional kernel as defined in \cref{eq:conv_kern_def}
can learn.
This quantity plays a key role throughout the proof of \cref{lem:eigendecay_main},
and \cref{lem:combined} later.
Finally, note that $\delta$ as defined in \cref{th:rates}
can be written as $\delta = \frac{\Ln-\Lr -1}\beta - \L$.

\begin{proof}[Proof of \cref{lem:eigendecay_main}]
First, we use \cref{prop:kern_eig} to write the cyclic convolutional kernel as
\begin{equation*}
    \Kf(x,x')
    = \sum_{\ii = 0}^\filsize{
        \sum_{\substack{\diam(S)\leq \filsize \\ \abs{S} = \ii}}{
            \lambda_S \Yf_S(x)\Yf_S(x')
        }
    } = \sum_{\i}{\lambda_{S_\i}  \Yf_{S_\i}(x)\Yf_{S_\i}(x')}
\end{equation*}
where the $\lambda_{S_\i}$ are ordered such that $\lambda_{S_{\i+1}} \leq \lambda_{S_\i}$.

For the first part of the proof,
we need to pick an $\m \in \N$
such that $n\lambda_{\m} \in \Theta\left(\max\{\reg,1\}\right)$
and $\m \in \Theta\left(\frac{n\filsize^{-\delta}}{\max\{\reg,1\}}\right)$.
We will equivalently choose $\m \in \Theta(d \filsize^{\L})$
with $\lambda_\m \in \Theta\left(\frac{1}{d \cdot \filsize^{\L+\delta}}\right)$;
since $n \in \Theta(d^\Ln)$ and $\max\{\reg,1\}\in \Theta(d^\Lr)$,
we have
\begin{align*}
    \Theta\left(\frac{\max\{\reg,1\}}{n}\right)
    =\Theta\left(\frac{d^{\Lr}}{d^{\Ln}}\right)
    =\Theta\left(\frac{d^{\Lr}}{d^{1+\Lr + \beta (\L+\delta)}}\right)
    = \Theta\left(\frac{1}{d \cdot \filsize^{\L+\delta}}\right),\\
    \Theta\left(
        \frac{n\filsize^{-\delta}}{\max\{\reg,1\}}
    \right)
    = \Theta\left(\frac{d^{\Ln}\filsize^{-\delta}}{d^{\Lr}}\right)
    = \Theta\left(\frac{d^{1+\Lr + \beta (\L+\delta)}\filsize^{-\delta}}{d^{\Lr}}\right)
    = \Theta(d \filsize^{\L}).
\end{align*}

The remainder of the proof proceeds in five steps:
we first construct a candidate $S_\m \subseteq \{1,\dots,d\}$
with $\diam(S_\m) \leq \filsize$,
show that the rate of the eigenvalue
corresponding to $\Yf_{S_\m}$ satisfies
$\lambda_{S_\m} = \lambda_\m \in \Theta\left(\frac{1}{d \cdot \filsize^{\L+\delta}}\right)$,
show that the rate of $\m \in \Theta(d \filsize^{\L})$,
establish $\Theta\left(\frac{n\filsize^{-\delta}}{\max\{\reg,1\}}\right)
\subseteq \bigO(n \cdot \filsize^{-\bd})$,
and finally show that $\lambda_{\m} \in o{\left(\frac{1}{d\cdot \filsize^{\Lgt-1}}\right)}$
for appropriate $\Lgt$.

\paragraph*{Construction of $\m$}
We consider two different $S_\m$ depending on $\delta$:
\begin{equation}\label{eq:SM}
    S_\m = \begin{cases}
        \{1,\dots,L,\floor{\filsize+1-\filsize^{1-\delta}}\} & \delta \in (0,1)\\
        \{1,\dots,L,\floor{\filsize/2}\} & \delta =0.
    \end{cases}
\end{equation}
For $d$---and hence $\filsize \in \Theta(d^\beta)$---large enough,
$S_\m$ is well-defined,
$\abs{S_\m} = \L + 1$,
and the diameter is
\begin{equation*}
    \diam(S_\m)=\begin{cases}
        \floor{\filsize+1-\filsize^{1-\delta}} & \delta \in (0,1)\\
        \floor{\filsize/2} & \delta =0 .
    \end{cases}
\end{equation*}
For the rest of the proof, assume that $d$ is sufficiently large.

\paragraph*{Rate of $\lambda_{S_\m}$}
Using \cref{prop:kern_eig} and $\abs{S_\m}=\L+1$,
we can write
\begin{equation*}
    \lambda_{S_\m}
    = \xi^{(\filsize)}_{\L + 1}
    \frac{\filsize + 1 - \diam(S_\m)}{d\B(\L + 1,\filsize)} .
\end{equation*}
First, we show that the numerator is in $\Theta\left(q^{1-\delta}\right)$
for both definitions of $S_\m$.
In the case where $\delta \in (0,1)$, we have
\begin{equation*}
    \filsize+1-\floor{\filsize+1-\filsize^{1-\delta}}
    = -\floor{-\filsize^{1-\delta}} = \ceil{q^{1-\delta}}
    \overset{(i)}\in \Theta\left(q^{1-\delta}\right),
\end{equation*}
where $(i)$ follows from $\delta < 1$ and $\filsize$ sufficiently large.
In the case where $\delta = 0$, we have
\begin{align*}
    \filsize+1-\floor{\filsize/2} &\leq \filsize + 1 \in \bigO(\filsize), \\
    \filsize+1-\floor{\filsize/2} &\geq \filsize/2 \in \Omega(\filsize).
\end{align*}
Thus, since $\delta = 0$ in this case,
the numerator is in $\Theta(\filsize) = \Theta(\filsize^{1-\delta})$.

As the denominator does not depend on $\delta$,
we use the same technique for both $\delta = 0$ and $\delta \in (0,1)$.
The classical bound on $\B(L+1,\filsize) = \binom{\filsize}{\L+1}$ yields
\begin{equation*}
    \filsize^{\L+1}
    \lesssim \left(\frac{\filsize}{\L+1}\right)^{\L+1}
    \leq \binom{\filsize}{\L+1}
    \leq e^{\L + 1} \left(\frac{\filsize}{\L+1}\right)^{\L+1}
    \lesssim \filsize^{\L+1}.
\end{equation*}
Therefore, $d\B(\L + 1,\filsize) \in \Theta\left(d \filsize^{\L + 1}\right)$.

Finally, since $\L+1 \leq \T = \ceil*{4+4\Ln/\beta}$,
we have $ \xi^{(\filsize)}_{\L+1} \in \Theta(1)$
by \cref{eq:ass_lb,eq:ass_ub_2} in \cref{ass:genericity}
for $d$ sufficiently large.
Combining all results then yields the desired rate of $\lambda_{S_\m}$ as follows:
\begin{equation}
    \lambda_{S_\m}
    \in \Theta\left(\frac{\filsize^{1-\delta}}{d\filsize ^{L+1}}\right)
    = \Theta\left(\frac{1}{d \cdot \filsize^{\L+\delta}}\right)
    \label{eq:lambda_m_rate} .
\end{equation}

\paragraph*{Rate of $\m$}
To establish $\m \in \Theta(d \filsize^{\L})$, we bound $\m$ individually from above and below.

\pseudoparagraph{Upper bound}
Since the eigenvalues are in decreasing order,
we can bound $\m$ from above by counting how many eigenvalues
are larger than $\lambda_\m$.
To do so, we use $\abs{S_\m} = \L+1$,
and show that for $d$ sufficiently large,
all $S_\i$ with $\abs{S_\i} > \L + 1$
correspond to eigenvalues $\lambda_{S_\i} < \lambda_{S_\m}$.
We first decompose
\begin{align*}
    \max_{\substack{\i : \abs{S_\i} > \L+1 \\ \diam(S_\i) \leq \filsize}}{\lambda_{S_\i}}
    = \max\left\{
        \underbrace{
            \max_{
                \substack{\i : \L+1 < \abs{S_\i} <\T \\ \diam(S_\i) \leq \filsize}
            }\lambda_{S_\i}
        }_{\revdefeq M_1},
        \underbrace{
            \max_{
                \substack{\i : \abs{S_\i} \geq \T \\ \diam(S_\i) \leq \filsize}
            }\lambda_{S_\i}
        }_{\revdefeq M_2}
    \right\}.
\end{align*}
For $M_1$, let $\i$ with $\L + 1 < \abs{S_\i} <\T$ be arbitrary. Then,
\begin{align*}
    \lambda_{S_\i}
    &\overset{(i)}{\in} \bigO\left(\frac{1}{d\filsize^{\abs{S_\i}- 1}}\right)
    \subseteq \bigO\left(\frac{1}{d\filsize^{(\L+2)-1}}\right)
    = \bigO\left(\frac{1}{d\filsize^{\L + 1}}
    \right) \overset{(ii)}\subseteq o(\lambda_{S_\m}),
\end{align*}
where we apply \cref{eq:beforeT} from \cref{lem:eigendecay} in $(i)$,
and use \cref{eq:lambda_m_rate} with $\delta < 1$ in $(ii)$.
This implies $M_1 \in o(\lambda_{S_\m})$.

For $M_2$, we directly get
\begin{align*}
    M_2 = \max_{
        \substack{\i : \abs{S_\i} \geq \T \\ \diam(S_\i) \leq \filsize}
    }\lambda_{S_\i}
    \overset{(i)}{\in} \bigO\left(\frac{1}{d\filsize^{\T- 1}}\right)
    \overset{(ii)}{\subseteq} \bigO\left(\frac{1}{d\filsize^{(\L+2)- 1}}\right)
    = \bigO\left(\frac{1}{d\filsize^{\L+1}}\right)
    \overset{(iii)}\subseteq  o(\lambda_{S_\m}),
\end{align*}
where we apply \cref{eq:afterT} from \cref{lem:eigendecay} in $(i)$,
$(ii)$ follows from $\L+2 \leq \T$,
and step $(iii)$ uses \cref{eq:lambda_m_rate} with $\delta < 1$.

Combined, we have
$\max_{\substack{\i : \abs{S_\i} > \L+1, \diam(S_\i) \leq \filsize}}{\lambda_{S_\i}}
= \max{\{M_1, M_2\}} \in o(\lambda_{S_\m})$.
Thus, for $d$ sufficiently large and $\abs{S_\i} > \L + 1$,
we have $\lambda_{S_\i} < \lambda_{S_\m}$.
For this reason, $\m$ is at most the number of eigenfunctions
with degree no larger than $\L + 1$:
\begin{equation*}
    \m
    \leq \sum_{\ii = 0}^{\L+1}{\C(\ii, \filsize, d)}
    \overset{(i)}{\in} \bigO(d\filsize^{\L}),
\end{equation*}
where $(i)$ uses \cref{lem:n_sets} for $d$ large enough
with $\binom{\filsize}{\ii} \in \Theta(\filsize^\ii)$.

\pseudoparagraph{Lower bound}
By construction of $S_\m$ in \cref{eq:SM},
we have $\diam(S_\m) \geq \floor{\filsize/2}$.
This, combined with \cref{prop:kern_eig},
implies that the indices of all polynomials
with degree $\L+1$ but diameter at most $\floor{\filsize/2}-1$ are smaller than $\m$.
Hence, for large enough $d$, \cref{lem:n_sets} yields the following lower bound:
\begin{equation*}
    \m \geq \C(\L+1, \floor{\filsize/2} - 1, d)
    = d\binom{ \floor{\filsize/2}-2}{\L}
    \geq d\left(\frac{ \floor{\filsize/2}-2}{\L}\right)^\L \in \Omega(d\filsize^\L) .
\end{equation*}

The upper and lower bound together then imply $ \m \in \Theta(d\filsize^\L)$.
This concludes the existence of an $\m \in \N$ such that $\lambda_\m$ and $\m$
exhibit the desired rates.

\paragraph*{Rate of $\m$ with respect to $n$}
We can write $n$ as
\begin{equation*}
    n \in \Theta(d^\Ln)
    = \Theta\left(d\cdot d^{\beta\frac{\Ln-1}{\beta}}\right)
    = \Theta\left(
        d \filsize^{
            \floor*{\frac{{\Ln-1}}\beta}
            + \left(\frac{{\Ln-1}}\beta-\floor*{\frac{{\Ln-1}}\beta}\right)
        }
    \right)
    = \Theta\left(d \filsize^{\floor*{\frac{{\Ln-1}}\beta} + \bd}\right) .
\end{equation*}
Combining this with $\L \leq \floor*{\frac{{\Ln-1}}\beta}$,
we directly get
$\Theta(d\filsize^\L) \subseteq \bigO(d\filsize^{\floor*{\frac{{\Ln-1}}\beta}})
= \bigO(n \filsize^{-\bd})$.

\paragraph*{Rate of $\lambda_\m$ for appropriate $\Lgt$}
Since $n\lambda_{\m} \in \Theta(\max\{\reg,1\})$,
we have
\begin{equation*}
    \lambda_{\m}
    \in \Theta\left(\frac{\max\{\reg,1\}}{n}\right)
    \overset{(i)}{=} \Theta\left(
        \frac{d^{\Lr}}{d\cdot d^{\Lr} \filsize^{\L + \delta}}
    \right)
    = \Theta\left(\frac{1}{d \filsize^{\L + \delta}}\right),
\end{equation*}
where $(i)$ uses the identity $\Ln = 1 + \Lr + \beta (\L + \delta)$.
Assume now $\Lgt \leq \ceil*{\frac{{\Ln-\Lr-1}}\beta}$.
For the remainder, we need to consider two cases depending on $\delta$.

If $\delta > 0$, then $\Lgt \leq \L + 1$, and we have
\begin{equation*}
    \lambda_{\m}
    \in\bigO\left(\frac{\filsize^{-\delta}}{d \filsize^{\Lgt-1}}\right)
    \subseteq o\left(\frac{1}{d \filsize^{\Lgt-1}}\right).
\end{equation*}
If $\delta = 0$, then $\Lgt \leq \L$, and we have
\begin{equation*}
    \lambda_{\m} \in\bigO\left(\frac{\filsize^{-\delta}}{d \filsize^{\L}}  \right)\subseteq\bigO\left(\frac{1}{d \filsize^{\Lgt}}  \right) \subseteq o\left(\frac{1}{d \filsize^{\Lgt-1}}  \right).
\end{equation*}
In both cases,
$\lambda_{\m}\in o\left(\frac{1}{d \filsize^{\Lgt-1}}\right)$,
concluding the proof.
\end{proof}

%% file: appendix/matrix_concentration.tex
\section{Matrix concentration}
\label{app:matrix_concentration}
This section considers the random matrix theory part of our main result.
First, \cref{ssec:matrix_concentration_early} focuses on the large eigenvalues
of our kernel, and proves \cref{lem:as_1_verification}.
Next, \cref{ssec:12decomposition,ssec:proof_lemcombined} focus on the tail
of the eigenvalues,
culminating in the proof of \cref{lem:combined}.
Lastly, \cref{ssec:technical_lemmas,ssec:rmt_lemmas}
establishes some technical tools that we use throughout the proofs.

\subsection{Proof of \texorpdfstring{\cref{lem:as_1_verification}}{Lemma~\ref{lem:as_1_verification}}}
\label{ssec:matrix_concentration_early}
In this proof, we show that the matrix $\Pm^\t\Pm/n$
concentrates around the identity matrix
for all $\m \in \bigO(n d^{-\beta\bd})$,
thereby establishing \cref{eq:ass1}.
Let $\mlarge$ be the largest $\m \in \bigO(n\filsize^{-\bd})$.
The proof consists of applying Theorem~5.44 from \cite{Vershynin2012}
to the matrix $\P_{\leq\mlarge}$,
and extending the result to all suitable choices of $\m$ simultaneously.

More precisely, let $\tilde{\c}$ be the implicit constant
of the $\bigO(n \filsize^{-\bd})$-notation,
and define $\mlarge$ to be the largest $\m \in \N$
with $\m \leq \tilde{\c} \cdot n\filsize^{-\bd}$.
Note that $\mlarge$ exists,
because $d$ is large enough and fixed.

\paragraph*{Bound for $\mlarge$}
To apply Theorem~5.44 from \cite{Vershynin2012},
we need to verify the theorem's conditions on the rows of $\P_{\leq\mlarge}$.
In particular, we show that the rows are independent,
have a common second moment matrix,
and that their norm is bounded.
Let $[\P_{\leq\mlarge}]_{\iii,:}$ indicate
the $\iii$-th row of $\P_{\leq\mlarge} \in \R^{n\times \mlarge}$.
We may write each row entry-wise as
\begin{equation*}
    [\P_{\leq\mlarge}]_{\iii,:}
    = \left[\Yf_{1}(x_\iii)\;\Yf_{2}(x_\iii)\;\cdots\;\Yf_{\mlarge}(x_\iii)\right]^\t .
\end{equation*}

First, the rows of $\P_{\leq\mlarge}$ are independent,
since each row depends on a different $x_\iii$,
and we assume the data to be i.i.d..

Second, since the eigenfunctions are orthonormal w.r.t.\ the data distribution,
the second moment of the rows is
$\E\left[[\P_{\leq\mlarge}]_{i,:}[\P_{\leq\mlarge}]_{i,:}^\t\right] = \I_\mlarge$
for all rows $\iii \in \{1, \dotsc, n\}$.

Third, to show that each row has a bounded norm,
we use the fact that the eigenfunctions $\Yf_\i$ in \cref{eq:Y_def}
over $\2^d$ satisfy $\Yf_\i(x_i)^2 = 1$ for all $\i$.
Thus, the norm of each row is
\begin{equation*}
    \norm*{[\P_{\leq\mlarge}]_{i,:}}_2
    = \sqrt{\sum_{\i = 1}^\mlarge \Yf_\i(x_i)^2}
    = \sqrt{\sum_{\i = 1}^\mlarge 1} = \sqrt{\mlarge} .
\end{equation*}

We can now apply Theorem~5.44 from \cite{Vershynin2012}.
For any $t\geq0$, this yields the following inequality
with probability $1-\mlarge\exp\{-ct^2\}$,
where $c$ is an absolute constant:
\begin{equation*}
    \norm*{\frac{\P_{\leq\mlarge}^\t\P_{\leq\mlarge}}{n}-\I_\mlarge}
    \leq \max\left\{\norm{\I_\mlarge}^{\frac{1}{2}}\Delta, \Delta^2\right\},
    \quad \text{where } \Delta = t\sqrt{\frac{\mlarge}{n}}.
\end{equation*}
The choice $t = \frac{1}{2}\sqrt{\frac{n}{\mlarge}}$
yields $\max\left\{\norm{\I_\mlarge}^{\frac{1}{2}}\Delta, \Delta^2\right\} = 1/2$,
and the following error probability for large enough $d$:
\begin{equation*}
    \mlarge \exp\left\{-\c \frac{n}{4\mlarge}\right\}
    \overset{(i)}{\lesssim} n \filsize^{-\bd} \exp\left\{-\c' \frac{n}{n \filsize^{-\bd}}\right\}
    \lesssim \filsize^{- \bd} \cdot d^{\Ln} \exp\{-\c' \filsize^{\bd}\}
    \lesssim \filsize^{-\bd},
\end{equation*}
where $(i)$ follows from $\mlarge \in \bigO(n \filsize^{-\bd})$.

\paragraph*{Bound for any $\m < \mlarge$}
Note that $\Pm^\t\Pm$ is a submatrix of $\P_{\leq\mlarge}^\t\P_{\leq\mlarge}$.
Thus,
\begin{equation*}
    \frac{\Pm^\t\Pm}n-\I_{\m} \quad \text{is also a submatrix of}\quad\frac{\P_{\leq\mlarge}^\t\P_{\leq\mlarge}}{n}-\I_\mlarge.
\end{equation*}
Therefore,
\begin{equation*}
    \norm*{\frac{\Pm^\t\Pm}n-\I_{\m}}
    \leq \norm*{\frac{\P_{\leq\mlarge}^\t\P_{\leq\mlarge}}{n}-\I_\mlarge}
    \leq 1/2
\end{equation*}
with probability at least $1-\c d^{-\beta\bd}$
uniformly over all $\m\leq \mlarge$.

\subsection{
Further decomposition of the terms after \texorpdfstring{$\m$}{\m}
}
\label{ssec:12decomposition}
In this section, we focus on the concentration
of the smallest and largest eigenvalue of the kernel matrix $\KM$
to prove \cref{lem:combined}.
However, this proof is involved, and requires additional tools.
In particular, we further decompose $\KfM$
into two kernels $\Kf_1$ and $\Kf_2$.

In the following, we consider the setting of \cref{th:rates}
with a convolutional kernel $\Kf$ that satisfies \cref{ass:genericity}.
We define the additional notation
\begin{equation*}
    \L \defeq\floor*{\frac{\Ln-\Lr-1}{\beta}}
    \quad \text{and} \quad
    \bL \defeq \floor*{\frac{\Ln-1}{\beta}}.
\end{equation*}
Intuitively, $\L$ is the maximum polynomial degree
that $\Kf$ can learn with regularization,
and $\bL$ is the analogue without regularization.
Finally, note that $\delta$ and $\bd$ as defined in \cref{th:rates}
can be written as $\delta = \frac{\Ln-\Lr -1}\beta - \L$
and $\bd = \frac{\Ln - 1}\beta - \bL$,
respectively.

We now introduce the two additional kernels,
and then show in \cref{lem:K1K2} that $\KfM = \Kf_1 + \Kf_2$.
First, applying \cref{prop:kern_eig} to $\Kf$ yields
\begin{equation}\label{eq:Kf_Si_dec}
    \Kf(x,x') = \sum_{\i}{\lambda_{S_\i} \Yf_{S_\i}(x)\Yf_{S_\i}(x')},
\end{equation}
where $\{S_\i\}_{\i > 0}$ is a sequence of all subsets $S_\i \subseteq \{1,\dots,d\}$
with $\diam(S_\i) \leq \filsize$,
ordered such that $\lambda_{S_\i} \geq \lambda_{S_{\i+1}}$.
Next, let $\m \in \N$ be such that $n\lambda_{S_\m}\in\Theta(\max\{\reg,1\})$,
and define the index sets
\begin{align*}
    \IndS_1 &\defeq \{\i \in \N \mid \i > \m \;\text{ and }\; \abs{S_\i} \leq \bL + 1\}, \\
    \IndS_2 &\defeq \{\i \in \N \mid \abs{S_\i} \geq \bL + 2 \}.
\end{align*}
Those sets induce the following kernels:
\begin{align*}
    \Kf_1(x,x') &\defeq \sum_{\i \in \IndS_1}\lambda_{S_\i} \Yf_{S_\i}(x)\Yf_{S_\i}(x'),
    & \Sf_1(x,x') &\defeq \sum_{\i \in \IndS_1}\lambda_{S_\i}^2 \Yf_{S_\i}(x)\Yf_{S_\i}(x'),\\
    \Kf_2(x,x') &\defeq \sum_{\i \in \IndS_2}\lambda_{S_\i} \Yf_{S_\i}(x)\Yf_{S_\i}(x'),
    & \Sf_2(x,x') &\defeq \sum_{\i \in \IndS_2}\lambda_{S_\i}^2 \Yf_{S_\i}(x)\Yf_{S_\i}(x'),
\end{align*}
where $\Sf_1$ and $\Sf_2$ are the squared kernels
corresponding to $\Kf_1$ and $\Kf_2$, respectively.
The empirical kernel matrices
$\K_1, \K_2,\S_1,\S_2 \in \R^{n \times n}$ are
\begin{equation*}
    [\K_1]_{i,j} = \Kf_1(x_i, x_j),
    \quad [\K_2]_{i,j} = \Kf_2(x_i, x_j),
    \quad [\S_1]_{i,j} = \Sf_1(x_i, x_j),
    \quad\text{and}\quad [\S_2]_{i,j} = \Sf_2(x_i, x_j).
\end{equation*}
Furthermore, as in the original kernel decomposition, we define the matrices
\begin{align*}
    \P_1 & \in \R^{n \times \abs{\IndS_1}},
    &[\P_1]_{i,j} &= \Yf_{S_{\i_j}}(x_i),\\
    \D_1 &\in \R^{\abs{\IndS_1} \times \abs{\IndS_1}},
    &\D_1 &= \diag(\lambda_{S_{\i_1}}, \dots, \lambda_{S_{\i_{\abs{\IndS_1}}}}),
\end{align*}
where $\{{\i_j}\}_{j = 1}^{\abs{\IndS_1}}$ is a sequence of all indices in $\IndS_1$
ordered such that $\lambda_{S_{\i_j}} \geq \lambda_{S_{\i_{j+1}}}$.
Intuitively, $\P_1,\D_1$ are the analogue
to $\Pm,\Dm$ in the original decomposition $\Kf = \Kfm + \KfM$.

Lastly, we define $\bm$ as the largest eigenvalue corresponding
to an eigenfunction $\Yf_S$ of degree $\abs{S} \geq \bL + 2$, that is,
\begin{equation*}
    \bm \defeq \min \IndS_2 .
\end{equation*}

Using the previous definitions,
the following lemma establishes that $\Kf_1$ and $\Kf_2$
indeed constitute a decomposition of $\KfM$.
\begin{lemma}[$1$-$2$ decomposition]
    \label{lem:K1K2}
    For $d$ sufficiently large,
    we have
    \begin{equation*}
        \Kf_{>\m}(x,x') = \Kf_1(x,x') + \Kf_2(x,x')
        \quad\text{and}\quad \Sf_{>\m}(x,x') = \Sf_1(x,x') + \Sf_2(x,x') .
    \end{equation*}
\end{lemma}
\begin{proof}
    For the decomposition of $\Kf_{>\m}$,
    we have to show that exactly the eigenfunctions with index larger than $m$
    appear in either $\Kf_1$ or $\Kf_2$,
    that is, $\IndS_1 \cup \IndS_2 = \{\i > \m\}$,
    and that no eigenfunction appears in both $\Kf_1$ or $\Kf_2$,
    that is, $\IndS_1 \cap \IndS_2 = \emptyset$.
    Furthermore,
    since we can write
    $\Sf_{>\m}(x,x') = \sum_{\i > \m}\lambda_{S_\i}^2 \Yf_{S_\i}(x)\Yf_{S_\i}(x')$
    by \cref{lem:squared_kernel},
    the same argument implies the $1$-$2$ decomposition of $\Sf_{>\m}$.

    First, from the definition of $\IndS_1$ and $\IndS_2$,
    it follows directly that $\IndS_1 \cap \IndS_2 = \emptyset$,
    that $\IndS_1 \cup \IndS_2 \supseteq \{\i > \m\}$,
    and that $\IndS_1 \subseteq \{\i > \m\}$.
    Hence, to conclude the proof, we only need to show that $\IndS_2 \subseteq \{\i > \m\}$.
    Since the eigenvalues are sorted in decreasing order,
    we equivalently show that, for $d$ sufficiently large,
    all eigenvalues $\lambda_{S_\i}$ with $\i \in\IndS_2$
    are smaller than $\lambda_{S_\m}\in\Theta(\max\{\reg,1\} / n)$.

    More precisely, we show that
    $\max_{\i \in \IndS_2}n\lambda_{S_\i} \in o(n\lambda_{S_\m}) = o(\max\{\reg,1\})$.
    Using $\T$ from \cref{ass:genericity}, we have
    \begin{align*}
        \max_{\i \in \IndS_2}n\lambda_{S_\i} = \max\left\{\underbrace{
            \max_{\i \in \IndS_2 : \abs{S_\i} <\T}n\lambda_{S_\i}}_{\revdefeq M_1},\underbrace{
            \max_{\i \in \IndS_2 : \abs{S_\i} \geq \T}n\lambda_{S_\i}}_{\revdefeq M_2}
        \right\}.
    \end{align*}
    For $M_1$, we bound a generic $\i \in \IndS_2$ with $ \abs{S_\i} <\T$
    as follows:
    \begin{align*}
        n\lambda_{S_\i}
        &\overset{(i)}{\in} \bigO\left(\frac{n}{d\filsize^{\abs{S_\i}- 1}}\right)
        \subseteq \bigO\left(\frac{n}{d\filsize^{(\bL+2)- 1}}\right)
        = \bigO\left(
            \frac{d^{1+\Lr+\beta(\L+\delta)}}{d^{1+\beta(\bL+\delta)}}d^{-\beta(1-\delta)}
        \right) \\
        &= \bigO(\max\{\reg,1\}d^{\beta(\L-\bL)}d^{-\beta(1-\delta)})
        \overset{(ii)}{\subseteq} o(n\lambda_{\m}),
    \end{align*}
    where $(i)$ applies \cref{eq:beforeT} from \cref{lem:eigendecay},
    and $(ii)$ uses $\bL\geq\L$ and $\delta < 1$.
    In particular, this implies $M_1 \in o(n\lambda_\m)$.

    For $M_2$ we have
    \begin{align*}
        \max_{\i \in \IndS_2 : \abs{S_\i} \geq \T}n\lambda_{S_\i}
        &= n \max_{\substack{\abs{S_\i} \geq \T \\ \diam(S_\i) \leq \filsize}}\lambda_{S_\i}
        \overset{(i)}{\in} \bigO\left(\frac{n}{d\filsize^{\T- 1}}\right) \\
        &\overset{(ii)}{\subseteq} \bigO\left(\frac{n}{d\filsize^{(\bL+2)- 1}}\right)
        = \bigO\left(
            \frac{d^{1+\Lr+\beta(\L+\delta)}}{d^{1+\beta(\bL+\delta)}}d^{-\beta(1-\delta)}
        \right) \\
        &= \bigO(\max\{\reg,1\}d^{\beta(\L-\bL)}d^{-\beta(1-\delta)})
        \overset{(iii)}{\subseteq} o(n\lambda_{\m}),
    \end{align*}
    where $(i)$ applies \cref{eq:afterT} from \cref{lem:eigendecay},
    $(ii)$ uses that $\bL+2 \leq \T$,
    and $(iii)$ follows from $\bL\geq\L$ and $\delta < 1$.

    Combining the bounds on $M_1$ and $M_2$,
    we have $\max_{\i \in \IndS_2}n\lambda_{S_\i} \in o(n\lambda_{S_\m}) = o(\max\{\reg,1\})$.
    Hence, for $d$ sufficiently large,
    all $\i \in\IndS_2$ yield
    $\lambda_{S_\i} < \lambda_{S_\m}$
    and consequently $\i > \m$.
\end{proof}

Using the $1$-$2$ decomposition,
we now prove \cref{lem:combined}.
We defer the auxiliary \crefrange{lem:K1}{lem:Q11}
to \cref{ssec:technical_lemmas},
and concentration-results to \cref{ssec:rmt_lemmas}.

\subsection{Proof of \texorpdfstring{\cref{lem:combined}}{Lemma~\ref{lem:combined}}}
\label{ssec:proof_lemcombined}
Throughout the proof, we assume $d$ to be large enough such that
all quantities are well-defined and all necessary lemmas apply.
In particular, we assume
the conditions of \cref{lem:K1K2} to be satisfied,
and that $\c < \floor{\filsize / 2} < \filsize < d/2$
for $\c$ in \cref{ass:genericity}.
Hence, $\L + 2 \leq \bL + 2 < \T < \floor{\filsize / 2}$,
and we can apply
\crefrange{lem:n_sets}{lem:Q11},
the setting of \cref{ssec:12decomposition},
as well as \cref{ass:genericity} throughout the proof.
We will mention additional implicit lower bounds on $d$ as they arise.

The proof proceeds in three steps:
we first bound $\rmin$ and $\rmax$,
then bound $\Tr(\S_{>\m})$,
and finally $\sum_{\iii = 1+\m}^{n}{\eigval_\iii{(\SM)}}$.
We do not establish the required matrix concentration results directly,
but apply various auxiliary lemmas.
All corresponding statements hold with either
probability at least $1-\tilde{\c} \filsize^{-\bd}$
or at least $1-\tilde{\c} \filsize^{-(1-\bd)}$
for context-dependent constants $\tilde{\c}$.
We hence implicitly choose a $\c > 0$ such that
collecting all error probabilities yields
the statement of \cref{lem:combined}
with probability at least $1-\c d^{-\beta\min\{\bd,1-\bd\}}$.

To start,
let $\m \in \N$ as in the statement of \cref{lem:combined},
and instantiate \cref{ssec:12decomposition} with that $\m$.
In particular,
\cref{lem:K1K2} yields the $1$-$2$ decomposition $\KM = \K_1 + \K_2$ and $\SM = \S_1+\S_2$,
which we will henceforth use.
Finally, we define
\begin{equation}\label{eq:Qdef}
    \Qf^{(d,\filsize)}_\ii(x,x')
    \coloneqq \sum_{\substack{\diam(S)\leq \filsize \\ \abs{S} =\ii}}{
        \frac{\filsize + 1 - \diam(S)}{d\B(\ii,\filsize)} \Yf_{S}(x)\Yf_{S}(x')
    }
\end{equation}
with the corresponding kernel matrix $\Q^{(d,\filsize)}_\ii \in \R^{n \times n}$.

\paragraph*{Bound on $\rmin$ and $\rmax$}
Remember the definition of $\rmin$ and $\rmax$:
\begin{equation*}
    \rmin = \frac{\mineig{\KM}+\reg}{\max\{\reg,1\}},\qquad
    \rmax = \frac{\norm{\KM}+\reg}{\max\{\reg,1\}}.
\end{equation*}
To bound those quantities, we have to bound $\mineig{\KM}$ and $\norm{\KM}$.
For the upper bound on $\norm{\KM}$,
we use the triangle inequality on
$\norm{\KM} = \norm{\K_1 + \K_2}$,
and then bound $\norm{\K_1}$ and $\norm{\K_2}$ individually.

Note that we can write $\K_1 = \P_1\D_1\P_1^\t$ by definition.
Hence,
\begin{equation*}
    \norm{\K_1}
    = \norm{\P_1\D_1\P_1^\t}
    = n\norm{\D_1}\norm*{\frac{\P_1^\t\P_1}{n}}
    \overset{(i)}{\leq} 1.5 n \norm{\D_1}
    \overset{(ii)}{\leq} 1.5 n \lambda_{\m}
    \overset{(iii)}{\in} \bigO\left(\max\{\reg,1\}\right),
\end{equation*}
where $(i)$ follows from \cref{lem:K1}
with probability at least $1-\tilde{\c}_1 \filsize^{-\bd}$,
$(ii)$ uses that all eigenvalues of $\Kf_1$ are at most $\lambda_\m$ by definition,
and $(iii)$ follows from $n\lambda_{\m} \in \Theta(\max\{\reg,1\})$.

Next,
\cref{lem:K2}
directly yields with probability at least $1-\tilde{\c}_2 \filsize^{-(1-\bd)}$
that $\norm{\K_2} \in \Theta(1)$ and $\mineig{\K_2} \in \Theta(1)$.
Hence,
with probability at least
$1 - (\tilde{\c}_1 \filsize^{-\bd} + \tilde{\c}_2 \filsize^{-(1-\bd)}) \geq 1-\c \filsize^{-\min\{\bd,1-\bd\}}$,
we have
\begin{align*}
    \norm{\K_1},\norm{\K_2} &\in \bigO\left(\max\{\reg,1\}\right),\\
    \mineig{\K_2} &\in \Omega(1).
\end{align*}
This implies
\begin{align*}
    \norm{\KM} + \reg \leq \norm{\K_1}+\norm{\K_2} +\reg &\in \bigO(\max\{\reg,1\}),\\
    \mineig{\KM}+\reg \geq \mineig{\K_2} + \reg &\in \Omega(\max\{\reg,1\}),
\end{align*}
and subsequently
\begin{align*}
    \rmax &= \frac{\norm{\KM} + \reg}{\max\{\reg,1\}} \in \bigO(1),\\
    \rmin &= \frac{\mineig{\KM} + \reg}{\max\{\reg,1\}} \in \Omega(1).
\end{align*}
Finally, since $\rmin \leq \rmax$, this yields $\rmin, \rmax \in \Theta(1)$.

\paragraph*{Bound on $\Tr(\S_{>\m})$}
We need to show that
\begin{equation}\label{eq:to_prove_S}
        \Tr(\SM)
        \lesssim d^{\Lr}\filsize^{-\delta} +d^{\Lr}\filsize^{\delta-1}
\end{equation}
with high probability,
where the two terms correspond to the $1$-$2$ decomposition
$\Tr{(\SM)} = \Tr{(\S_1)}+\Tr{(\S_2)}$.
We differentiate between $\bL = \L$ and $\bL > \L$.
Intuitively, the case $\bL = \L$ corresponds to interpolation or weak regularization,
because the maximum degree of learnable polynomials
with regularization equals the one without regularization.
Conversely, $\bL > \L$ corresponds to strong regularization.

\pseudoparagraph{Case $\bL = \L$ (interpolation or weak regularization)}
In this setting,
\begin{equation}
    \label{eq:dbd}
    \delta = \frac{\Ln-\Lr-1}{\beta} - \L
    = \frac{\Ln-1}{\beta} - \bL - \frac{\Lr}{\beta}
    = \bd - \frac{\Lr}{\beta}.
\end{equation}
First, \cref{lem:S2}
yields $\Tr(\S_{2}) \in \Theta(\filsize^{-(1-\bd)})$
with probability at least $1-\tilde{\c}_3 {\filsize^{-(1-\bd)}}$.
Therefore,
\begin{equation*}
    \Tr(\S_2)
    \in \Theta(\filsize^{-(1-\bd)})
    \overset{(i)}{=} \Theta(\filsize^{-(1-\d) + \frac{\Lr}{\beta}})
    = \Theta\left(d^{\Lr} \filsize^{-(1-\d)}\right),
\end{equation*}
where $(i)$ follows from \cref{eq:dbd}.

We now consider $\Tr(\S_1)$:
\begin{equation*}
    \Tr(\S_1) = n\Tr\left(\frac{\P_{1}^\t\P_{1}}n\D_1^2\right)
    \leq n \norm*{\frac{\P_{1}^\t\P_{1}}n} \Tr(\D_1^2)
    \overset{(i)}{\leq} 1.5 n\sum_{\i \in \IndS_1} \lambda_{S_\i}^2,
\end{equation*}
where $(i)$ follows from \cref{lem:K1}
with probability at least $1-\tilde{\c}_1 \filsize^{-\bd}$.
The bound continues as
\begin{align*}
    \Tr(\S_1) &\lesssim n \sum_{\i \in \IndS_1} \lambda_{S_\i}^2 \overset{(i)}{\leq} n \lambda_{\m}^2 \abs{\IndS_1} \\ %
    &\overset{(ii)}{\in} \bigO\left(\max\{\reg,1\}^2\frac{\abs{\IndS_1}}{n}\right) \\
    &\overset{(iii)}{\subseteq} \bigO\left(d^{2\Lr}\frac{d\filsize^{\bL}}{d\filsize^{\bL+\bd}}\right)
    = \bigO\left(d^{2\Lr}\filsize^{-\bd}\right)
    \overset{(iv)}=\bigO\left(d^{\Lr}\filsize^{-\delta}\right),
\end{align*}
where $(i)$ uses $\lambda_{\i}\leq \lambda_{\m}$ for all $\i \in \IndS_1$ by definition,
$(ii)$ uses $n\lambda_{\m}\in \Theta(\max\{\reg,1\})$,
and $(iv)$ follows from \cref{eq:dbd}.
Furthermore, $(iii)$ uses the following bound of $\abs{\IndS_1}$:
\begin{equation*}
    \abs*{\IndS_1}
    \leq \sum_{\ii = 0}^{\bL + 1}\C(\ii,\filsize, d)
    \overset{(i)}{=} 1 + d \sum_{\ii = 1}^{\bL + 1}{\binom{\filsize - 1}{\ii - 1}}
    \overset{(ii)}{\leq} 1 + d \sum_{\ii = 1}^{\bL + 1}{\left(e \frac{\filsize - 1}{\ii - 1}\right)^{\ii - 1}}
    \overset{(iii)}{\in} \bigO(d\cdot \filsize^{\bL}),
\end{equation*}
where $\C(\ii,\filsize, d)$ is defined in \cref{eq:C_def},
$(i)$ follows from \cref{lem:n_sets},
$(ii)$ is a classical bound on the binomial coefficient,
and $(iii)$ follows from the fact that the term corresponding to $\bL + 1$
dominates the polynomial.

Finally, collecting the upper bounds on $\Tr(\S_1)$ and $\Tr(\S_2)$ yields
\begin{equation*}
    \Tr(\S_{\geq \m}) = \Tr(\S_1) + \Tr(\S_2)
    \in \bigO\left(d^{\Lr}\filsize^{-\delta} + d^{\Lr} \filsize^{-(1-\d)}\right)
\end{equation*}
with probability at least
$1- (\tilde{\c}_3 {\filsize^{-(1-\bd)}} + \tilde{\c}_1 \filsize^{-\bd})
\geq 1-\c \filsize^{-\min\{\bd,1-\bd\}}$.

\pseudoparagraph{Case $\bL > \L$ (strong regularization)}
In this setting, the dominating rate will arise from $\Tr(\S_1)$.
We start by linking $\bd$ and $\delta$ in analogy to \cref{eq:dbd}:
\begin{equation}\label{eq:dbd2}
    \delta = \frac{\Ln-\Lr-1}{\beta}-\L
    = -\frac{\Lr}{\beta}+ \frac{\Ln-1}{\beta}-\bL + \bL - \L
    = \bd -\frac{\Lr}{\beta}+ \bL-\L .
\end{equation}
Next, as in the previous case,
\cref{lem:S2} yields $\Tr(\S_{2}) \in \Theta(\filsize^{-(1-\bd)})$
with probability at least $1-\tilde{\c}_3 {\filsize^{-(1-\bd)}}$,
and therefore
\begin{equation*}
    \Tr(\S_2)
    \in \Theta(\filsize^{-(1-\bd)})
    = \Theta(\filsize^{\bd - 1})
    \overset{(i)}{=} \Theta\left(d^\Lr \filsize^{\delta-(\bL-\L) - 1}\right)
    \overset{(ii)}{\subseteq }o\left(d^\Lr \filsize^{-(1-\delta)} \right),
\end{equation*}
where $(i)$ follows from \cref{eq:dbd2},
and $(ii)$ from $\bL > \L$.

To bound $\Tr(\S_1)$, we start as in the previous case:
\begin{equation*}
    \Tr(\S_1) = n\Tr\left(\frac{\P_{1}^\t\P_{1}}n\D_1^2\right)
    \leq n \norm*{\frac{\P_{1}^\t\P_{1}}n} \Tr(\D_1^2)
    \overset{(i)}{\leq} 1.5 n\sum_{\i \in \IndS_1} \lambda_{S_\i}^2,
\end{equation*}
where $(i)$ follows from \cref{lem:K1}
with probability at least $1-\tilde{\c}_1 \filsize^{-\bd}$.
We then decompose the sum over all squared eigenvalues with index in $\IndS_1$ as
\begin{align*}
    n\sum_{\i \in \IndS_1} \lambda_{S_\i}^2 &
    = \underbrace{
        n\sum_{\substack{\i \in \IndS_1 \\ \abs{S_\i} \leq \L+1}} \lambda_{S_\i}^2
    }_{\revdefeq \tterm_1} + \underbrace{
        n\sum_{\substack{\i \in \IndS_1 \\ \abs{S_\i} = \L+2}} \lambda_{S_\i}^2
    }_{\revdefeq \tterm_2} + \underbrace{
        n  \sum_{\substack{\i \in \IndS_1 \\ \abs{S_\i} \geq \L + 3}} \lambda_{S_\i}^2
    }_{\revdefeq \tterm_3},
\end{align*}
and bound the three terms individually.

First, we upper-bound $\tterm_1$ as follows:
\begin{align*}
    \tterm_1 &= n\sum_{\substack{\i \in \IndS_1 \\ \abs{S_\i} \leq \L+1}} \lambda_{S_\i}^2 \overset{(i)}\leq \frac{n^2\lambda_{\m}^2}{n}\sum_{\substack{\i \in \IndS_1 \\ \abs{S_\i} \leq \L+1}}1 \\
    &\leq \frac{n^2\lambda_{\m}^2}{n} \sum_{\ii = 0}^{\L+1}\C(\ii,\filsize,d) \overset{(ii)} \in \bigO\left(\frac{\max\{\reg,1\}^2}{n}d\filsize^{\L}\right)\\
    &= \bigO\left(\frac{d^{2\Lr}}{d \cdot d^{\Lr}\cdot \filsize^{\L+\delta}}d\filsize^{\L}\right)
    = \bigO\left(d^{\Lr}\filsize^{-\delta}\right),
\end{align*}
where $(i)$ follows from $\lambda_\i \leq \lambda_\m$ for all $\i > \m$
due to the decreasing order of eigenvalues.
Step $(ii)$ applies $n\lambda_{\m} \in \Theta(\max\{\reg,1\})$,
as well as $\sum_{\ii = 0}^{\L + 1}\C(\ii,\filsize, d) \in \bigO(d\cdot \filsize^{\L})$,
which follows as in the other case from \cref{lem:n_sets}
and the classical bound on the binomial coefficient.

Second, the upper bound of $\tterm_2$ arises as follows:
\begin{align*}
    \tterm_2 = n\sum_{\substack{\i \in \IndS_1 \\ \abs{S_\i} = \L+2}} \lambda_{S_\i}^2
    &\overset{(i)}{=} n (\xi_{\L+2}^{(\filsize)} )^2
    \sum_{\substack{\i \in \IndS_1 \\ \abs{S_\i} = \L+2}}{\left(
        \frac{\filsize + 1-\diam(S_\i)}{d\B(\L+2,\filsize)}
    \right)^2} \\
    &\overset{(ii)}{\leq }(\xi_{\L+2}^{(\filsize)} )^2 \frac{n\cdot \filsize}{d\B(\L+2,\filsize)}
    \sum_{\substack{\diam(S) \leq \filsize \\ \abs{S} = \L+2}}{
        \frac{\filsize + 1-\diam(S)}{d\B(\L+2,\filsize)}
    } \\
    &\overset{(iii)}{\lesssim }
    \frac{n\cdot \filsize}{d \cdot \filsize^{\L+2}}
    \sum_{\substack{\diam(S) \leq \filsize \\ \abs{S} = \L+2}}{
        \frac{\filsize + 1-\diam(S)}{d\B(\L+2,\filsize)}
    } \\
    &\overset{(iv)}{=} \frac{n \cdot \filsize}{d\cdot \filsize^{\L+2}}
    \Qf_{\L + 2}^{(d,\filsize)}(x,x) \\
    &\overset{(v)}{\in}
    \bigO\left( \frac{d\cdot d^{\Lr} \filsize^{\L+\delta} \cdot \filsize}{d\cdot \filsize^{\L+2}} \right)
    =
    \bigO\left(d^{\Lr} \filsize^{-(1-\delta)}\right),
\end{align*}
where $(i)$ follows from \cref{prop:kern_eig},
and $(ii)$ uses $\filsize + 1-\diam(S_\i)\leq \filsize$.
Next, $(iii)$ uses that \cref{eq:ass_lb,eq:ass_ub_2} in \cref{ass:genericity}
imply $\xi_{\L + 2}^{(\filsize)} \in \Theta(1)$,
and applies the bound
$\B(\L + 2,\filsize)=\binom{\filsize}{\L+2}\leq(e\filsize / (\L+2))^{\L+2}$.
Step $(iv)$ uses $\Yf_{S}(x)\Yf_{S}(x) = 1$ for all $S$ and $x \in \2^d$,
together with the definition of $\Qf_{\L + 2}^{(d,\filsize)}$.
Lastly, $(v)$ applies \cref{lem:Q11}
and $n \in \Theta(d^\Ln) = \Theta(d\cdot d^{\Lr} \filsize^{\L+\delta})$.

Third, we upper-bound $\tterm_3$:
\begin{equation*}
    \tterm_3= n \sum_{\substack{\i \in \IndS_1 \\ \abs{S_\i} \geq \L + 3}}
        \lambda_{S_\i}^2
    \leq n\left(\max_{\i \in \IndS_1, \abs{S_\i} \geq \L+3}\lambda_{S}\right)
        \sum_{\substack{\i \in \IndS_1 \\ \abs{S_\i} \geq \L+3}}{\lambda_{S_{\i}}}
    \lesssim n\max_{\i \in \IndS_1, \abs{S_\i} \geq \L+3}(\lambda_{S_\i}).
\end{equation*}
The last step follows from
\begin{align*}
    \sum_{\substack{\i \in \IndS_1 \\ \abs{S_\i} \geq \L+3}}{\lambda_{S_{\i}}}
    &\leq \sum_{\ii = \L+3}^{\bL + 1}{
        \sum_{\substack{\diam(S) \leq \filsize \\ \abs{S} = \ii}}{\lambda_S}
    } \\
    &\overset{(i)}{=} \sum_{\ii = \L+3}^{\bL + 1}{
        \sum_{\substack{\diam(S) \leq \filsize \\ \abs{S} = \ii}}{
            \xi^{(\filsize)}_{\ii}
            \frac{\filsize + 1 - \diam(S)}{d\B(\ii,\filsize)}
        }
    } \\
    &\overset{(ii)}{=} \sum_{\ii = \L+3}^{\bL + 1}{
        \xi^{(\filsize)}_{\ii}
        \sum_{\substack{\diam(S) \leq \filsize \\ \abs{S} = \ii}}{
            \frac{\filsize + 1 - \diam(S)}{d\B(\ii,\filsize)}
            \Yf_{S}(x)\Yf_{S}(x)
        }
    } \\
    &\overset{(iii)}{=} \sum_{\ii = \L+3}^{\bL + 1}{
        \xi^{(\filsize)}_{\ii}
    }
    \overset{(iv)}{\lesssim} 1,
\end{align*}
where $(i)$ follows from \cref{prop:kern_eig},
$(ii)$ uses $\Yf_{S}(x)\Yf_{S}(x) = 1$ for all $S$ and $x \in \2^d$,
$(iii)$ applies the definition of $\Qf_{\ii}^{(d,\filsize)}$ and \cref{lem:Q11},
and $(iv)$ follows from \cref{eq:ass_lb,eq:ass_ub_2} in \cref{ass:genericity}
since $\bL + 1 \leq \T$.

For $\max_{\i \in \IndS_1, \abs{S_\i} \geq \L+3}(\lambda_{S_\i})$,
we bound each element individually:
\begin{align*}
    \lambda_{S_\i}
    \overset{(i)}{\in} \bigO{\left( \frac{1}{d\filsize^{\abs{S_\i}-1}} \right)}
    \subseteq \bigO{\left( \frac{1}{d\filsize^{(\L + 3)-1}} \right)},
\end{align*}
where $(i)$ uses \cref{eq:beforeT} in \cref{lem:eigendecay}
since $\i \leq \bL+1<\T$ by definition of $\IndS_1$.
Hence, we obtain the following bound on $\tterm_3$:
\begin{equation*}
    \tterm_3
    \lesssim n\max_{\i \in \IndS_1, \abs{S_\i} \geq \L+3}(\lambda_{S_\i}) \in \bigO\left(\frac{n}{d\filsize^{\L+2}}\right) =\bigO\left(
        \frac{d\cdot d^{\Lr}\filsize^{\L+\delta}}{d\filsize^{\L+2}}
    \right)
    \subseteq \bigO\left(d^{\Lr}\filsize^{-(1-\delta)}\right).
\end{equation*}

Finally, we can bound $\Tr(\S_1)$ as
\begin{equation*}
    \Tr(\S_1) \leq 1.5 n\sum_{\i \in \IndS_1} \lambda_{S_\i}^2
    = \term_1 + \term_2 + \term_3
    \in \bigO{\left(
        d^{\Lr}\filsize^{-\delta}
        + d^{\Lr} \filsize^{-(1-\delta)}
    \right)},
\end{equation*}
which yields
\begin{equation*}
    \Tr(\S_{\geq \m}) = \Tr(\S_1) + \Tr(\S_2)
    \in \bigO\left(d^{\Lr}\filsize^{-\delta} + d^{\Lr} \filsize^{-(1-\d)}\right)
\end{equation*}
as desired with probability at least
$1- (\tilde{\c}_3 {\filsize^{-(1-\bd)}} + \tilde{\c}_1 \filsize^{-\bd})
\geq 1-\c \filsize^{-\min\{\bd,1-\bd\}}$.

\paragraph*{Bound on $\sum_{\iii = 1+\m}^{n}{\eigval_\iii{(\SM)}}$}
As before, we differentiate between no/weak and strong regularization,
that is, between $\bL = \L$ and $\bL > \L$:

\pseudoparagraph{Case $\bL = \L$ (interpolation or weak regularization)}
In this case, we start by directly bounding
\begin{align}
    \sum_{\iii = 1+\m}^{n}{\eigval_\iii{(\SM)}}
    &= \sum_{\iii = 1+\m}^{n}{\eigval_\iii{(\S_1+\S_2)}}
    \geq \sum_{\iii = 1+\m}^{n}{\eigval_\iii{(\S_2)}} \nonumber \\
    &= \sum_{\iii = 1}^{n}{\eigval_\iii{(\S_2)}}
    - \sum_{\iii = 1}^{\m}{\eigval_\iii{(\S_2)}}
    \overset{(i)}\geq \Tr(\S_{2}) - \m \norm{\S_{2}} \label{eq:trick}\\
    &\overset{(ii)}{\in} \Omega{\left(
        \filsize^{-(1-\bd)}-\frac{\m}{d\filsize^{\bL+1}}
    \right)} \nonumber \\
    &\overset{(iii)}{=} \Omega{\left(
        d^{\Lr}\filsize^{-(1-\delta)}-\frac{\m}{d\filsize^{\bL+1}}
    \right)}, \nonumber
\end{align}
where $(i)$ bounds each of the first $\m$ eigenvalues of $S_2$ with the largest one,
$(ii)$ follows from \cref{lem:S2} with probability at least $1-\tilde{\c}_3\filsize^{-(1-\bd)}$,
and $(iii)$ from \cref{eq:dbd} since $\bL = \L$.

To conclude the lower bound, it suffices to show
that $\frac{\m}{d\filsize^{\bL+1}} \in o(d^{\Lr}\filsize^{-(1-\delta)})$:
\begin{equation*}
    \frac{\m}{d\filsize^{\bL+1}}
    \overset{(i)}{\in} \bigO\left(\frac{n\filsize^{-\delta}}{\max\{\reg,1\}d\filsize^{\bL+1}}\right)
    = \bigO\left(\frac{ \filsize^{-\delta} d \filsize^{\bd+ \bL}}{d^{\Lr}d\filsize^{\bL+1}}\right)
    \overset{(ii)}{=} \bigO\left(\filsize^{-\bd}d^{\Lr}\filsize^{-(1-\delta)}\right)
    \overset{(iii)}{\subseteq} o(d^{\Lr}\filsize^{-(1-\delta)}),
\end{equation*}
where $(i)$ follows from $m \in \bigO\left(\frac{n\filsize^{-\delta}}{\max\{\reg,1\}}\right)$,
and $(ii)$ from \cref{eq:dbd}.
For $(iii)$, note that $\bd = 0$ for a sufficiently large $\c$
yields a vacuous result.
We hence assume without loss of generality that $\bd > 0$,
which justifies the step.
This concludes the proof for the current case
with probability at least
$1-\tilde{\c}_3\filsize^{-(1-\bd)} \geq 1-\c \filsize^{-\min\{\bd,1-\bd\}}$.

\pseudoparagraph{Case $\bL > \L$ (strong regularization)}
In this case, we define the additional index set
\begin{equation*}
    \IndS_3
    \defeq \{\i \in \IndS_1 \mid \abs{S_\i} = \L+2\}
\end{equation*}
with $\S_3,\P_3,\D_3$ analogously to $\S_1,\P_1,\D_1$ in \cref{ssec:12decomposition},
but using $\IndS_3$ instead of $\IndS_1$.
Since $\IndS_3\subseteq\IndS_1$,
it follows that
$\P_3^\t\P_3$ is a submatrix of $\P_1^\t\P_1$, and thus
\begin{equation*}
    \frac{\P_3^\t\P_3}n - \I_{\abs{\IndS_3}}
    \quad\text{is a submatrix of}\quad
    \frac{\P_1^\t\P_1}n - \I_{\abs{\IndS_1}}.
\end{equation*}
This particularly implies
\begin{equation}
    \label{eq:K3}
    \norm*{\frac{\P_3^\t\P_3}n - \I_{\abs{\IndS_3}}}
    \leq \norm*{\frac{\P_1^\t\P_1}n - \I_{\abs{\IndS_1}}}
    \overset{(i)}{\leq} 1/2,
\end{equation}
where $(i)$ follows from \cref{lem:K1}
with probability at least $1-\tilde{\c}_1 \filsize^{-\bd}$.

We now move our focus back to the lower bound of $\sum_{\iii = 1+\m}^{n}{\eigval_\iii{(\SM)}}$:
\begin{equation}\label{eq:tr_dec}
    \sum_{\iii = 1+\m}^{n}{\eigval_\iii{(\SM)}}
    \overset{(i)}\geq \sum_{\iii = 1+\m}^{n}{\eigval_\iii{(\S_1)}}
    \overset{(ii)}\geq \sum_{\iii = 1+\m}^{n}{\eigval_\iii{(\S_3)}}
    \overset{(iii)}\geq \Tr(\S_{3}) - \m \norm{\S_3},
\end{equation}
where $(i)$ follows from the $1$-$2$ decomposition $\SM = \S_{1}+\S_2$,
$(ii)$ from the fact that $\IndS_3 \subseteq \IndS_1$,
and $(iii)$ analogously to \cref{eq:trick}.
Similar to the previous case,
we conclude the proof by first showing that
$\Tr{(\S_3)} \in \Omega\left(d^{\Lr} \filsize^{-(1-\delta)}\right)$,
and then $\m \norm{\S_3} \in o(\Tr(\S_3))$.

For the lower bound of $\Tr(\S_{3})$, we start with
\begin{align*}
    \Tr(\S_{3})
    &= n\Tr\left(\frac{1}n\P_3\D_3^2\P_3^\t\right)\\
    &\geq n\Tr\left(\D_3^2\right) \mineig{\frac{\P_{3}^\t\P_3}n}\\
    &\overset{(i)}\geq 0.5 n\sum_{\i \in \IndS_3} \lambda_{\i}^2\\
    &\overset{(ii)}{=} n (\xi_{\L+2}^{(\filsize)} )^2
    \sum_{\substack{\abs{S} = \L+2 \\ \diam(S)\leq \filsize}}{
        \left(\frac{\filsize + 1-\diam(S)}{d\B(\L+2,\filsize)}\right)^2
    }
\end{align*}
where $(i)$ follows with high probability from \cref{eq:K3}.
Step $(ii)$ applies \cref{prop:kern_eig},
and the fact that
$\IndS_3 = \{\i \in \N \mid \abs{S_\i} = \L + 2 \text{ and } \diam{(S)} \leq \filsize\}$
for $d$ sufficiently large.
To show this, we use
\begin{equation*}
    \lambda_\m \in \Theta\left(\frac{d^\Lr}{n}\right) = \Theta\left(\frac{d^\Lr}{dd^{\Lr}\filsize^{\L+\delta}}\right)= \Theta\left(\frac{1}{d\filsize^{\L+\delta}}\right).
\end{equation*}
Since the eigenvalues are in decreasing order
and $\L + 2 \leq \bL + 1$ in the current case,
we only need to show that $\lambda_{S} < \lambda_\m$
for all $S \subseteq \{1, \dotsc, d\}$ with $\abs{S} = \L+2$ and $\diam(S) \leq \filsize$:
\begin{equation*}
    \lambda_{S}
    \overset{(i)}{\in} \bigO\left(\frac{1}{d\filsize^{\L+1}}\right)
    = o\left(\frac{1}{d\filsize^{\L+\delta}}\right) = o(\lambda_\m),
\end{equation*}
where $(i)$ applies \cref{eq:beforeT} in \cref{lem:eigendecay}
since $\L+2 < \T$.
Thus,
$\lambda_{S} < \lambda_\m$ for all $S \subseteq \{1, \dotsc, d\}$
with $\abs{S} = \L+2$ and $\diam(S) \leq \filsize$
if $d$ is sufficiently large,
which we additionally assume from now on.

The lower bound of $\Tr(\S_{3})$ continues as follows:
\begin{align*}
    \Tr(\S_{3})
    &\geq n (\xi_{\L+2}^{(\filsize)} )^2
    \sum_{\substack{\abs{S} = \L+2 \\ \diam(S)\leq \filsize}}{
        \left(\frac{\filsize + 1-\diam(S)}{d\B(\L+2,\filsize)}\right)^2
    } \\
    &\geq n (\xi_{\L+2}^{(\filsize)} )^2
    \sum_{\substack{\abs{S} = \L+2\\\diam(S)\leq \floor{\filsize/2}}}{
        \left(\frac{\filsize + 1-\diam(S)}{d\B(\L+2,\filsize)}\right)^2
    } \\
    &\overset{(iii)}{\geq} (\xi_{\L+2}^{(\filsize)})^2
    \frac{n\cdot \filsize/2}{d\B(\L+2,\filsize)}
    \sum_{\substack{\abs{S} = \L+2 \\ \diam(S)\leq \floor{\filsize/2}}}{
        \frac{\floor{\filsize/2} + 1-\diam(S)}{d\B(\L+2,\filsize)}
    } \\
    &\overset{(iv)}{\gtrsim} (\xi_{\L+2}^{(\filsize)})^2
    \frac{n\cdot \filsize/2}{d\B(\L+2,\filsize)}
    \sum_{\substack{\abs{S} = \L+2\\\diam(S)\leq \floor{\filsize/2}}}{
        \frac{\floor{\filsize/2} + 1-\diam(S)}{d\B(\L+2,\floor{\filsize/2})}
    },
\end{align*}
where $(iii)$ follows from $\filsize \geq \floor{\filsize/2}$
and $\filsize+1-\diam(S)\geq \filsize/2$.
Step $(iv)$ follows from the fact that
$\B(\L+2,\filsize)$ and $\B(\L+2,\floor{\filsize/2})$
are of the same order;
this follows from a classical bound on the binomial coefficient:
\begin{equation*}
    \B(\L+2,\filsize)
    \leq \left(\frac{e \filsize}{\L+2}\right)^{\L+2}
    \lesssim \left(\frac{\floor{\filsize/2}}{\L+2}\right)^{\L+2}
    \leq \B(\L+2,\floor{\filsize/2}).
\end{equation*}

We conclude the lower bound on $\Tr(\S_{3})$ as follows:
\begin{align}
    \Tr(\S_{3})
    &\gtrsim (\xi_{\L+2}^{(\filsize)})^2
    \frac{n\cdot \filsize/2}{d\B(\L+2,\filsize)}
    \sum_{\substack{\abs{S} = \L+2\\\diam(S)\leq \floor{\filsize/2}}}{
        \frac{\floor{\filsize/2} + 1-\diam(S)}{d\B(\L+2,\floor{\filsize/2})}
    } \nonumber \\
    &\overset{(v)}{=} (\xi_{\L+2}^{(\filsize)})^2
    \frac{n\cdot \filsize/2}{d\B(\L+2,\filsize)}
    \Qf_{\L+2}^{(d,\floor{\filsize/2})}(x,x) \nonumber \\
    &\overset{(vi)}{=} (\xi_{\L+2}^{(\filsize)})^2
    \frac{n\cdot \filsize/2}{d\B(\L+2,\filsize)} \nonumber \\
    &\overset{(vii)}{\in} \Omega\left(
        \frac{n\cdot \filsize}{d\B(\L+2,\filsize)}
    \right) \nonumber \\
    &\overset{(viii)}{=} \Omega\left(
        \frac{d\cdot d^{\Lr} \filsize^{\L+\delta}\cdot \filsize}{d\cdot \filsize^{\L+2}}
    \right)
    = \Omega\left(d^{\Lr} \filsize^{-(1-\delta)}\right),\label{eq:trs3_bound}
\end{align}
where $(v)$ uses $\Yf_{S}(x)\Yf_{S}(x) = 1$ for all $S$ and $x \in \2^d$
with the definition of $\Qf_{\L+2}^{(d,\floor{\filsize/2})}$ in \cref{eq:Qdef},
$(vi)$ applies \cref{lem:Q11} with $\floor{q/2}$ as filter size,
$(vii)$ follows from \cref{eq:ass_lb} in \cref{ass:genericity},
and $(viii)$ uses the classical bound on the binomial coefficient.

Finally, for the upper bound on $\m\norm{\S_3}$, we have
\begin{align}
    \m\norm{\S_3}
    &= \m \norm{\P_3\D_3\P_3^\t}
    \leq \m n \norm{\D_3} \norm*{\frac{\P_3^\t\P_3}{n}}
    \overset{(i)}{\leq} 1.5 \m n \norm{\D_3}
    = \m n \max_{\i \in \IndS_3}{\lambda_\i^2} \nonumber \\
    &\overset{(ii)}{=} \m n  \max_{\i \in \IndS_3}{\left(
        \xi_{\L+2}^{(\filsize)}\frac{\filsize + 1-\diam(S_\i)}{d\B(\L+2,\filsize)}
    \right)^2}
    \leq \frac{\m}{n} (\xi_{\L+2}^{(\filsize)})^2\left(
        \frac{n \cdot \filsize}{d\B(\L+2,\filsize)}
    \right)^2 \nonumber\\
    &\overset{(iii)}\in \bigO\left(
        \frac{\m}n \left(\frac{d d^{\Lr} \filsize^{\L+\delta} \cdot \filsize}{d\filsize^{\L+2}}
    \right)^2\right)
    \overset{(iv)}{\subseteq} \bigO\left(
        \filsize^{-\delta}d^{-\Lr}\left(d^{\Lr}\filsize^{-(1-\delta)}\right)^2
    \right) \nonumber \\
    &= \bigO\left(
        \filsize^{-\delta}\filsize^{-(1-\delta)}\left(d^{\Lr}\filsize^{-(1-\delta)}\right)
    \right)
    \overset{(v)}{\subseteq} o\left(d^{\Lr}\filsize^{-(1-\delta)}\right)
    \overset{(vi)}{\subseteq} o(\Tr(\S_3)),\notag
\end{align}
where $(i)$ follows with high probability from \cref{eq:K3},
and $(ii)$ from \cref{prop:kern_eig}.
Step $(iii)$ uses that \cref{eq:ass_lb,eq:ass_ub_2} in \cref{ass:genericity}
yield $\xi_{\L+2}^{(\filsize)} \in \Theta(1)$,
and $\B(\L+2,\filsize) = \binom{\filsize}{\L+2} \in \bigO(\filsize^{\L+2})$.
Furthermore,
$(iv)$ follows from $\m \in \bigO\left(\frac{n\filsize^{-\delta}}{\max\{\reg,1\}}\right)$,
$(v)$ from $\filsize^{-\delta}\filsize^{-(1-\delta)} \in o(1)$,
and $(vi)$ from the lower bound on $\Tr(\S_3)$ in \cref{eq:trs3_bound}.

Finally, combining this result with \cref{eq:tr_dec,eq:trs3_bound}, we have
\begin{equation*}
    \sum_{\iii = 1+\m}^{n}{\eigval_\iii{(\SM)}}
    \geq \Tr(\S_3) - \m \norm{\S_3}
    \in \Omega(\Tr(\S_3)) \subseteq \Omega\left(d^{\Lr}\filsize^{-(1-\delta)}\right)
\end{equation*}
with probability at least
$1-\tilde{\c}_1 \filsize^{-\bd} \geq 1-\c \filsize^{-\min\{\bd,1-\bd\}}$.

\subsection{Technical lemmas}\label{ssec:technical_lemmas}
We use the following technical lemmas in the proof of \cref{lem:combined}.
All results assume the setting of \cref{ssec:12decomposition},
particularly a kernel as in \cref{th:rates}
that satisfies \cref{ass:genericity}.

\begin{lemma}[Bound on $\K_1$]
    \label{lem:K1}
    In the setting of \cref{ssec:12decomposition},
    for $d/2 > \filsize \geq \bL + 1$, we have
    \begin{equation*}
        \norm*{\frac{\P_1^\t\P_1}n-\I_{\abs{\IndS_1}}} \leq 1/2
    \end{equation*}
    with probability at least $1-\c \filsize^{-\bd}$
    uniformly over all choices of $\m$ and $\reg$.
\end{lemma}
\begin{proof}
    The proof follows a very similar argument to \cref{lem:as_1_verification}
    with minor modifications.
    First, define
    \begin{equation*}
        \IndSlarge_1 \defeq \{\i \in \N \mid \abs{S_\i} \leq \bL + 1\}.
    \end{equation*}
    Note that $\IndSlarge_1$ does not depend on $\m$ or $\reg$,
    and $\IndS_1 \subseteq  \IndSlarge_1$ for any $\m$.
    Furthermore,
    \begin{equation*}
        \abs{\IndSlarge_1}
        = \sum_{\ii = 0}^{\bL + 1}\C(\ii,\filsize, d)
        \overset{(i)}{=} 1 + \sum_{\ii = 1}^{\bL + 1} d\binom{\filsize-1}{\ii-1}
        \overset{(ii)}{\leq} 1 + d\sum_{\ii = 1}^{\bL + 1}{
            \left(\frac{e (\filsize - 1)}{\ii-1}\right)^{\ii-1}
        }
        \overset{(iii)}{\in} \bigO(d\cdot \filsize^{\bL}),
    \end{equation*}
    where $(i)$ follows from \cref{lem:n_sets},
    $(ii)$ is a classical bound on the binomial coefficient,
    and $(iii)$ follows from the fact that the largest degree monomial
    dominates the others.

    Next, we define $\hat{\P}_1 \in \R^{n\times \abs{\IndSlarge_1}}$
    with $[\hat{\P}_1]_{i,j} = \Yf_{S_{\i_j}}(x_i)$
    where $\i_j$ is the $j$-th element in $\IndSlarge_1$.
    Using the same arguments as in the proof of \cref{lem:as_1_verification},
    it follows that the rows of $\hat{\P}_1$ are independent,
    that their norm is bounded by $\abs{\IndSlarge_1}$,
    and that they have an expected outer product equal to $\I_{\abs{\IndSlarge_1}}$.

    Hence, as in the proof of \cref{lem:as_1_verification},
    we can apply Theorem~5.44 from \cite{Vershynin2012}.
    Choosing $t = \frac{1}{2}\sqrt{\frac{n}{\abs{\IndSlarge_1}}}$ yields
    \begin{equation*}
        \norm*{\frac{\hat{\P}_1^\t\hat{\P}_1}n - \I_{\abs{\IndSlarge_1}}}\leq 1/2
    \end{equation*}
    with probability at least $1- \c \filsize^{-\bd}$
    for some absolute constant $\c$.

    Finally, since $\IndS_1\subseteq \IndSlarge_1$
    for all choices of $\reg$ and $\m$,
    \begin{equation*}
        \norm*{\frac{\P_1^\t\P_1}n-\I_{\abs{\IndS_1}}}
        \leq \norm*{\frac{\hat{\P}_1^\t\hat{\P}_1}n-\I_{\abs{\IndSlarge_1}}} \leq 1/2
    \end{equation*}
    with probability at least $1- \c \filsize^{-\bd}$
    uniformly over all $\reg$ and $\m$.
\end{proof}

\begin{lemma}[Bound on $\K_2$]
    \label{lem:K2}
    In the setting of \cref{ssec:12decomposition},
    if $\c < \floor{\filsize / 2} < \filsize < d/2$
    for $\c$ as in \cref{ass:genericity},
    we have
    \begin{equation*}
        \c_1 \leq \mineig{\K_2} \leq\norm*{\K_2} \leq \c_2
    \end{equation*}
    for some positive constants $\c_1, \c_2$
    with probability at least $1-\tilde{\c} {\filsize^{-(1-\bd)}}$.
\end{lemma}
\begin{proof}
    First, the condition on $d$
    implies $\filsize > \T$
    and ensures that we can apply
    \cref{lem:tail_kernel_bound,lem:tail_kernel_bound_late},
    and \cref{ass:genericity}.
    Furthermore, note that $\T - 1 > \bL + 2$.

    \Cref{prop:kern_eig} and the definition of $\Kf_2$ yield
    the following decomposition:
    \begin{equation*}
        \Kf_2(x,x') = \sum_{\ii = \bL + 2}^{\filsize}{
            \xi_\ii^{(\filsize)}
            \sum_{\substack{\diam(S)\leq \filsize \\ \abs{S} =\ii}}{
                \frac{\filsize + 1 - \diam(S)}{d\B(\ii,\filsize)} \Yf_{S}(x)\Yf_{S}(x')
            }
        }
        = \sum_{\ii = \bL + 2}^{\filsize}{
            \xi_\ii^{(\filsize)}\Qf^{(d,\filsize)}_\ii(x,x')
        } ,
    \end{equation*}
    where $\Qf^{(d,\filsize)}_\ii(x,x')$ is defined in \cref{eq:Qdef}.
    Then, using the triangle inequality with non-negativity of the $\xi_\ii^{(\filsize)}$
    from \cref{eq:ass_lb,eq:ass_psd} in \cref{ass:genericity}, we have
    \begin{align*}
        \norm{\K_2}
        &= \norm*{\sum_{\ii = \bL + 2}^{\filsize} \xi_\ii^{(\filsize)} \Q_\ii^{(d,\filsize)}}
        = \norm*{\sum_{\ii = \bL + 2}^{\T-1} \xi_\ii^{(\filsize)} \Q_\ii^{(d,\filsize)}
        + \sum_{\ii = \T}^{\filsize} \xi_\ii^{(\filsize)} \Q_\ii^{(d,\filsize)}} \\
        &\leq \sum_{\ii = \bL + 2}^{\T-1} \xi_\ii^{(\filsize)} \norm*{\Q_\ii^{(d,\filsize)}}
        + \norm*{\sum_{\ii = \T}^{\filsize} \xi_\ii^{(\filsize)} \Q_\ii^{(d,\filsize)}} \\
        &\overset{(i)}{\leq} 1.5 \sum_{\ii = \bL + 2}^{\T-1} \xi_\ii^{(\filsize)} +\c_3
        \overset{(ii)}{\leq} \c_2
        \quad \text{with probability $\geq 1-\c\filsize^{-(1-\bd)}$}.
    \end{align*}
    $\Q_{\ii}^{(d,\filsize)}$ is the kernel matrix
    corresponding to $\Qf^{(d,\filsize)}_\ii(x,x')$,
    $(i)$ uses \cref{lem:tail_kernel_bound,lem:tail_kernel_bound_late},
    and $(ii)$ uses non-negativity
    and additionally \cref{eq:ass_ub_2} in \cref{ass:genericity}.

    The lower bound follows similarly from
    \begin{equation*}
        \mineig{\K_2}
        \geq \sum_{\ii = \bL + 2}^{\filsize}{\xi_\ii^{(\filsize)} \mineig{\Q_\ii^{(d,\filsize)}}}
        \geq {\xi_{\bL+2}^{(\filsize)} \mineig{\Q_{\bL+2}^{(d,\filsize)}}}
        \overset{(i)}\geq \frac{1}{2} \xi_{\bL+2}^{(\filsize)}
        \overset{(ii)}\geq \c_1 ,
    \end{equation*}
    where $(i)$ follows from \cref{lem:tail_kernel_bound}
    with probability at least $1-\c\filsize^{-(1-\bd)}$,
    and $(ii)$ from \cref{eq:ass_lb} in \cref{ass:genericity}.

    Since \cref{lem:tail_kernel_bound} yields both the upper and lower bound
    for all $\ii$ uniformly
    with probability $1-\c\filsize^{-(1-\bd)}$,
    this concludes the proof.
\end{proof}

\begin{lemma}[Bound on $\Tr(\S_2)$]
    \label{lem:S2}
    In the setting of \cref{ssec:12decomposition},
    if $\c < \floor{\filsize / 2} < \filsize < d/2$
    for $\c$ as in \cref{ass:genericity},
    we have with probability at least $1-\tilde{\c} {\filsize^{-(1-\bd)}}$ that
    \begin{equation*}
        \Tr(\S_{2})
        \in \Theta(\filsize^{-(1-\bd)})
        \quad \text{and} \quad
        \norm{\S_2} \in \bigO\left(\frac{1}{d\filsize^{\bL+1}}\right),
    \end{equation*}
    where $\S_2$ is defined in \cref{ssec:12decomposition}.
\end{lemma}
\begin{proof}
    Throughout the proof, the conditions on $d$ and hence $\filsize \in \Theta(d^\beta)$
    ensure that we can apply \cref{ass:genericity,lem:tail_kernel_bound},
    as well as $\bL + 2 \leq \T < \floor{\filsize / 2}$.

    First, \cref{lem:K2} yields $\norm*{\K_2} \in \Theta(1)$
    with probability at least $1-\c' {\filsize^{-(1-\bd)}}$,
    which we will use throughout the proof.
    Next, we bound $\norm*{\S_2}$ in two steps.
    For this, remember that $\lambda_{\bm}$ is the largest eigenvalue corresponding
    to an eigenfunction $\Yf_S$ of degree $\abs{S} \geq \bL + 2$.

    \paragraph*{Proof that $\norm{\S_2}\leq \lambda_{\bm}\norm{\K_2}$}
    Define $\pp_\i \in \R^{n \times n}$
    with $[\pp_\i]_{i,j} \defeq \Yf_{S_\i}(x_i) \Yf_{S_\i}(x_j)$
    for all $i,j \in \{1, \dotsc, n\}$, $\i \in \IndS_2$,
    and let $\v$ be any vector in $\R^n$.
    Then,
    \begin{align*}
        \norm{\S_2\v}
        &\overset{(i)}{=} \norm*{\left(
            \sum_{\i \in \IndS_2}\lambda_{\i}^2 \pp_\i
        \right)\v} \\
        &\leq \norm*{
            \max_{\i \in \IndS_2}\{\lambda_{\i}\}
            \sum_{\i \in \IndS_2}\lambda_{\i} \pp_\i \v
        } \\
        &\overset{(ii)}{=} \lambda_{\bm} \norm*{\left(
            \sum_{\i \in \IndS_2}\lambda_{\i} \pp_\i
        \right)\v} \\
        &= \lambda_{\bm}\norm{\K_2\v},
    \end{align*}
    where $(i)$ follows from the definition of $\Sf_2$
    and $(ii)$ from the definition of $\bm$.

    \paragraph*{Proof that $\lambda_{\bm} \in \Theta\left(\frac{1}{d\filsize^{\bL + 1}}\right)$}
    We show that $\lambda_{\i} \in \bigO\left(\frac{1}{d\filsize^{\bL + 1}}\right)$
    for all $\i \in \IndS_2$,
    and that there exists $\tm \in \IndS_2$
    with $\lambda_{\tm} \in \Omega\left(\frac{1}{d\filsize^{\bL + 1}}\right)$.
    Since $\lambda_{\bm} = \max_{\i \in \IndS_2}{\lambda_{S_\i}}$,
    those two facts imply $\lambda_{\bm} \in \Theta\left(\frac{1}{d\filsize^{\bL + 1}}\right)$.

    Let $\i \in \IndS_2$ be arbitrary.
    \Cref{lem:eigendecay} yields
    \begin{align*}
        \lambda_{\i} &\in \begin{cases}
            \bigO\left(\frac{1}{d\filsize^{\abs{S_\i}- 1}}\right) & \abs{S_\i} < \T\\
            \bigO\left(\frac{1}{d\filsize^{\T- 1}}\right) & \abs{S_\i} \geq \T\\
        \end{cases}\\
        &\overset{(i)}\subseteq \bigO\left(\frac{1}{d\filsize^{\bL+1}}\right),
    \end{align*}
    where $(i)$ follows from $\abs{S_\i} \geq \bL + 2$ and $\T \geq \bL + 2$.

    Now we show that there exists $\tm$
    with $\lambda_{\tm} \in \Omega\left(\frac{1}{d\filsize^{\bL + 1}}\right)$.
    We choose $\tm$ with $S_{\tm} = \{1,2,\dots,\bL+2\}$.
    Note that $\bL+2 = \diam(S_{\tm}) \leq \filsize$
    and thus $\tm\in \IndS_2$.
    Next, \cref{prop:kern_eig} yields
    \begin{equation*}
        \lambda_{S_{\tm}} = \xi_{\bL+2}^{(\filsize)} \frac{(\filsize+1-\diam(S_{\tm}))}{d\B(\bL+2,\filsize)} \overset{(i)}\in \Theta\left(\frac{\filsize}{d\filsize^{\bL+2}}\right) \subseteq \Theta\left(\frac{1}{d\filsize^{\bL+1}}\right),
    \end{equation*}
    where $(i)$ follows from \cref{eq:ass_lb,eq:ass_ub_2} in \cref{ass:genericity}.

    Finally, combining the previous two results and $\norm{\K_2} \in \Theta(1)$, we have
    \begin{equation*}
        \norm{\S_2}
        \leq \lambda_{\bm}\norm{\K_2}
        \in \bigO\left(\lambda_{\bm}\right)
        = \bigO\left(\frac{1}{d\filsize^{\bL + 1}}\right) .
    \end{equation*}

    \paragraph*{Upper bound of $\Tr(\S_2)$}
    The upper bound also follows directly from the last two results:
    \begin{equation*}
        \Tr(\S_2)
        \leq n \norm{\S_2}
        \in \bigO\left(\frac{n}{d\filsize^{\bL + 1}}\right)
        = \bigO\left(\frac{d\filsize^{\bL+ \bd}}{d\filsize^{\bL+ 1}}\right)
        = \bigO(\filsize^{-(1-\bd)}) .
    \end{equation*}

    \paragraph*{Lower bound of $\Tr(\S_2)$}
    The lower bound requires a more refined argument.
    \begin{align*}
        \Sf_2(x,x')
        &= \sum_{\i \in \IndS_2} \lambda_{S_\i}^2\Yf_{S_\i}(x)\Yf_{S_\i}(x')
        \geq \sum_{\substack{\diam(S)\leq \filsize \\ \abs{S} = \bL + 2}}{
            \lambda_{S}^2\Yf_{S}(x)\Yf_{S}(x')
        } \\
        &\geq \sum_{\substack{\diam(S)\leq \floor{\filsize/2} \\ \abs{S} = \bL + 2}}\lambda_{S}^2\Yf_{S}(x)\Yf_{S}(x') \\
        &\overset{(i)}{=} (\xi_{\bL+2}^{(\filsize)})^2\sum_{\substack{\diam(S)\leq \floor{\filsize/2} \\ \abs{S} = \bL + 2}}\frac{(\filsize + 1 - \diam(S))^2}{d^2\B(\bL+2,\filsize)^2} \Yf_{S}(x)\Yf_{S}(x') \\
        &= \frac{(\xi_{\bL+2}^{(\filsize)})^2}{d\B(\bL+2,\filsize)}\sum_{\substack{\diam(S)\leq \floor{\filsize/2} \\ \abs{S} = \bL + 2}}(\filsize + 1 - \diam(S))\frac{\filsize + 1 - \diam(S)}{d\B(\bL+2,\filsize)} \Yf_{S}(x)\Yf_{S}(x') \\
        &\overset{(ii)}{\geq} \frac{(\xi_{\bL+2}^{(\filsize)})^2\filsize/2}{d\B(\bL+2,\filsize)}\sum_{\substack{\diam(S)\leq \floor{\filsize/2} \\ \abs{S} = \bL + 2}}\frac{\floor{\filsize/2} + 1 - \diam(S)}{d\B(\bL+2,\filsize)} \Yf_{S}(x)\Yf_{S}(x'),
    \end{align*}
    where $(i)$ follows from \cref{prop:kern_eig}.
    In $(ii)$, we use that, as long as $\diam(S)\leq \floor{\filsize/2}$,
    we have $q+1-\diam(S) \geq \frac{\filsize}{2}$
    and $q \geq  \floor{\filsize/2}$.
    Continuing the bound, we have
    \begin{align*}
        \Sf_2(x,x')
        &\geq \frac{(\xi_{\bL+2}^{(\filsize)})^2\filsize/2}{d\B(\bL+2,\filsize)}\sum_{\substack{\diam(S)\leq \floor{\filsize/2} \\ \abs{S} = \bL + 2}}\frac{\floor{\filsize/2} + 1 - \diam(S)}{d\B(\bL+2,\filsize)} \Yf_{S}(x)\Yf_{S}(x') \\
        &\overset{(iii)}{\geq}\c_{\bL}' \frac{(\xi_{\bL+2}^{(\filsize)})^2\filsize/2}{d\B(\bL+2,\filsize)}\sum_{\substack{\diam(S)\leq \floor{\filsize/2} \\ \abs{S} = \bL + 2}}\frac{\floor{\filsize/2} + 1 - \diam(S)}{d\B(\bL+2,\floor{\filsize/2})} \Yf_{S}(x)\Yf_{S}(x') \\
        &\overset{(iv)}{=}\c_{\bL}' \frac{(\xi_{\bL+2}^{(\filsize)})^2\filsize/2}{d\B(\bL+2,\filsize)}\Qf_{\bL +2}^{(d,\floor{\filsize/2})}(x,x') .
    \end{align*}
    In $(iii)$, we use the classical bound on the binomial coefficient
    $\binom{\filsize}{\bL+2}=\B(\bL+2,\filsize)$ as follows:
    \begin{equation*}
        \B(\bL+2,\filsize)
        \leq {(2e)^{\bL+2}}\left(\frac{\filsize/2}{\bL+2}\right)^{\bL+2}
        \leq{(2e)^{\bL+2}\c_{\bL}}\left(\frac{\floor{\filsize/2}}{\bL+2}\right)^{\bL+2}
        = \c_{\bL}'\B(\bL+2,\floor{\filsize/2}),
    \end{equation*}
    where $\c_{\bL}'$ is a constant that depends only on $\bL$.
    Finally, $(iv)$ follows from the definition
    of $\Qf_{\bL +2}^{(d,\floor{\filsize/2})}$ in \cref{eq:Qdef}.

    The bound on $\Sf_2(x,x')$ implies
    \begin{equation*}
        \mineig{\S_{2}}
        \geq \c_{\bL}' \frac{(\xi_{\bL+2}^{(\filsize)})^2\filsize/2}{d\B(\bL+2,\filsize)}
        \mineig{\Q_{\bL +2}^{(d,\floor{\filsize/2})}},
    \end{equation*}
    and allows us to ultimately lower-bound $\Tr(\S_2)$ as follows:
    \begin{align*}
        \Tr(\S_2)
        &\geq n\mineig{\S_{2}}
        \geq \c_{\bL}' n\frac{
            (\xi_{\bL+2}^{(\filsize)})^2\filsize/2
        }{
            d\B(\bL+2,\filsize)
        }
        \mineig{\Q_{\bL +2}^{(d,\floor{\filsize/2})}} \\
        &\overset{(i)}{\gtrsim} \c_{\bL}' n\frac{
            (\xi_{\bL+2}^{(\filsize)})^2\filsize/2
        }{
            d\B(\bL+2,\filsize)
        } \\
        &\overset{(ii)}{\gtrsim} d^\Ln \frac{\filsize}{d \filsize^{\bL+2}}
        = \frac{d\filsize^{\bL+\bd}}{d \filsize^{\bL+1}} = \filsize^{-(1-\bd)},
    \end{align*}
    where $(i)$ follows from the lower bound on $\mineig{\Q_{\bL +2}^{(d,\floor{\filsize/2})}}$
    in \cref{lem:tail_kernel_bound}
    with probability at least $1-c''\floor{\filsize/2}^{-(1-\bd)} \geq 1-c''\filsize^{-(1-\bd)}$,
    and $(ii)$ follows from \cref{eq:ass_lb} in \cref{ass:genericity}
    and the classical lower bound on the binomial coefficient.
    This yields the desired lower bound $\Tr(\S_2) \in \Omega(\filsize^{-(1-\bd)})$.

    Finally, collecting all error probabilities concludes the proof.
\end{proof}

\begin{lemma}[Diagonal elements of $\Qf^{(d,\filsize)}$]
    \label{lem:Q11}
    Let $\ii, \filsize, d \in \N$
    with $0 < \ii \leq \filsize < d/2$
    and $x\in \{-1,1\}^d$. Then,
    \begin{equation*}
        \Qf_{\ii}^{(d,\filsize)}(x,x) = 1.
    \end{equation*}
\end{lemma}
\begin{proof}
    First,
    \begin{align*}
        \Qf^{(d,\filsize)}_\ii(x,x)
        &= \sum_{\substack{\diam(S)\leq \filsize \\ \abs{S} =\ii}}{
            \frac{\filsize + 1 - \diam(S)}{d\B(\ii,\filsize)} \Yf_{S}(x)^2
        } \\
        &\overset{(i)}{=} \frac{1}{d\B(\ii,\filsize)}
        \sum_{\substack{\diam(S)\leq \filsize \\ \abs{S} =\ii}}{
            \filsize + 1 - \diam(S)
        } \\
        &= \frac1{d\B(\ii,\filsize)}
        \sum_{\diam = \ii}^{\filsize}{
            ({\filsize + 1 - \diam})
            \sum_{\substack{\diam(S)= \diam \\ \abs{S} =\ii}} 1
        }
    \end{align*}
    where $(i)$ follows from the fact that $\Yf_{S}(x)^2 = 1$.

    Note that
    $\sum_{\substack{\diam(S)= \diam \\ \abs{S} =\ii}} 1$
    matches the definition of $\tilde{\C}(\ii, \diam,d)$
    in the proof of \cref{lem:n_sets}.
    Next, we use the following recurrence:
    \begin{align*}
        \sum_{\diam = \ii}^\filsize{
            (\filsize+1-\diam)\tilde{\C}(\ii,\diam,d)
         }
         &= \sum_{\diam = \ii}^\filsize{
            \left(\tilde{\C}(\ii,\diam,d) + ((\filsize-1)+1-\diam)\tilde{\C}(\ii,\diam,d)\right)
         } \\
        &= {\C}(\ii,\filsize,d)
        + 0 + \sum_{\diam = \ii}^{\filsize-1}{
            ((\filsize-1)+1-\diam)\tilde{\C}(\ii,\diam,d)
        },
    \end{align*}
    where the last step uses the fact that
    $\C(\ii, \filsize, d) = \sum_{\diam = \ii}^\filsize{\tilde{\C}(\ii,\diam,d)}$
    by definition,
    and that the term corresponding to $\diam = \filsize$ in the second sum is zero.
    Recursively applying this formula $\filsize-\ii$ times yields
    \begin{equation*}
        \sum_{\diam = \ii}^\filsize{(\filsize+1-\diam)\tilde{\C}(\ii,\diam,d)}
        = \sum_{\diam = \ii}^\filsize{\C(\ii,\diam,d)} .
    \end{equation*}

    Using this identity, we finally get
    \begin{align*}
        \Qf^{(d,\filsize)}_\ii(x,x)
        &= \frac1{d\B(\ii,\filsize)}
        \sum_{\diam = \ii}^{\filsize}{\C(\ii,\diam,d)}
        \overset{(i)}= \frac1{d\B(\ii,\filsize)}
        \sum_{\diam = \ii}^{\filsize} d \binom{\diam-1}{\ii-1}
        \overset{(ii)} = \frac1{d\B(\ii,\filsize)} d\binom{\filsize}{\ii}
        \overset{(iii)}=1,
    \end{align*}
    where $(i)$ follows from \cref{lem:n_sets},
    $(ii)$ from the hockey-stick identity,
    and $(iii)$ from the definition of $\B(\ii,\filsize)$ in \cref{eq:b_formula}.
\end{proof}

\subsection{Random matrix theory lemmas}\label{ssec:rmt_lemmas}
We use the following lemmas related to random matrix theory in the proof of \cref{lem:combined}.
The first two results bound the kernel's intermediate and late eigenvalues.
\begin{lemma}[Bound on the kernel's intermediate tail]
    \label{lem:tail_kernel_bound}
    In the setting of \cref{ssec:12decomposition},
    for $\T - 1 \leq \filsize < d/2$,
    with probability at least $1-c\filsize^{-(1-\bd)}$,
    all $\ii \in \N$ with $\bL+2 \leq \ii <\T$ satisfy
    \begin{equation*}
        \norm{\Q_{\ii}^{(d,\filsize)} - \I_n} \leq \frac12,
    \end{equation*}
    where $\T$ is defined in \cref{ass:genericity}
    and $\Q$ in \cref{eq:Qdef}.
\end{lemma}
This lemma particularly implies $\norm{\Q_{\ii}^{(d,\filsize)}}\leq 1.5$
and $\mineig{\Q_{\ii}^{(d,\filsize)}}\geq 1/2$
for all $\bL+2 \leq \ii <\T$ with high probability.
\begin{lemma}[Bound on the kernel's late tail]
    \label{lem:tail_kernel_bound_late}
    In the setting of \cref{ssec:12decomposition},
    if $\c \leq \floor{\filsize / 2} - 1 < \filsize < d/2$,
    for $\c$ and $\T$ as in \cref{ass:genericity},
    we have
    \begin{equation*}
        \norm*{\sum_{\ii = \T}^{\filsize} \xi_\ii^{(\filsize)}\Q^{(d,\filsize)}_\ii}
        \lesssim 1
    \end{equation*}
    with probability at least $1-\tilde{\c} d^{-1}$.
\end{lemma}

\begin{proof}[Proof of \cref{lem:tail_kernel_bound}]
    First,
    \cref{lem:Q11} yields that the diagonal elements of $\Q_{\ii}^{(d,\filsize)}$ are just $1$.
    Hence, we define $\W_{\ii}^{(d,\filsize)} \coloneqq \Q_{\ii}^{(d,\filsize)}-\I_n$,
    and want to show that, with high probability,
    $\norm{\W_{\ii}^{(d,\filsize)}}\leq 1/2$ for all $\bL+2 \leq \ii < \T$ at the same time.

    The proof makes use of \cref{lem:ghorbani}.
    Therefore, we need to find an appropriate $\M^{(\ii,\filsize)}$
    for each considered $\ii$,
    and show that the conditions of the lemma hold.

    The first condition follows directly from the construction of
    $\W_{\ii}^{(d,\filsize)}$ and $\diag(\Q_{\ii}^{(d,\filsize)}) = \I_n$.

    To establish the second condition, we have for all $i\neq j, k \neq j$
    \begin{align*}
        &\abs*{\E_{x_j}\left[[\W_{\ii}^{(d,\filsize)}]_{i,j}[\W_{\ii}^{(d,\filsize)}]_{j,k}\right]} \\
        &= \abs*{\frac{1}{(d\B(\ii,\filsize))^2}\sum_{\substack{\diam(S)\leq \filsize \\ \abs{S} =\ii}}\sum_{\substack{\diam(S')\leq \filsize \\ \abs{S'} =\ii}}{(\filsize + 1 - \diam(S))(\filsize + 1 - \diam(S'))}\Yf_{S}(x_i) \E\left[\Yf_{S}(x_j) \Yf_{S'}(x_j)\right]\Yf_{S'}(x_k)} \\
        &\overset{(i)}= \frac{1}{(d\B(\ii,\filsize))^2}\abs*{\sum_{\substack{\diam(S)\leq \filsize \\ \abs{S} =\ii}}(\filsize + 1 - \diam(S))^2\Yf_{S}(x_i)\Yf_{S}(x_k)} \\
        &\leq \frac{\filsize}{d\B(\ii,\filsize)}\frac{1}{d\B(\ii,\filsize)}\abs*{\sum_{\substack{\diam(S)\leq \filsize \\ \abs{S} =\ii}}(\filsize + 1 - \diam(S))\Yf_{S}(x_i)\Yf_{S}(x_k) }\\
        &= \left({\frac{d\B(\ii,\filsize)}\filsize}\right)^{-1}\abs*{\left[\W_{\ii}^{(d,\filsize)}\right]_{i,k}} \leq \left({\frac{d\B(\ii,\filsize)}{\ii\filsize}}\right)^{-1}\abs*{\left[\W_{\ii}^{(d,\filsize)}\right]_{i,k}},
    \end{align*}
    where $(i)$ follows from orthogonality of the eigenfunctions.
    Hence,
    \begin{equation*}
        \M^{(\ii,\filsize)}= d\B(\ii,\filsize)/\ii \filsize
    \end{equation*}
    satisfies the second condition in \cref{lem:ghorbani} for all $\bL+2\leq \ii < \T$.

    The extra $\ii$ factor is necessary for the third condition to hold.
    As in \cref{lem:ghorbani}, let $p\in\N, p\geq 2$.
    Then, for all $i\neq j$,
    \begin{align*}
        \E\left[\abs{[\W_{\ii}^{(d,\filsize)}]_{i,j}}^p\right]^{\frac1p}
        &\overset{(i)}{\leq} (p-1)^{\ii/2}\sqrt{\E\left[[\W_{\ii}^{(d,\filsize)}]_{i,j}^2\right]}
        \overset{(ii)}\leq \sqrt{p^{\ii} \sum_{\substack{\diam(S) \leq \filsize \\ \abs{S} =\ii}} \frac{(\filsize + 1 - \diam(S))^2}{(d\B(\ii,\filsize))^2}}\\
        &\overset{(iii)}\leq \sqrt{p^\ii\frac{\filsize^2}{(d\B(\ii,\filsize))^2} \C(\ii,\filsize,d)}=\sqrt{p^\ii \frac{\ii\filsize}{d\B(\ii,\filsize)} }\sqrt{ \frac{\filsize \C(\ii,\filsize,d)}{\ii d\B(\ii,\filsize)}}
        \overset{(iv)}{=} \sqrt{ \frac{p^\ii}{\M^{(\ii,\filsize)} }},
    \end{align*}
    where $(i)$ follows from hypercontractivity in \cref{lem:hypercontractivity},
    $(ii)$ from orthogonality of the eigenfunctions,
    and $(iii)$ from the definition of $\C(\ii,\filsize,d)$
    as well as $1 \leq \diam(S) \leq \filsize$.
    Step $(iv)$ follows from \cref{lem:n_sets},
    the definition of $\B(\ii,\filsize)$ in \cref{eq:b_formula},
    $\binom{\filsize-1}{\ii-1}\frac{\filsize}{\ii} = \binom{\filsize}{\ii}$,
    and the definition of $\M^{(\ii,\filsize)}$.

    Since all conditions are satisfied,
    \cref{lem:ghorbani} yields for all $p\in \N, p > 2$
    \begin{equation*}
        \Pr\left(\norm{\W_{\ii}^{(d,\filsize)}}>1/2\right)
        \leq \c_{\ii, 1}^pp^{3p} n\left(\frac{n}{\M^{(\ii,\filsize)}}\right)^p
        + \c_{\ii, 2}^p \left(\frac{n}{\M^{(\ii,\filsize)}}\right)^2,
    \end{equation*}
    where $\c_{\ii, 1}, \c_{\ii, 2}$ are positive constants that depend on $\ii$.
    In particular, if $p \geq 2+\frac{\frac1\beta + \bL+\bd}{1-\bd}$, then we get
    \begin{equation*}
        \Pr\left(\norm{\W_{\ii}^{(d,\filsize)}}>1/2\right) \leq  \c_{\ii}\filsize^{-2(1-\bd)},
    \end{equation*}
    where $\c_\ii$ is a positive constant that depends on $\ii$.
    To avoid this dependence,
    we can take the union bound over all $\ii = \bL +2,\dots,\T-1$:
    \begin{equation*}
        \Pr\left(
            \exists \ii \in \{\bL +2,\dots,\T-1\} \colon \norm{\W_{\ii}^{(d,\filsize)}}>1/2
        \right)
        \leq \filsize^{-2(1-\bd)}
        \sum_{\ii = \bL +2}^{\T-1}{\c_{\ii}} \leq \c'_{\bL}\filsize^{-2(1-\bd)},
    \end{equation*}
    where $\c'_{\bL}$ only depends on $\bL$ and $\T$,
    which are fixed in our setting.
    Finally, additionally note that neither $\bL$ nor $\T$ depend on $\Lr$.
\end{proof}

\begin{proof}[Proof of \cref{lem:tail_kernel_bound_late}]
First,
\cref{lem:Q11} yields that the diagonal elements of $\Q_{\ii}^{(d,\filsize)}$ are just $1$
for all $\ii \in \{\T, \dotsc, \filsize\}$.
Hence, we define $\W_{\ii}^{(d,\filsize)} \coloneqq \Q_{\ii}^{(d,\filsize)}-\I_n$,
and decompose the kernel matrix as
\begin{equation*}
    \sum_{\ii = \T}^{\filsize} \xi_\ii^{(\filsize)}\Q^{(d,\filsize)}_\ii
    = \sum_{\ii = \T}^{\filsize} \xi_\ii^{(\filsize)}\W^{(d,\filsize)}_\ii
    + \sum_{\ii = \T}^{\filsize} \xi_\ii^{(\filsize)}\I_n .
\end{equation*}
We can hence apply the triangle inequality to bound the norm as follows:
\begin{align*}
    \norm*{\sum_{\ii = \T}^{\filsize} \xi_\ii^{(\filsize)}\Q^{(d,\filsize)}_\ii } &\leq \norm*{\sum_{\ii = \T}^{\filsize} \xi_\ii^{(\filsize)}\W^{(d,\filsize)}_\ii }+\norm*{\sum_{\ii = \T}^{\filsize} \xi_\ii^{(\filsize)}\I_n}\notag\\
    &\overset{(i)}{\leq}
    \sum_{\ii = \T}^{\filsize}{ \xi_\ii^{(\filsize)}\norm{\W^{(d,\filsize)}_\ii } }
    +\sum_{\ii = \T}^{\filsize}{ \xi_\ii^{(\filsize)}\norm{ \I_n} } \\
    &\overset{(ii)}{\leq} \sum_{\ii = \T}^{\filsize}{
        \xi_\ii^{(\filsize)}\sqrt{\sum_{i\neq j}^n[\W^{(d,\filsize)}_\ii]_{i,j}^2 }
    } + \sum_{\ii = \T}^{\filsize}\xi_\ii^{(\filsize)} \\
    &\overset{(iii)}{\leq} \sum_{\ii = \T}^{\filsize}{
        \underbrace{
            \xi_\ii^{(\filsize)}\sqrt{n^2\max_{i\neq j}[\W^{(d,\filsize)}_\ii]_{i,j}^2 }
        }_{\revdefeq \elem_\ii}
    } + \c'' \\
    &\overset{(iv)}{\leq} \sum_{\ii = \T}^{\filsize} \frac{\c'}{\filsize}+ \c ''
    \lesssim 1, \quad \text{with probability} \geq 1-\c d^{-1},
\end{align*}
where $(i)$ uses non-negativity of the $\xi_\ii^{(\filsize)}$
from \cref{eq:ass_lb,eq:ass_psd} in \cref{ass:genericity},
$(ii)$ bounds the operator norm with the Frobenius norm,
and $(iii)$ additionally bounds the sum of the $\xi_{\ii}^{(\filsize)}$
using \cref{eq:ass_ub_2} in \cref{ass:genericity}.
Step $(iv)$ use a bound that we show in the remainder of the proof:
with probability at least $1-\c d^{-1}$,
we have $\elem_\ii \leq \c'/\filsize$
uniformly over all $\ii \in \{\T, \dotsc, \filsize\}$.

We first bound $\elem_\ii$ for a fixed $\T \leq \ii \leq \filsize$ as follows:
\begin{align*}
    &\Pr(\elem_\ii > 1/\filsize)
    = \Pr\left(
        \xi_\ii^{(\filsize)}\sqrt{n^2\max_{i\neq j}[\W^{(d,\filsize)}_\ii]_{i,j}^2 }> 1/\filsize
    \right)\\
    &= \Pr\left(
        \max_{i\neq j}[\W^{(d,\filsize)}_\ii]_{i,j}^2 > \frac{1}{n^2(\xi_\ii^{(\filsize)})^2\filsize^2}
    \right)
    = \Pr\left(
        \exists i\neq j \colon [\W^{(d,\filsize)}_\ii]_{i,j}^2  >\frac{1}{n^2(\xi_\ii^{(\filsize)})^2\filsize^2}
    \right) \\
    &\overset{(i)}{\leq}
    n(n-1) \Pr\left(
        [\W^{(d,\filsize)}_\ii]_{1,2}^2  >\frac{1}{n^2(\xi_\ii^{(\filsize)})^2\filsize^2}
    \right) \\
    &\overset{(ii)}{\leq}
    n^4(\xi_\ii^{(\filsize)})^2\filsize^2 \E\left[[\W^{(d,\filsize)}_\ii]_{1,2}^2\right]\\
    &\overset{(iii)}{=}
    \frac{n^4(\xi_\ii^{(\filsize)})^2\filsize^2}{(d\B(\ii,\filsize))^2}
    \sum_{\substack{\diam(S),\diam(S') \leq \filsize \\ \abs{S},\abs{S'} =\ii}}{
        (\filsize+1-\diam(S))(\filsize+1-\diam(S'))
        \underbrace{\E\left[
            \Yf_{S}(x_1)\Yf_{S}(x_2)\Yf_{S'}(x_1)\Yf_{S'}(x_2)
        \right]}_{\delta_{S,S'}}
    }\\
    &= \frac{n^4(\xi_\ii^{(\filsize)})^2\filsize^2}{(d\B(\ii,\filsize))^2}
    \sum_{\substack{\diam(S)\leq \filsize \\ \abs{S} =\ii}} (\filsize+1-\diam(S))^2 \\
    &\overset{(iv)}{\leq}
    \frac{(\filsize+1-\ii )n^4(\xi_\ii^{(\filsize)})^2\filsize^2}{d\B(\ii,\filsize)}
    \sum_{\substack{\diam(S)\leq \filsize \\ \abs{S} =\ii}}{
        \frac{(\filsize+1-\diam(S))}{d\B(\ii,\filsize)} \Yf_{S}(x_1)^2
     } \\
    &\leq \frac{n^4(\xi_\ii^{(\filsize)})^2\filsize^3 }{d\B(\ii,\filsize)}
    \Qf^{(d,\filsize)}_{\ii}(x_1,x_1)
    \overset{(v)}{=} \frac{ n^4\filsize^3}{d} \frac{(\xi_\ii^{(\filsize)})^2}{\B(\ii,\filsize)},
\end{align*}
where $(i)$ follows from the union bound
and the distribution of the off-diagonal entries in $\W^{(d,\filsize)}$,
and $(ii)$ from the Markov inequality.
In step $(iii)$, we use orthogonality of the eigenfunctions,
as well as the fact that $\W_{\ii}^{(d,\filsize)}$ and $\Q_{\ii}^{(d,\filsize)}$
coincide on off-diagonal entries by construction.
Step $(iv)$ follows from $\Yf_S(x)^2=1$ for all $S$ and $x \in \{-1,1\}^d$.
Finally, step $(v)$ applies \cref{lem:Q11}.

Next, we use the union bound over all $\elem_\ii$ of interest:
\begin{align*}
    &\Pr\left(\exists \ii \in \{\T,\dots, \filsize\} \colon \elem_\ii > 1/q\right)
    \leq \sum_{\ii = \T}^\filsize \Pr\left(\elem_\ii > 1/q\right)\\
    &\leq \sum_{\ii = \T}^\filsize \frac{ n^4\filsize^3}{d} \frac{(\xi_\ii^{(\filsize)})^2}{\B(\ii,\filsize)}\\
    &= \frac{ n^4\filsize^3}{d}\left(
        \sum_{\ii = \T}^{\ceil{\filsize/2}} \frac{(\xi_\ii^{(\filsize)})^2}{\binom{\filsize}{\ii}}
        + \sum_{\ii = \ceil{\filsize/2} + 1}^{\filsize-\T} \frac{(\xi_\ii^{(\filsize)})^2}{\binom{\filsize}{\ii}}
        + \sum_{\ii = \filsize-\T+1}^\filsize \frac{(\xi_\ii^{(\filsize)})^2}{\binom{\filsize}{\ii}}
    \right) \\
    &\overset{(i)}{=} \frac{ n^4\filsize^3}{d}\left(
        \sum_{\ii = \T}^{\ceil{\filsize/2}} \frac{(\xi_\ii^{(\filsize)})^2}{\binom{\filsize}{\ii}}
        + \sum_{\ii' = \T}^{\filsize-\ceil{\filsize/2}-1} \frac{(\xi_{\filsize - \ii'}^{(\filsize)})^2}{\binom{\filsize}{\filsize-\ii'}}
        + \sum_{\ii' = 0}^{\T-1}\frac{(\xi_{\filsize-\ii'}^{(\filsize)})^2}{\binom{\filsize}{\filsize-\ii'}}
    \right) \\
    &\overset{(ii)}{\leq}
    \frac{ n^4\filsize^3}{d}\left(
        \underbrace{
            \sum_{\ii = \T}^{\ceil{\filsize/2}}{\frac{
                (\xi_\ii^{(\filsize)})^2 + (\xi_{\filsize - \ii}^{(\filsize)})^2
            }{\binom{\filsize}{\ii}}}
        }_{\revdefeq \term_1}
        + \underbrace{
            \sum_{\ii' = 0}^{\T-1}\frac{(\xi_{\filsize-\ii'}^{(\filsize)})^2}{\binom{\filsize}{\ii'}}
        }_{\revdefeq \term_2}
    \right),
\end{align*}
where $(i)$ substitutes $\ii' = \filsize- \ii$,
and $(ii)$ uses $\filsize-\ceil{\filsize/2}-1\leq \ceil{\filsize/2}$ as well as the fact that $\binom{\filsize}{\filsize-\ii'} = \binom{\filsize}{\ii'}$.
We bound both $\term_1$ and $\term_2$ using \cref{ass:genericity}.

For $\term_1$ in particular, \cref{eq:ass_lb,eq:ass_psd,eq:ass_ub_2}
imply that all $\xi_\ii^{(\filsize)} \lesssim 1$.
Hence,
\begin{equation*}
    \term_1 = \sum_{\ii = \T}^{\ceil{\filsize/2}}{\frac{
        (\xi_\ii^{(\filsize)})^2 + (\xi_{\filsize - \ii}^{(\filsize)})^2
    }{\binom{\filsize}{\ii}}}
    \lesssim \sum_{\ii = \T}^{\ceil{\filsize/2}}{\frac{1}{\binom{\filsize}{\ii}}}
    \overset{(i)}{\leq} \sum_{\ii = \T}^{\ceil{\filsize/2}}{
        \frac{1}{\binom{\filsize}{\T}}
    }
    = \frac{\ceil{\filsize/2}-\T+1}{\binom{\filsize}{\T}}
    \overset{(ii)}{\leq} \frac{\T^\T \filsize}{\filsize^{\T}}
    \lesssim \frac{1}{\filsize^{\T-1}},
\end{equation*}
where $(i)$ exploits
that $\T$ is the value in $\{\T, \dots, \ceil{\filsize/2}\}$ the furthest away from $\filsize/2$,
and thus
\begin{equation*}
    \min_{\ii \in \{\T, \dots, \ceil{\filsize/2}\}}\binom{\filsize}{\ii} = \binom{\filsize}{\T},
\end{equation*}
and $(ii)$ follows from the classical lower bound on the binomial coefficient.

For $\term_2$, we have
\begin{equation*}
    \term_2
    = \sum_{\ii' = 0}^{\T-1}\frac{(\xi_{\filsize-\ii'}^{(\filsize)})^2}{\binom{\filsize}{\ii'}}
    \overset{(i)}{\leq} \sum_{\ii' = 0}^{\T-1}{
        \frac{\ii'^{\ii'}(\xi_{\filsize-\ii'}^{(\filsize)})^2}{\filsize^{\ii'}}
    }
    \overset{(ii)}{\leq} \sum_{\ii' = 0}^{\T-1}{
        \frac{\ii'^{\ii'}\left(\frac{\c'}{\filsize^{\T-\ii'+1}} \right)^2}{\filsize^{\ii'}}
    }
    \lesssim \sum_{\ii' = 0}^{\T-1}{\frac{\ii'^{\ii'}}{\filsize^{2\T-2\ii'+2+\ii'}}}
    \leq \T \frac{\T^{\T}}{\filsize^{\T+2}}
    \lesssim \frac{1}{\filsize^{\T-1}},
\end{equation*}
where $(i)$ uses the classical bound on the binomial coefficient,
and $(ii)$ \cref{eq:ass_ub_1} in \cref{ass:genericity}.

Combining the bounds on $\term_1$ and $\term_2$ finally yields
\begin{align*}
    \Pr\left(\exists \ii \in \{\T,\dots, \filsize\} \colon \elem_\ii > 1/q\right)
    &\leq \sum_{\ii = \T}^\filsize \Pr\left(\elem_\ii > 1/q\right)
    \leq \frac{n^4\filsize^{3}}{d} (\term_1 + \term_2)\\
    &\lesssim \frac{n^4}{d\filsize^{\T-4}} \lesssim d^{4\Ln-1-\beta(\T-4)} \overset{(i)}\lesssim \frac{1}{d},
\end{align*}
where $(i)$ follows from the definition
of $\T = \ceil{4+\frac{4\Ln}{\beta}}$ in \cref{ass:genericity}.
\end{proof}

The next statement is a non-asymptotic version of Proposition~3 from \citet{Ghorbani19}.
\begin{lemma}[Graph argument]\label{lem:ghorbani}
    Let $\W\in \R^{n\times n}$ be a random matrix that satisfies the following conditions:
    \begin{enumerate}
        \item $[\W]_{i,i} = 0, \; \forall i \in \{1,\dots,n\}$.
        \item There exists $\M>0$ such that,
        for all $i,j,k \in \{1,\dots,n\}$ with $i\neq j$ and $j \neq k$, we have
        \begin{equation*}
          \abs*{\E_{x_j}\left[[\W]_{i,j} [\W]_{j,k}\right]}
          \leq \frac{1}{\M} \abs*{[\W]_{i,k}}.
        \end{equation*}
        \item There exists $\ii \in \N$ such that, for all $p\in\N$, $p\geq 2$
        and all $i,j \in \{1,\dots,n\}$, $i\neq j$, we have
        \begin{equation*}
          \E\left[\abs{[\W]_{i,j}}^p\right]^{1/p} \leq \sqrt{\frac{p^{\ii}}{\M}}.
        \end{equation*}
    \end{enumerate}

    Then, for all $p\in \N$, $p > 2$,
    \begin{equation*}
    \Pr\left(\norm{\W}>1/2\right)
    \leq \c_1^pp^{3p} n\left(\frac{n}{\M}\right)^p+\c_2^p \left(\frac{n}{\M}\right)^2,
    \end{equation*}
    where $\c_1$ and $\c_2$ are positive constants that depend on $\ii$.
\end{lemma}
\begin{proof}
    Repeating the steps in the proof of Proposition~3 from \cite{Ghorbani19}, we get
    \begin{equation*}
    \E[\norm{\W}^{2p}]\leq \E\left[\Tr(\W^{2p})\right]
    \leq (\c p)^{3p}\frac{n^{p+1}}{\M^{ p}} + \c'^p\left(\frac{n}{\M}\right)^2.
    \end{equation*}
    Note that the proof in \cite{Ghorbani19} assumes $\M$ to be in the order of $d^\ii$.
    We get rid of this assumption and keep $\M$ explicit.
    Furthermore, \citet{Ghorbani19} use their Lemma~4 during their proof,
    but we use our \cref{lem:hypercontractivity} instead.

    Ultimately, we apply the Markov inequality to get a high-probability bound:
    \begin{align*}
    \Pr(\norm{\W}\geq 1/2)
    &= \Pr\left(\norm{\W}^{2p}\geq (1/2)^{2p}\right) \\
    &\leq \frac{\E[\norm{\W}^{2p}]}{(1/2)^{2p}}
    = \left(\frac{\c^{3}}{(1/2)^2}\right)^pp^{3p}\frac{n^{p+1}}{\M^{ p}}
    + \left(\frac{\c'}{(1/2)^2}\right)^p\left(\frac{n}{\M}\right)^2.
    \end{align*}
    Renaming the constants concludes the proof.
\end{proof}

\begin{lemma}[Hypercontractivity]
    \label{lem:hypercontractivity}
    For all $\ii, \filsize, d \in \N$ and $p\geq 2$, we have
    \begin{equation*}
        \E_{x,x' \sim \mathcal{U}(\2^d)}\left[
            \abs{\Qf^{(d,\filsize)}_\ii(x,x')}^p
        \right]^{1/p}
        \leq (p-1)^{\ii/2}\sqrt{\E_{x,x' \sim \mathcal{U}(\2^d)}\left[
            (\Qf^{(d,\filsize)}_\ii(x,x'))^2
        \right]},
    \end{equation*}
    where $\Qf^{(d,\filsize)}_\ii(x,x')$ is defined in \cref{eq:Qdef}.
\end{lemma}
\begin{proof}
    Let $x,x' \sim \mathcal{U}(\2^d)$
    and let $z$ be the entry-wise product of $x$ and $x'$.
    Then, for all $S \subseteq \{1, \dotsc, d\}$,
    $\Yf_S(x)\Yf_S(x')$ depends only on $z$:
    \begin{equation*}
        \Yf_S(x)\Yf_S(x')
        = \left(\prod_{i\in S}[x]_i\right) \left(\prod_{i\in S}[x']_i\right)
        = \prod_{i\in S}[x]_i[x']_i = \prod_{i\in S}[z]_i.
    \end{equation*}
    Hence, $\Qf^{(d,\filsize)}_\ii(x,x')$ also only depends on $x$ and $x'$ via $z$.
    Furthermore, note that $z \sim \U(\2^d)$.
    Therefore, we can use hypercontractivity \citep{bc75}
    as for instance in Lemma~4 from \cite{misiakiewicz2021learning}
    to conclude the proof.
\end{proof}

%% file: appendix/training_error.tex
\section{Optimal regularization and training error}
\label{app:training_error}

\subsection{Optimal regularization}
\label{app:opt_reg}
In the main text we often refer to the optimal regularization $\regopt$,
defined as the minimizer of the risk $\risk(\freg)$.
While we cannot calculate $\regopt$ directly,
we only need the rate $\Lropt$ such that
$\max\{\regopt, 1\} \in \Theta\left(d^{\Lropt}\right)$.
Furthermore, it is not a priori clear that such a $\Lropt$
minimizes the rate exponent of the risk in \cref{th:rates}.
The current subsection establishes that this is indeed the case,
and provides a way to determine $\Lropt$.

We introduce some shorthand notation for the rate exponents in \cref{th:rates}:
\begin{align*}
    \ratev(\Lr; \Ln, \Ls, \beta) &\defeq \frac{-\Ls-\Lr}{\Ln} - \frac{\beta}{\Ln} \min\{\delta, 1-\delta\},\\
    \rateb(\Lr; \Ln, \Lgt, \beta) &\defeq -2-\frac{2}{\Ln}(-\Lr-1-\beta(\Lgt-1)),\\
    \rate(\Lr; \Ln, \Ls,\Lgt, \beta) &\defeq \max\left\{\ratev(\Lr; \Ln, \Ls, \beta), \rateb(\Lr; \Ln, \Lgt, \beta)\right\}.
\end{align*}
We highlight that $\rateb$ and $\rate$ depend on $\Lr$
also through $\delta = \frac{\Ln-\Lr-1}{\beta} -\floor*{\frac{\Ln-\Lr-1}{\beta}}$.
Hence, in the setting of \cref{th:rates}, we have with high probability that
\begin{align*}
\var(\freg) &\in \Theta(n^{\ratev(\Lr; \Ln, \Ls, \beta)}),\\
\bias^2(\freg) &\in \Theta(n^{\rateb(\Lr; \Ln, \Lgt, \beta)}),\\
\risk^2(\freg) &\in \Theta(n^{\rate(\Lr; \Ln, \Ls,\Lgt, \beta)}).
\end{align*}
In the following, we view those quantities as functions of $\Lr$,
with all other parameters fixed.
Next, we additionally define
\begin{align*}
    \regopt &\defeq \argmin _{\lambda \geq 0 \mid \max{\{\lambda, 1\}} \in \bigO(d^{\Lru})} \risk(\freg),\\
    \Lrmin &\defeq \argmin_{\Lr \in \left[0,\Lru\right]} \rate(\Lr; \Ln, \Ls, \Lgt, \beta),\\
    \ratemin &\defeq \min_{\Lr \in \left[0,\Lru\right]} \rate(\Lr; \Ln, \Ls, \Lgt, \beta),\\
    \Lru &\defeq\Ln-1-\beta(\Lgt-1).
\end{align*}
First, we remark that
$\Lrmin$---the set of regularization rates that minimize the risk rate---might
have cardinality larger than one.
However, it cannot be empty:
$[0, \Lru]$ is a closed set,
and \cref{lem:eta_prop} below shows that $\rate$ is a continuous function.

Second, $\Lru$ defines the scope of the minimization domain,
guaranteeing that the constraint on $\Lgt$ in \cref{th:rates} holds
for all candidate $\Lr$.

\pseudoparagraph{Rate of optimal regularization $\Lropt$ vs.\ optimal rate $\Lrmin$}
Let $\Lropt$ be the rate of the optimal regularization strength
such that $\max{\{\regopt, 1\}} \in \Theta(d^{\Lropt})$.
It is a priori not clear that $\Lropt$ minimizes $\rate$.
However, \cref{lem:opt_reg} bridges the two quantities,
and guarantees with high probability that the rate of $\regopt$ minimizes the rate of the risk.
\begin{lemma}[Optimal regularization and optimal rate]
    \label{lem:opt_reg}
    In the setting of \cref{th:rates},
    assume $\Ln > 0$, $\beta \in (0, 1)$, $\Ls \geq -\Lru$,
    $\Lgt \in \left[1,\ceil{\frac{\Ln-1}{\beta}}\right]\cap \N$.
    Then, for $d$ sufficiently large,
    with probability at least $1-\c d^{-\beta\min\{\bd,1-\bd\}}$
    there exists $\l \in \Lrmin$ such that
    \begin{equation*}
        \max\{\regopt,1\} \in \Theta\left(d^{\l}\right).
    \end{equation*}
\end{lemma}

Hence, we only need to obtain a minimum rate
$\l \in \Lrmin$ instead of $\Lropt$.
In order to propose a method for this,
we first establish properties of $\rate,\rateb,\ratev$ in the following lemma.

\begin{lemma}[Properties of $\rate$]
    \label{lem:eta_prop}
    Assume $\Ln > 0$, $\beta \in (0, 1)$, $\Ls \geq -\Lru$,
    $\Lgt \in \left[1,\ceil{\frac{\Ln-1}{\beta}}\right]\cap \N$.
    \begin{enumerate}
        \item\label{item:monotonicity}
        Over $[0,\Lru]$,
        $\ratev(\cdot;\Ln,\Ls,\beta)$ is continuous and non-increasing,
        and $\rateb(\cdot;\Ln,\Lgt,\beta)$ is continuous and strictly increasing.
        \item\label{item:interval}
        $\Lrmin \defeq \argmin_{\Lr \in \left[0,\Lru\right]} \rate(\Lr; \Ln, \Ls, \Lgt, \beta)$
        is a closed interval.
        \item\label{item:minimum}
        If there exists $\lt \in \left[0, \Lru\right]$
        with $\ratev(\lt;\Ln,\Ls,\beta) = \rateb(\lt;\Ln,\Lgt,\beta)$,
        then $\ratemin = \rate(\lt;\Ln,\Ls,\Lgt,\beta)$.
        Otherwise, $\ratemin = \rate(0;\Ln,\Ls,\Lgt,\beta)$.
        \item\label{item:late}
        Every $\overl \in [0,\Lru]$
        with $\overl \geq \l$ for all $\l \in \Lrmin$
        satisfies
        \begin{equation*}
            \rate(\overl; \Ln,\Ls,\Lgt,\beta) = \rateb(\overl; \Ln, \Lgt, \beta).
        \end{equation*}
        \item\label{item:early}
        Let $\underl \in [0,\Lru]$
        with $\underl \leq \l$ for all $\l \in \Lrmin$.
        If $\rate(\underl; \Ln,\Ls,\Lgt,\beta) - \ratemin \leq \c$
        where $\c > 0$ is constant
        \footnote{We omit an explicit definition of $\c$ here for brevity and refer to the proof instead.}
        and depends only on $\Ln,\Ls,\Lgt,\beta$,
        then
        \begin{equation*}
            \rate(\underl; \Ln,\Ls,\Lgt,\beta)
            = \ratemin + \frac{2}{\Ln}\left(\min{(\Lrmin)}-\underl\right).
        \end{equation*}
    \end{enumerate}
\end{lemma}

\pseudoparagraph{Finding an optimal rate}
\Cref{lem:eta_prop} suggests a simple strategy to find a $\l\in\Lrmin$ numerically:
search the intersection of $\ratev$ and $\rateb$ in $[0,\Lru]$;
if found, then the intersection point is optimal, otherwise $\l = 0$ is optimal.
Note that, if the intersection point exists,
it is unique and easy to numerically approximate,
since $\ratev$ is non-increasing, and $\rateb$ is strictly increasing.

\pseudoparagraph{Calculating numerical solutions}
However, \cref{lem:eta_prop} also shows that $\Lrmin$
is an interval and thus might contain multiple values.
In that case, the proposed strategy might not necessarily retrieve the rate of $\regopt$,
but a different $\l\in \Lrmin$.
Yet, \cref{th:rates} guarantees
that both the optimally regularized estimator
and any estimator regularized with $\max{\{\reg, 1\}} \in d^{\l}$ for any $\l \in \Lrmin$
have a risk vanishing with the same rate $n^{\ratemin}$ with high probability.
In particular, this allows us to exhibit the rate
of the optimally regularized estimator in \cref{fig:rates_opt_reg}.
Finally, because of the multiple descent phenomenon
(see, for example, \cref{fig:theory}),
we do not expect either $\Lropt$ or $\betath$ to attain an easily readable closed-form expression.
Nevertheless, simple optimization procedures allow us to calculate
accurate numerical approximations.

Finally, we prove \cref{lem:opt_reg,lem:eta_prop}.
\begin{proof}[Proof of \cref{lem:opt_reg}]
    Let $\Lropt$ be such that
    \begin{equation*}
        \max{{\{\regopt, 1\}}} \in \Theta\left(d^{\Lropt}\right).
    \end{equation*}
    Combining the bounds on bias and variance from \cref{th:rates},
    we have
    \begin{align}
        \risk^2(\freg) &= \bias^2(\freg) + \var^{2}(\freg) \nonumber \\
        &\in \Theta\left(
            n^{-2-\frac{2}{\Ln}(-\Lr-1-\beta(\Lgt-1))}
            + n^{\frac{-\Ls-\Lr}{\Ln} - \frac{\beta}{\Ln} \min\{\delta, 1-\delta\}}
        \right) \nonumber \\
        &\in \Theta\left(n^{\rate(\Lr; \Ln, \Ls,\Lgt, \beta)}\right)
        \quad\text{with probability } \geq 1-cd^{-\min\{\bd,1-\bd\}} \label{eq:risk_full_bound}
    \end{align}
    uniformly for all $\Lr \in \left[0,\Lru\right]$
    with $\max\{\reg, 1\} \in \Theta(d^\Lr)$.

    Throughout the remainder of this proof,
    we drop the dependencies on $\Ln,\Ls,\Lgt,\beta$ in the notation of $\rate$
    and simply write $\rate(\Lr)$.
    We further omit repeating that each step is true
    with probability at least $1-cd^{-\min\{\bd,1-\bd\}}$,
    but imply it throughout.

    The goal of the proof is to show that there
    exists $\l \in \Lrmin$ sufficiently close to $\Lropt$;
    formally, we need to show that
    \begin{equation}\label{eq:lopt_like_l}
        d^{\Lropt-\l} \in \Theta(1).
    \end{equation}
    If this is the case, then the definition of $\Lropt$ yields the conclusion as follows:
    \begin{equation*}
        \max{{\{\regopt, 1\}}} \in \Theta\left(d^{\Lropt}\right)
        = \Theta\left(d^{\l}d^{\Lropt-\l}\right)
        \overset{(i)}{=} \Theta\left(d^{\l}\right),
    \end{equation*}
    where $(i)$ uses \cref{eq:lopt_like_l}.

    Towards an auxiliary result,
    we first apply \cref{eq:risk_full_bound} to $\Lropt$ and $\l\in \Lrmin$
    with $\max\{\reg, 1\} \in \Theta(d^\l)$:
    \begin{align*}
        \risk^2(\hat{f}_{\reg}) &\leq \c_2 n^{\rate(l)} = \c_2 n^{\ratemin},\\
        \risk^2(\fopt) &\geq \c_1 n^{\rate(\regopt)},
    \end{align*}
    where $\c_1,\c_2 > 0$ are the constants hidden by the $\Theta$-notation
    in \cref{eq:risk_full_bound}.
    Next, the optimality of $\regopt$ yields $\c_1<\c_2$, and
    \begin{equation*}
        \c_1 n^{\rate(\Lropt)}
        \leq \risk^2(\fopt)
        \leq \risk^2(\hat{f}_{l})
        \leq \c_2 n^{\ratemin} .
    \end{equation*}
    This implies
    \begin{equation}
        n^{\rate(\Lropt)-\ratemin} \leq \frac{\c_2}{\c_1}
        \quad \Rightarrow \quad
        \rate(\Lropt)-\ratemin \leq \frac{\log({\c_2/\c_1})}{\log n},
        \label{eq:opt_reg_aux}
    \end{equation}
    where the second implication uses that $\rate(\Lropt)-\ratemin\geq 0$,
    since $\ratemin$ is the minimum rate.

    With this result, we finally focus on establishing \cref{eq:lopt_like_l}
    which yields the claim of this lemma.
    \Cref{lem:eta_prop} shows that $\Lrmin$ is an interval;
    hence, we distinguish three cases:

    \pseudoparagraph{Case $\Lropt \in \Lrmin$}
    Picking $\Lropt = \l \in \Lrmin$ directly yields $d^{\Lropt-\Lropt} = d^0 \in \Theta(1)$.

    \pseudoparagraph{Case $\Lropt < \min{(\Lrmin)}$}
    Let $\underl \defeq \min(\Lrmin)$.
    \cref{eq:opt_reg_aux} yields
    $\rate(\Lropt) -\ratemin \leq \c$
    for any $\c > 0$ if $d$ is large enough.
    Hence, for $d$ fixed but sufficiently large,
    \cref{lem:eta_prop} yields
    \begin{equation*}
        \rate(\Lropt)
        = \ratemin + \frac{2}{\Ln}\left(\underl-\Lropt\right).
    \end{equation*}
    Applying \cref{eq:opt_reg_aux}, we get
    \begin{equation*}
        \underl-\Lropt
        = \frac{\Ln}{2}(\rate(\Lropt)-\ratemin)
        \leq \frac{\Ln\log(\c_2/\c_1)}{2 \log n}.
    \end{equation*}
    Now, since $\Lropt\leq \underl$, $d^{\Lropt-\underl} \in \bigO(1)$.
    Furthermore,
    \begin{equation*}
        d^{\underl-\Lropt}
        \leq d^{\frac{\Ln\log(\c_2/\c_1)}{2 \log n}}
        \overset{(i)}{=} d^{\frac{\log(\c_2/\c_1)}{2 \log d}}
        = \sqrt{\c_2/\c_1} \in \bigO(1),
    \end{equation*}
    where $(i)$ follows from $n\in \Theta(d^\Ln)$.
    Since both $d^{\underl-\Lropt}$ and $\frac1{d^{\underl-\Lropt}}=d^{\Lropt-\underl}$
    are in $\bigO(1)$,
    \begin{equation*}
        d^{\Lropt-\underl} \in \Theta(1),
    \end{equation*}
    which establishes \cref{eq:lopt_like_l} for $\overl \in \Lrmin$
    and thereby concludes this case.

    \pseudoparagraph{Case $\Lropt > \max{(\Lrmin)}$}
    Let $\overl \defeq \max(\Lrmin)$. Then, \cref{lem:eta_prop} yields
    \begin{align*}
        \rate(\Lropt) &= -2-\frac{2}{\Ln}(-\Lropt-1-\beta(\Lgt-1)) \\
        &= -2-\frac{2}{\Ln}(-\Lropt-\overl+\overl-1-\beta(\Lgt-1))\\
        &= -2-\frac{2}{\Ln}(-\overl-1-\beta(\Lgt-1)) - \frac{2}{\Ln}(\overl-\Lropt)\\
        &\overset{(i)}= \ratemin + \frac{2}{\Ln}(\Lropt-\overl),
    \end{align*}
    where $(i)$ uses that $\overl \in \Lrmin$.
    Applying \cref{eq:opt_reg_aux}, we get
    \begin{equation*}
        \Lropt-\overl
        = \frac{\Ln}{2}(\rate(\Lropt)-\ratemin)
        \leq \frac{\Ln\log(\c_2/\c_1)}{2 \log n}.
    \end{equation*}
    Analogously to the previous case,
    this implies $d^{\Lropt-\overl} \in \bigO(1)$,
    and the fact that $\Lropt > \overl$ implies $d^{\Lropt-\overl} \in \Omega(1)$.
    Together, this yields
    \begin{equation*}
        d^{\Lropt-\overl} \in \Theta(1),
    \end{equation*}
    which establishes \cref{eq:lopt_like_l} for $\overl \in \Lrmin$
    and thereby concludes this case.
\end{proof}

\begin{proof}[Proof of \cref{lem:eta_prop}]
    Throughout this proof,
    we drop the dependencies on $\Ln,\Ls,\Lgt,\beta$ in the notation of $\rate,\ratev,\rateb$
    and simply write $\rate(\l), \ratev(\l), \rateb(\l)$.

    \paragraph*{Continuity and monotonicity (\cref{item:monotonicity})}
    We first show that $\ratev(\l)$ and $\rateb(\l)$ are continuous functions.
    $\rateb(\l)$ is an affine function of $\l$, hence continuous.
    $\ratev(\l)$, however, additionally depends on $\l$ via $\delta$,
    which is not linear.
    Hence, to show that $\ratev(\l)$ is continuous,
    we need to show that $\min\{\delta,1-\delta\}$ is continuous.
    Consider the triangle wave function
    \begin{equation*}
        \dsaw(t) \defeq \min\{t-\floor{t}, 1-(t-\floor{t})\},
    \end{equation*}
    which is well-known to be continuous.
    Because $\min\{\delta,1-\delta\} = \dsaw\left(\frac{\Ln-\l-1}{\beta}\right)$,
    and $\frac{\Ln-\l-1}{\beta}$ is a linear function of $\l$,
    we get that $\ratev(\l)$ is also continuous.

    For monotonicity, we consider the derivatives
    of $\ratev(\l)$ and $\rateb(\l)$.
    For $\rateb(\l)$, we have
    \begin{equation*}
        \partial_\l \rateb(\l) = \frac{2}{\Ln}.
    \end{equation*}
    Since $\rateb$ is an affine function,
    and $\frac{2}{\Ln} > 0$,
    this also implies that $\rateb$ is strictly increasing.
    For $\ratev$, we need to distinguish two cases:
    \begin{align*}
        \partial_\l \ratev(\l) &= \begin{cases}
            \partial_\l\frac{-\Ls - \l - \beta \delta}{\Ln} & \delta < 1-\delta\\
            \partial_\l\frac{-\Ls - \l - \beta (1-\delta)}{\Ln} & \delta \geq 1-\delta
        \end{cases}\\
        &= \begin{cases}
            0 & \delta < 1-\delta , \\
            -\frac{2}{\Ln} & \delta \geq 1-\delta .
        \end{cases}
    \end{align*}
    Since $\ratev$ is a continuous function with non-positive derivatives,
    this implies that $\ratev$ is non-increasing.

    \paragraph*{$\rate$ decomposition (\cref{item:minimum})}
    First assume that there exists $\lt \in \left[0, \Lru\right]$
    with $\ratev(\lt) = \rateb(\lt)$.
    Since $\ratev$ is non-increasing and $\rateb$ is strictly increasing,
    \begin{equation*}
        \rate(\l) \defeq \max\{\rateb(\l), \ratev(\l)\}
        = \begin{cases}
            \ratev(\l)&\l < \lt,\\
            \rateb(\l)&\text{otherwise.}
        \end{cases}
    \end{equation*}
    In particular,
    $\rate(\l) > \rate(\lt)$ for all $\l > \lt$
    as $\rateb$ is strictly increasing,
    and $\rate(\l) \geq \rate(\lt)$ for all $\l < \lt$
    since $\ratev$ is non-increasing.
    Combined, this yields $\ratemin = \rate(\lt;\Ln,\Ls,\Lgt,\beta)$.

    Next, assume $\lt$ does not exist,
    that is, $\ratev(\l) \neq \rateb(\l)$
    for all $\l \in \left[0, \Lru\right]$.
    Then, due to continuity,
    either $\rateb(\l) > \ratev(\l)$ for all $\l \in \left[0, \Lru\right]$,
    or $\ratev(\l) > \rateb(\l)$ for all $\l \in \left[0, \Lru\right]$.
    However, a closer analysis shows that the latter is not possible:
    the rates at the boundary $\Lru$ are
    \begin{align*}
        \rateb(\Lru)
        &= -2 -\frac{2}{\Ln}\left(-\left(\Ln-1-\beta(\Lgt-1)\right)-1-\beta(\Lgt-1)\right)
        = 0, \\
        \ratev(\Lru)
        &= - \frac{\Ls+\left(\Ln-1-\beta(\Lgt-1)\right)}{\Ln}-\frac{\beta}{\Ln}\min\{\delta,1-\delta\}
        \overset{(i)} \leq 0,
    \end{align*}
    where $(i)$ follows from the assumption $\Ls \geq -\Lru$
    and the fact that $\min\{\delta,1-\delta\} \geq 0$.
    Hence, if $\lt$ does not exist,
    $\ratev(\l) < \rateb(\l), \forall \l \in [0,\Lru]$.
    In particular, $\ratemin$ is the minimum of the strictly increasing
    function $\rateb(\l)$ over $[\left[0, \Lru\right]$,
    and therefore attained only at $0$.

    Lastly, combining both cases of $\lt$ yields the following convenient expression:
    \begin{equation}
        \rate(\l) = \begin{cases}
            \ratev(\l)&\text{$\lt$ exists and $\l < \lt$,}\\
            \rateb(\l)&\text{otherwise.}
        \end{cases}
        \label{eq:rate_decomp}
    \end{equation}

    \paragraph*{Closed interval of solutions (\cref{item:interval})}
    We again differentiate whether $\lt$ as in the previous step exists or not.
    If $\lt$ does not exist,
    then the previous step already yields $\Lrmin = \{0\}$,
    which is a closed interval.
    Next, assume $\lt$ exists.
    Then, all $\l \in \Lrmin$ satisfy $\l \leq \lt$,
    since $\lt \in \Lrmin$ and $\rateb$ is strictly increasing.
    Since further $\ratev$ is continuous and non-increasing over
    $\left[0, \lt\right] \supseteq \Lrmin$, $\Lrmin$ is an interval.
    Finally, $\rate(\l) = \max\{\rateb(\l), \ratev(\l)\}$
    is the maximum of two continuous functions, hence itself also continuous.
    Therefore, the minimizers $\Lrmin$ of $\rate$ are a closed set.
    This concludes that $\Lrmin$ is a closed interval.

    \paragraph*{Proof of \cref{item:late,item:early}}
    \Cref{item:late} follows straightforwardly from \cref{eq:rate_decomp}:
    If $\lt$ does not exist,
    then $\rate(\l) = \rateb(\l)$ for all $\l \in \left[0, \Lru\right]$.
    Similarly, if $\lt$ exists, $\rate(\l) = \rateb(\l)$ for all $\l \geq \lt$,
    in particular for $\overl \geq \lt \in \Lrmin$.

    \cref{item:early} requires additional considerations.
    In the case where $\lt$ does not exist,
    we have $\Lrmin = \{0\}$, and the result follows directly.
    Otherwise, using \cref{eq:rate_decomp},
    we have $\rate(\l) = \ratev(\l)$ for any $\l \leq \underl$,
    as $\underl \leq \min({\Lrmin}) \leq \lt$.
    As shown in the proof of \cref{item:monotonicity},
    $\ratev(\l)$ alternates between derivatives $0$ and $-2/\Ln$.
    We claim that there exists a left neighborhood of $\min{(\Lrmin)}$
    where the derivative is $-2/\Ln$.
    Assume towards a contradiction that no such left neighborhood exists.
    Then, there must be a left neighborhood of $\min{(\Lrmin)}$ where the derivative is $0$,
    since $\ratev(\l)$ alternates between only two derivatives.
    However, in that left neighborhood, $\ratev(\l)$ is constant
    with all values equal to $\ratemin$.
    Hence, there exists $\l < \min{(\Lrmin)}$ with
    $\rate(\l) = \ratev(\l) = \ratemin$
    and thus $\l \in \Lrmin$,
    which is a contradiction.

    Thus, there exists a left neighborhood of $\min{(\Lrmin)}$
    with diameter $\varepsilon > 0$
    throughout which the derivative is $-\frac{2}{\Ln}$.
    Then, for all
    $\l \in [\min{(\Lrmin)}-\varepsilon, \min{(\Lrmin)}]$,
    \begin{equation*}
        \ratev(\l) = \ratemin - \frac{2}{\Ln}(\l - \min{(\Lrmin)}).
    \end{equation*}
    Finally,
    since $\ratev(\l)$ is non-increasing,
    as long as $\ratev(\underl) \leq \ratemin + \frac{2}{\Ln}\varepsilon$
    we have $\underl \geq \min{(\Lrmin)}-\varepsilon$.
    Hence, choosing
    $\c = \frac{2}{\Ln}\varepsilon$
    yields the statement of \cref{item:early}.
\end{proof}

\subsection{Proof of \texorpdfstring{\cref{th:trainerror}}{Theorem~\ref{th:trainerror}}}
\label{ssec:training_error_proof}
The informal \cref{th:trainerror} in the main text relies on a $\betath$,
defined as the intersection of the variance and bias rates from \cref{th:rates}
for the interpolator $\hat f_0$ (setting $\Lr =0$).
Whenever $\betath$ is unique,
the fact that $\bias^2(\hat f_0)$ in \cref{th:rates} strictly increases
as a function of $\beta$ induces a phase transition:
for $\beta > \betath$, the bias dominates the rate of the risk in \cref{th:rates},
while for $\beta \leq \betath$, the variance dominates.
In particular, \cref{lem:opt_reg,lem:eta_prop} imply that interpolation is harmless
whenever the bias dominates, and harmful if the variance dominates.

Intuitively, \cref{th:trainerror} considers optimally regularized estimators,
and varies the inductive bias strength via $\beta$.
The formal \cref{th:trainerror_formal} below
presents a different perspective:
it considers $\beta$ fixed,
and instead differentiates whether the optimal risk rate in \cref{th:rates}
results from $\Lr > 0$ (harmful interpolation)
or $\Lr = 0$ (harmless interpolation).

Whenever $\betath$ is well-defined and unique,
as for example in the setting of \cref{fig:theory},
the two perspectives coincide.
However, one can construct pathological edge cases
where variance and bias rates intersect on an interval of $\beta$ values,
or the dominating quantity of the risk rate in \cref{th:rates} as a function of $\beta$
alternates between the variance and bias.
While \cref{th:trainerror} fails to capture such settings,
\cref{th:trainerror_formal} still applies.
We hence present \cref{th:trainerror_formal} as a more general result.

\begin{theorem}[Formal version of \cref{th:trainerror}]
\label{th:trainerror_formal}
In the setting of \cref{th:rates}
using the notation from \cref{app:opt_reg},
let $\Ln > 0$, $\beta \in (0, 1)$, $\Ls \geq -\Lru$,
and $\Lgt \in \left[1,\ceil{\frac{\Ln-1}{\beta}}\right]\cap \N$.
Let further
$\regopt = \argmin_{\lambda \geq 0 \mid \max{\{\lambda, 1\}} \in \bigO(d^{\Lru})} \risk(\freg)$.
Then, the expected training error behaves as follows:
\begin{enumerate}
    \item\label{item:harfull_int}
    If all $\l \in \argmin_{\l \in [0,\Lru]}{\rate(\l; \Ln, \Ls,\Lgt, \beta)}$
    satisfy $\rate(\l; \Ln, \Ls,\Lgt, \beta)<\rate(0; \Ln, \Ls,\Lgt, \beta)$, then
    \begin{equation*}
        \abs*{
            \E_{\epsilon} \left[\frac{1}{n} \sum_i{(\fopt(x_i) - y_i)^2}\right] - \sigma^2
        } \in \bigO(d^{-\l} + \risk^2(\fopt)),
        \quad \text{w.p.\ } \geq 1-\c d^{-\beta\min\{\bd, 1-\bd\}}.
    \end{equation*}
    \item\label{item:harmless_int}
    If all $\l \in (0,\Lru]$
    satisfy $\rate(\l; \Ln, \Ls,\Lgt, \beta) > \rate(0; \Ln, \Ls,\Lgt, \beta)$, then
    \begin{equation*}
        \E_{\epsilon}\left[\frac1n\sum_i (\fopt(x_i)- y_i)^2\right]
        \leq \tilde{\c} \sigma^2 + \bigO\left(\risk^2(\fopt)\right),
        \quad \text{w.p.\ } \geq 1-\c d^{-\beta\min\{\bd, 1-\bd\}}
    \end{equation*}
    for some constant $\tilde{\c} < 1$.
\end{enumerate}
\end{theorem}

Intuitively, the two cases in \cref{th:trainerror_formal}
correspond to harmful and harmless interpolation,
where the optimal rate in \cref{th:rates}
is for $\Lr > 0$ and $\Lr = 0$, respectively.
Then, \cref{lem:opt_reg} yields with high probability
that also $\Lropt > 0$ and $\Lropt = 0$
in the first and second case, respectively.
Finally, we remark that \cref{th:trainerror_formal} lacks an edge case:
if both some $\Lr > 0$ and $\Lr = 0$ minimize the risk rate in \cref{th:rates} simultaneously,
\cref{lem:opt_reg} fails to differentiate whether $\Lropt$ is zero or positive.
However, that edge case corresponds to either interpolation or very weak regularization.
Hence, we conjecture the corresponding model's training error
to behave similar to the second case in \cref{th:trainerror_formal}.

\begin{proof}[Proof of \cref{th:trainerror_formal}]
    First, \cref{lem:opt_reg} yields
    with probability at least $1-\c d^{-\beta\min\{\bd,1-\bd\}}$ that
    $\max\{\regopt,1\}\in \Theta(d^{\l})$,
    where $\l$ is a minimizer of $\rate(\l; \Ln, \Ls,\Lgt, \beta)$.
    The condition in \cref{item:harfull_int} guarantees that all optimal $\l$
    are positive,
    while the condition in \cref{item:harmless_int} ensures that the optimal $\l = 0$.

    \paragraph*{Harmful interpolation setting (\cref{item:harfull_int})}
    In this case, $\l > 0$,
    and thus $\regopt \in \Theta(d^\l)$ for $d$ sufficiently large.
    We start by applying \cref{th:rates} with $\Lr = \l$ in this setting.

    Within the proof of \cref{th:rates} in \cref{subsec:proof_s2},
    we pick a $\m \in \N$
    such that for $d$ sufficiently large,
    with probability at least $1-\c d^{-\beta\min\{\bd,1-\bd\}}$,
    $\m$ satisfies the conditions of
    \crefrange{lem:as_1_verification}{lem:eigendecay_main}
    and \cref{th:vidya}.
    For the remainder of this proof,
    let $\m$
    be the same as in the proof of \cref{th:rates}
    for $\Lr = \l$.
    This $\m$ satisfies the conditions of \cref{prop:trainerror_context_agnostic},
    which hence yields
    \begin{equation*}
        \frac{\regopt^2 \sigma^2}n \Tr\left(\H^{-2}\right)
        \leq \E_{\epsilon} \left[\frac1n\sum_i (\fopt(x_i) - y_i)^2\right]
        \leq \frac{\regopt^2 \sigma^2}{n}
            \Tr\left(\H^{-2}\right)
            + 6\regopt^2 \frac{\rmax^2}{\rmin^2}\frac{\norm{{\Dm^{-1}\a}}^2}{n^2},
    \end{equation*}
    where $\H \defeq \K+\reg \I_n$.
    Then, using that
    $a \leq b \leq a + d$
    implies $\abs{b - c} \leq \abs{a - c} + d$,
    we further get
    \begin{equation}
        \abs*{\E_{\epsilon} \left[\frac1n\sum_i (\fopt(x_i) - y_i)^2\right] - \sigma^2}
        \leq \underbrace{
            \sigma^2\abs*{\frac{\regopt^2 }n \Tr\left(\H^{-2}\right) - 1}
        }_{\defeq T_1}
        + \underbrace{
            6 \regopt^2\frac{\rmax^2}{\rmin^2} \frac{\norm{\Dm^{-1}\a}^2}{n^2}
        }_{\defeq T_2}.
        \label{eq:harmfull_int_bound}
    \end{equation}
    We first bound $T_2$ as follows:
    \begin{align}
        T_2 = 6 \regopt^2\frac{\rmax^2}{\rmin^2} \frac{\norm{\Dm^{-1}\a}^2}{n^2}
        &\overset{(i)}{\lesssim}
            \frac{d^{2\l}}{d^{2\Ln}} \frac{\rmax^2}{\rmin^2}\norm{\Dm^{-1}\a}^2
        \nonumber\\
        &\overset{(ii)}{\lesssim}
            \frac{d^{2\l}}{d^{2\Ln}} \cdot d^2 \filsize^{2(\Lgt-1)}
        \nonumber \\
        &=
            \frac{d^{2\l}}{d^{2}d^{2\l}d^{2\left(\Ln-\l-1\right)}}
            \cdot d^2 \filsize^{2(\Lgt-1)}
        \nonumber \\
        &\lesssim d^{-2\left(\Ln - \l - 1 - \beta(\Lgt-1)\right)} \nonumber \\
        &\overset{(iii)}{\in} \bigO\left(\bias^2(\fopt)\right)
        \subseteq \bigO\left(\risk^2(\fopt)\right),\label{eq:t_2_bound}
    \end{align}
    where $(i)$ uses $n\in \Theta(d^{\Ln})$ and $\regopt \in \Theta(d^{\l})$,
    $(ii)$ applies \cref{lem:ground_truth_rate} for the rate of $\norm{\Dm^{-1}\a}$
    and \cref{lem:combined} for $\rmax^2,\rmin^2 \in \Theta(1)$,
    and $(iii)$ matches the expression to the rate of the bias in \cref{th:rates}.

    For $T_1$, we will bound $\Tr\left(\H^{-2}\right)$ from above and below
    using \cref{lem:tr_inv}.
    We hence introduce the following notation:
    \begin{equation*}
        \K_{-2} \defeq  \Km + \K_1, \quad \H_2 \defeq \K_2 + \regopt\I_n,
    \end{equation*}
    so that $\H = \H_2 + \K_{-2}$,
    and where $\K_1$ and $\K_2$ are defined in \cref{ssec:12decomposition}.
    Furthermore, the rank of $\K_{-2}$ is at most
    $\abs{\N\setminus\IndS_2} = \abs{\{\i \in \N \mid \abs{S_\i} < \bL + 2 \}}$,
    that is, the number of eigenfunctions that contribute to $\K_{-2}$.
    The rate of $\abs{\N\setminus\IndS_2}$ is
    \begin{equation*}
        \abs{\N\setminus\IndS_2}
        \overset{(i)}{=} \sum_{\ii = 0}^{\bL + 1}{\C(\ii, \filsize, d)}
        \overset{(ii)}{=} 1 + \sum_{\ii = 1}^{\bL + 1}{d\binom{\filsize-1}{\ii-1}}
        \overset{(iii)}{\in} \bigO(d \cdot \filsize^{\bL}),
    \end{equation*}
    where $(i)$ uses the definition of $\C(\ii, \filsize, d)$
    in \cref{eq:C_def},
    $(ii)$ applies \cref{lem:n_sets} with $d$ sufficiently large,
    and $(iii)$ uses the classical bound
    $\binom{\filsize-1}{\ii-1} \leq\left(e\frac{\filsize-1}{\ii-1}\right)^{\ii-1}$
    as well as the fact that the largest monomial dominates the sum.
    Finally, since $n\in \Theta(d\cdot \filsize^{\bL +\bd})$,
    we have
    \begin{equation}\label{eq:rateK_2}
        \rank(\K_{-2})
        \leq \abs{\N\setminus\IndS_2}
        \in \bigO(n \filsize^{-\bd})
        \subseteq o(n).
    \end{equation}

    Therefore, for $d$ and hence $n \in \Theta(d^\Ln)$ sufficiently large,
    $\rank(\K_{-2}) < n$,
    and \cref{lem:K2} yields
    $c_1 \leq \mineig{\H_2} \leq \norm{\H_2} \leq c_2$ for some constants $c_1, c_2 > 0$
    with probability at least $1-\c' {\filsize^{-(1-\bd)}}$.
    We can thus instantiate \cref{lem:tr_inv} with $\Mm=\H, \Mm_1 = \K_{-2}, \Mm_2 = \H_2$.
    This implies that:
    \begin{equation*}
        \Tr\left(\H^{-2} \right)
        \geq \frac{n-\rank(\K_{-2})}{\norm{\H_2}^2}
        \geq \frac{n-\rank(\K_{-2})}{(c_2+\regopt)^2}.
    \end{equation*}
    The upper bound on $\Tr\left(\H^{-2}\right)$ simply follows from
    \begin{align}
        \Tr\left(\H^{-2} \right)
        &\leq n \norm{\H^{-2}}\notag\\
        &= \frac{n}{\mineig{\Km+\K_1+\K_2+\regopt \I_n}^2}\notag\\
        &\leq \frac{n}{\mineig{\K_2+\regopt \I_n}^2}\notag\\
        &\overset{(i)} \leq \frac{n}{(c_1 + \regopt)^2},\label{eq:up_bound_tr}
    \end{align}
    where $(i)$ uses the previous lower bound on $\mineig{\K_2}$ from \cref{lem:K2}.

    Combining the upper and lower bounds on $\Tr\left(\H^{-2}\right)$ yields
\begin{align*}
    \frac{n-\rank(\K_{-2})}{(c_2+\regopt)^2}
        &&&\leq \Tr\left(\H^{-2} \right)
        &&\leq \frac{n}{(c_1+\regopt)^2}\\
    \frac{n-\rank(\K_{-2})}{n}\frac{\regopt^2}{(c_2+\regopt)^2}
        &&&\leq \frac{\regopt^2}{n}\Tr\left(\H^{-2} \right)
        &&\leq \frac{\regopt^2}{(c_1+\regopt)^2}\\
    -\frac{\rank(\K_{-2})}{n}\frac{\regopt^2}{(c_2+\regopt)^2}
    +\frac{\regopt^2}{(c_2+\regopt)^2}-1
        &&&\leq \frac{\regopt^2}{n}\Tr\left(\H^{-2} \right) -1
        &&\leq \frac{\regopt^2}{(c_1+\regopt)^2}-1\\
    -\frac{\rank(\K_{-2})}{n}\frac{\regopt^2}{(c_2+\regopt)^2}
    -\frac{2c_2\regopt +c_2^2}{(c_2+\regopt)^2}
        &&&\leq \frac{\regopt^2}{n}\Tr\left(\H^{-2} \right) -1
        &&\leq -\frac{2c_1\regopt +c_1^2}{(c_1+\regopt)^2}.
\end{align*}
Taking absolute values,
a simple case distinction and $\c_1 < \c_2$ yields the following bound:
\begin{align*}
    \abs*{\frac{\regopt^2}{n}\Tr\left(\H^{-2} \right) -1}
    &\leq \frac{2\c_2+\regopt}{(\c_2+\regopt)^2}
    + \frac{\rank(\K_{-2})}{n} \frac{\regopt^2}{(\c_2+\regopt)^2}\\
    &\overset{(i)}\lesssim \frac{d^{\l}}{d^{2\l}}
    + \frac{n \filsize^{-\bd}}{n} \frac{d^{2\l}}{d^{2\l}}
    = d^{-\l} + \filsize^{-\bd}
    \in \bigO(d^{-\l}),
\end{align*}
where $(i)$ uses the rate of $\rank(\K_{-2})$ from \cref{eq:rateK_2}
and $\regopt \in \Theta\left(d^\l\right)$.
Combining this bound on $T_1$ with the bound on $T_2$ in \cref{eq:t_2_bound}
and collecting all error probabilities
concludes the current case.

\paragraph*{Harmless interpolation setting (\cref{item:harmless_int})}
In this setting,
the only minimizer of the risk rate is $\l = 0$,
and hence
\begin{equation*}
    \regopt \leq \max{\{\regopt, 1\}} \in \bigO(d^0) = \bigO(1).
\end{equation*}
As for the previous case,
let $\m$
be the same as in the proof of \cref{th:rates}
for $\Lr = 0$.
With probability at least $1-\c d^{-\beta\min\{\bd,1-\bd\}}$,
this $\m$ again satisfies the conditions of \cref{prop:trainerror_context_agnostic},
which yields
\begin{equation*}
    \E_{\epsilon} \left[\frac1n\sum_i (\fopt(x_i) - y_i)^2\right]
    \leq \underbrace{\frac{\regopt^2 \sigma^2}n \Tr\left(\H^{-2}\right)}_{\defeq T_3}
    + \underbrace{6\regopt^2 \frac{\rmax^2}{\rmin^2}\frac{\norm{{\Dm^{-1}\a}}^2}{n^2}}_{T_2}.
\end{equation*}
Furthermore, we apply the same steps with $\l = 0$ as in the previous case (\cref{eq:t_2_bound})
to bound $T_2$:
\begin{equation*}
    T_2 \in \bigO\left(\risk^2(\fopt)\right).
\end{equation*}
For $T_3$,
we use the same bound on $\Tr\left(\H^{-2} \right)$ as in \cref{eq:up_bound_tr}
with the same probability as follows:
\begin{align*}
    T_3 = \frac{\regopt^2\sigma^2}{n} \Tr\left(\H^{-2} \right)
    &\leq \frac{\regopt^2\sigma^2}{n} \frac{n}{(\c_1 + \regopt)^2} \\
    &= \sigma^2\frac{\regopt^2}{(\c_1+\regopt)^2} \\
    &\overset{(i)}\leq \sigma^2\frac{(\c'')^2}{(\c_1+\c'')^2},
\end{align*}
where $(i)$ follows for $d$ sufficiently large from $\regopt \in \bigO(1)$ with $\c'' > 0$.
Since $c_1 > 0$, we have $\frac{(\c'')^2}{(\c_1+\c'')^2} < 1$.
Hence, combining the bounds on $T_2$ and $T_3$,
as well as collecting all probabilities,
we get the desired result for this case.
\end{proof}

\subsection{Technical lemmas}
\begin{lemma}[Trace of the inverse]
    \label{lem:tr_inv}
    Let $\Mm, \Mm_1,\Mm_2 \in \R^{n\times n}$ be symmetric positive semi-definite matrices
    with $\Mm = \Mm_1 + \Mm_2$.
    Furthermore, assume that $\mineig{\Mm_2}>0$ and $\rank(\Mm_1) < n$.
    Then,
    \begin{equation*}
        \Tr(\Mm^{-2}) \geq \frac{n-\rank(\Mm_1)}{\norm{\Mm_2}^2}.
    \end{equation*}
\end{lemma}
\begin{proof}
    First, we apply the identity
    \begin{equation*}
        \Mm^{-1} = (\Mm_2 + \Mm_1)^{-1} = \Mm_2^{-1}
        - \underbrace{\Mm_2^{-1}\Mm_1(\Mm_2 + \Mm_1)^{-1}}_{\defeq \Am},
    \end{equation*}
    which holds since $\Mm_2$, and thus $\Mm$, are full rank.
    Next, $\Am$ is a product of matrices including $\Mm_1$;
    hence, the rank of $\Am$ is bounded by $\rank(\Mm_1) < n$.
    Let now $\{\v_1,\dots,\v_{\rank(\Am)}\}$ be an orthonormal basis
    of $\col(\Am)$,
    and let $\{\v_{\rank(\Am)+1},\dots,\v_{n}\}$ be an orthonormal basis
    of $\col(\Am)^\perp$.
    Thus, $\{\v_1,\dots,\v_n\}$ is an orthonormal basis of $\R^n$,
    and similarity invariance of the trace yields
    \begin{align*}
        \Tr(\Mm^{-2})
        &= \sum_{i = 1}^n{\v_i^\t\Mm^{-2}\v_i}
        = \sum_{i = 1}^{\rank(\Am)}{\v_i^\t\Mm^{-2}\v_i}
            + \sum_{i = \rank(\Am)+1}^n{\v_i^\t\Mm^{-2}\v_i} \\
        &\geq \sum_{i = \rank(\Am)+1}^n{\v_i^\t\left({\Mm_2}^{-1}-\Am\right)^2\v_i} \\
        & \overset{(i)}{=} \sum_{i = \rank(\Am)+1}^n{
            \left(
                \v_i^\t {\Mm_2}^{-2}\v_i
                - \underbrace{ \v_i^\t \Am}_{=0}{\Mm_2}^{-1}\v_i
                - \v_i^\t {\Mm_2}^{-1}\underbrace{\Am \v_i}_{=0}
                + \underbrace{\v_i^\t\Am \Am\v_i}_{=0}
            \right)
        }\\
        &= \sum_{i = \rank(\Am)+1}^n{\v_i^\t {\Mm_2}^{-2}\v_i} \\
        &\overset{(ii)}\geq (n-\rank(\Am))\mineig{{\Mm_2}^{-2}}
        =(n-\rank(\Am))\frac{1}{\norm{{\Mm_2}}^2} \\
        &\overset{(iii)}{\geq} \frac{n-\rank(\Mm_1)}{\norm{\Mm_2}^2},
    \end{align*}
    where $(i)$ uses that, for all $i>\rank(\Am)$,
    $\v_\iii$ is orthogonal to the column space of $\Am$,
    and $\Am$ is symmetric.
    Furthermore, $(ii)$ uses that all $\v_i$ have norm $1$,
    and $(iii)$ that $\rank(\Am)\leq \rank(\Mm_1)$.
\end{proof}

\begin{lemma}[Fixed-design training error]
    \label{prop:trainerror_context_agnostic}
    In the setting of \cref{th:vidya},
    let $\m \in \mathbb{N}$ such that
    $\rmin > 0$,
    \cref{eq:ass1} holds,
    and the ground truth satisfies $\fstar(x) = \sum_{\i = 1}^\m \a_\i \Pf_\i(x)$.
    Then,
    \begin{equation*}
        \frac{\reg^2 \sigma^2}n \Tr\left(\H^{-2}\right)
        \leq \E_{\epsilon} \left[
            \frac{1}{n}\sum_i{(\hat{f_\reg}(x_i) - y_i)^2}
        \right]
        \leq \frac{\reg^2 \sigma^2}n \Tr\left(\H^{-2}\right)
            + 6\reg^2 \frac{\rmax^2}{\rmin^2}\frac{\norm{{\Dm^{-1}\a}}^2}{n^2},
    \end{equation*}
    where $\H \defeq \K+\reg \I_n$.
\end{lemma}
\begin{proof}
    The kernel ridge regression estimator is
    \begin{equation*}
        \freg(x^*) = \y^\t \H^{-1}\k(x^*),
    \end{equation*}
    where $\y \defeq [y_1,\dotsc, y_n]^\t$
    with $y_\iii = \fstar(x_\iii) + \epsilon_\iii$,
    and $\k(x^*) \defeq [\Kf(x_1,x^*),\dots,\Kf(x_n,x^*)]^\t$.
    Thus, the estimator evaluated at $x_1,\dots,x_n$
    is $\K\H^{-1}\y$, which yields the following training error:
    \begin{equation*}
        \frac{1}{n} \sum_i{(\freg(x_i) - y_i)^2}
        = \frac{1}{n}\norm{\left(\I_n-\K\H^{-1}\right)\y}^2
        \overset{(i)}{=} \frac{\reg^2}{n}\norm{\H^{-1}\y}^2,
    \end{equation*}
    where $(i)$ follows from
    \begin{equation*}
        \I_n-\K\H^{-1}
        = \I_n-\K(\K+\reg\I_n)^{-1}
        = \I_n-(\K+\reg\I_n-\reg\I_n)(\K+\reg\I_n)^{-1}
        = \reg(\K+\reg\I_n)^{-1}.
    \end{equation*}

    Next, by the assumptions on the ground truth,
    we can write $\y = \e + \Pm\a$.
    Thus, the expected training error with respect to the noise is
    \begin{align*}
        \E_{\epsilon}\left[\frac1n\sum_i (\freg(x_i) - y_i)^2\right]
        &=\frac{\reg^2}{n}\E_{\epsilon} \left[(\e + \Pm\a)^\t\H^{-2}(\e + \Pm\a)\right] \\
        &= \underbrace{\frac{\reg^2 \sigma^2}n \Tr\left(\H^{-2}\right)}_{\defeq T_1}
        + \underbrace{\frac{\reg^2 }{n}\a^\t\Pm^\t\H^{-2}\Pm\a}_{\defeq T_2}.
    \end{align*}

    First, $T_2 > 0$ since $\H$ is positive semi-definite.
    Therefore, $T_1$ already yields the desired lower bound on the expected training error.
    For the upper bound, we bound $T_2$ as follows:
    \begin{align*}
        T_2 &= \frac{\reg^2 }{n}( \Dm^{-1}\a)^\t \Dm\Pm^\t\H^{-2}\Pm\Dm (\Dm^{-1}\a)\\
        &\overset{(i)}{=} \frac{\reg^2}{n} (\Dm^{-1}\a)^\t
            \left(\Dm^{-1} + \Pm^\t\HM^{-1}\Pm\right)^{-1}
            \Pm^\t\HM^{-2}\Pm \\
        &\qquad\qquad\left(\Dm^{-1} + \Pm^\t\HM^{-1}\Pm\right)^{-1}(\Dm^{-1}\a)\\
        &= \frac{\reg^2}{n}
            \norm*{\HM^{-1}\Pm\left(\Dm^{-1} + \Pm^\t\HM^{-1}\Pm\right)^{-1}\Dm^{-1}\a}^2 \\
        &\leq \reg^2 \frac{\norm*{\Pm^\t\HM^{-2}\Pm}}n\norm*{\left({\Dm^{-1}} + \Pm^\t\HM^{-1}\Pm\right)^{-1}\Dm^{-1}\a}^2\\
        &\overset{(ii)}{\leq} \frac{1.5\reg^2 }{(\mineig{\KM}+\reg)^2} B_1.\\
        &\overset{(iii)}{\leq} 6\reg^2  \frac{\rmax^2\max\{\reg,1\}^2}{\left(\mineig{\KM}+\reg\right)^2}\frac{\norm{{\Dm^{-1}\a}}^2}{n^2} = 6\reg^2 \frac{\rmax^2}{\rmin^2}\frac{\norm{{\Dm^{-1}\a}}^2}{n^2},
    \end{align*}
    where $\HM \defeq \KM + \reg\I_n$.
    Step $(i)$ follows from \cref{lem:20_from_bartlett},
    step $(ii)$ uses \cref{eq:ass1} and matches the term $B_1$ from the proof of \cref{th:vidya},
    and step $(iii)$ applies the bound on $B_1$ achieved in \cref{th:vidya}.
    This upper-bounds the expected training error
    and thereby concludes the proof.
\end{proof}

%% file: appendix/experiment_details.tex
\section{Experimental details}
\label{sec:experimentdetails}

This section describes our experimental setup
and includes additional details.
We provide the code to replicate all experiments and plots
in \url{https://github.com/michaelaerni/iclr23-InductiveBiasesHarmlessInterpolation}.

\subsection{Setup for filter size experiments}
\label{ssec:experimentdetails_setup_filters}

The following describes the main filter size experiments
presented in \cref{ssec:empirical_filters}.

\paragraph*{Network architecture}
We use a special CNN architecture that amplifies
the role of filter size as an inductive bias.
Each model of the main filter size experiments
in \cref{fig:filters_error}
has the following architecture:
\begin{enumerate}
    \item Convolutional layer with $128$ filters of varying size
    and no padding
    \item Global max pooling over the spatial feature dimensions
    \item ReLU activation
    \item Linear layer with $256$ output features
    \item ReLU activation
    \item Linear layer with $1$ output feature
\end{enumerate}
All convolutional and linear layers use a bias term.
Since we employ a single convolutional layer before global max pooling,
the convolutional filter size directly determines the maximum size
of an input patch that can influence the CNN's output.
Note that this architecture reduces to an MLP
if the filter size equals the input image size.

\paragraph*{Optimization procedure}
We use the same training procedure for all settings in \cref{fig:filters_error}.
Optimization minimizes the logistic loss for $300$ epochs
of mini-batch SGD with momentum $0.9$ and batch size $100$.
We linearly increase the learning rate from $10^{-6}$
to a peak value of $0.2$
during the first $50$ epochs,
and then reduce the learning rate according to an inverse square-root decay
every $20$ epochs.
For peak learning rate $\gamma_0$,
a decay rate $L$, the inverse square-root decay schedule at epoch $t \geq 0$ is
\begin{equation}
    \frac{\gamma_0}{\sqrt{1 + \floor{t / L}}} .
    \label{eq:experimentdetails_sqrtlr}
\end{equation}
Learning rate warm-up helps to capture the early-stopped test error
more precisely.
Whenever possible, we use deterministic training algorithms,
so that our results are as reproducible as possible.
We selected all hyperparameters to minimize
the training loss of the strongest inductive bias (filter size $5$)
on noisy training data,
with the constraint that all other settings still converge and interpolate.
Note that we do not use data augmentation, dropout, or weight decay.

\paragraph*{Evaluation}
We observed that all models achieved their minimum test error either
at the beginning or very end of training.
Hence, our experiments evaluate the test error every $2$ epochs during the first $150$ epochs,
and every $10$ epochs afterwards to save computation time.
We use an oracle, that is, the true test error,
to determine the optimal early stopping epoch in retrospective.
The optimal early stopping training error is always over the entire training set
(including potential noise) for a fixed model,
not an average over mini-batches.
To mitigate randomness in both the training data and optimization procedure,
we average over multiple dataset and training seeds.
More precisely, we sample $5$ different pairs of training and test datasets.
For each dataset, we fit $15$ randomly initialized models
per filter size on the same dataset,
and calculate average metrics.
The plots then display the mean and standard error over
the $5$ datasets.

\paragraph*{Dataset}
All filter size experiments use synthetic images.
For a fixed seed, the experiments generate $200$ training
and $100$k test images,
both having an equal amount of positive and negative classes.
Given a class, the sampling procedure iteratively
scatters $10$ shapes on a black $32 \times 32$ image.
A single shape is either a circle (negative class) or a cross (positive class),
has a uniformly random size in $[3, 5]$,
and a uniformly random center such that all shapes
end up completely inside the target image.
We use a conceptual line width of $0.5$ pixels,
but discretize the shapes into a grid.
See \cref{fig:filters_data} for examples.
A single dataset seed fully determines the training data,
test data, and all scattered shapes.

\begin{figure}[tb]
    \centering
    \begin{subfigure}[b]{\figsixcol}
        \centering
        \includegraphics[width=\linewidth]{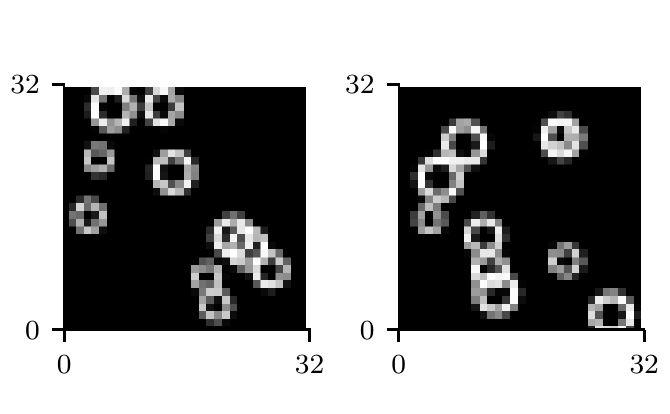}
        \caption{Negative samples}
    \end{subfigure}
    \hfill%
    \begin{subfigure}[b]{\figsixcol}
        \centering
        \includegraphics[width=\linewidth]{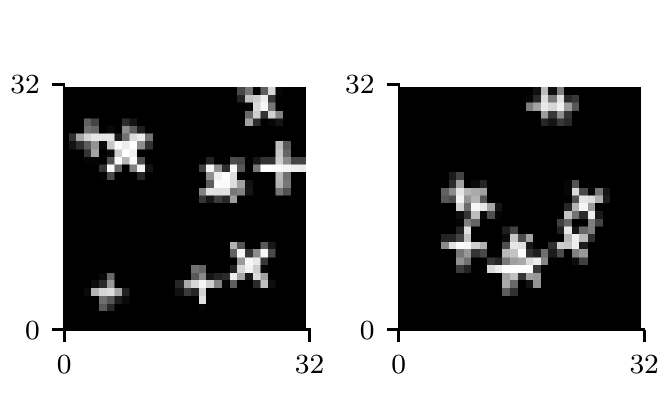}
        \caption{Positive samples}
    \end{subfigure}
    \caption{
        Example synthetic images used in the filter size experiments.
    }
    \label{fig:filters_data}
\end{figure}

\paragraph*{Noise model}
In the noisy case, we select $20\%$ of all training samples
uniformly at random without replacement,
and flip their label.
The noise is deterministic per dataset seed
and does not change between different optimization runs.
Note that we never apply noise to the test data.

\subsection{Setup for rotational invariance experiments}
\label{ssec:experimentdetails_setup_rotations}

The following describes the rotational invariance experiments
presented in \cref{ssec:empirical_rotations}.

\paragraph*{Dataset}
We use the EuroSAT \citep{Helber18} training split
and subsample it into $7680$ raw training and $10$k raw test samples in a stratified way.
For a fixed number of rotations $k$, we generate a training dataset as follows:
\begin{enumerate}
    \item In the noisy case, select a random $20\%$ subset of all training samples without replacement;
    for each, change the label
    to one of the other $9$ classes uniformly at random.
    \item For the $i$-th training sample ($i \in \{1, \dotsc, 7680\}$):
    \begin{enumerate}
        \item Determine a random offset angle $\alpha_i$.
        \item Rotate the original image by each angle in
        $\{\alpha_i + j \cdot (360^\circ / k) \mid j = 0, \dotsc, k - 1 \}$.
        \item Crop each of the $k$ rotated $64 \times 64$ images to $44 \times 44$
        so that no network sees black borders from image interpolation.
    \end{enumerate}
    \item Concatenate all $k \times 7680$ samples into a single training dataset.
    \item Shuffle this final dataset (at the beginning of training and every epoch).
\end{enumerate}
To generate the actual test dataset,
we apply a random rotation to each raw test sample independently,
and crop the rotated images to the same size as the training samples.
This procedure rotates every image exactly once,
and uses random angle offsets to avoid distribution shift effects from image interpolation.
Note that all random rotations are independent of the label noise and
the number of training rotations.
Hence, all experiments share the same test dataset.
Furthermore, since we apply label noise before rotating images,
all rotations of an image consistently share the same label.

\paragraph*{Network architecture}
All experiments use a \wrn{} \citep{Zagoruyko16}
with $16$ layers, widen factor $6$,
and default PyTorch weight initialization.
We chose the width and depth such that all networks
are sufficiently overparameterized
while still being manageable in terms of computational cost.

\paragraph*{Optimization procedure}
We use the same training procedure for all settings in \cref{fig:rotations}.
Optimization minimizes the softmax cross-entropy loss
using mini-batch SGD with momentum $0.9$ and batch size $128$.
Since the training set size grows in the number of rotations,
all experiments fix the number of gradient updates to $144$k.
This corresponds to $200$ epochs over a dataset with $12$ rotations.
Similar to the filter size experiments,
we linearly increase the learning rate from zero
to a peak value of $0.15$ during the first $4800$ steps,
and then reduce the learning rate according to an inverse square-root decay
(\cref{eq:experimentdetails_sqrtlr}) every $960$ steps.
Whenever possible, we use deterministic training algorithms,
so that our results are as reproducible as possible.
We selected all hyperparameters to minimize
the training loss of the strongest inductive bias ($12$ rotations)
on noisy training data,
with the constraint that all other settings still converge and interpolate.
As for all experiments in this paper,
we do not use additional data augmentation, dropout, or weight decay.

\paragraph*{Evaluation}
Similar to the filter size experiments,
we evaluate the test error more frequently during early training iterations:
every $480$ steps for the first $9600$ steps,
every $1920$ steps afterwards.
The experiments again use the actual test error
to determine the best step for early-stopping,
and calculate the corresponding training error
over the entire training dataset,
including all rotations and potential noise.
Due to the larger training set size and increased computational costs,
we only sample a single training and test dataset,
and report the mean and standard error of all metrics over five training seeds.

\subsection{Difference to \ddd{}}
\label{ssec:experimentdetails_ddd}

As mentioned in \cref{ssec:empirical_filters},
our empirical observations resemble the \ddd{} phenomenon.
This subsection expands on the discussion and provides
additional details about how this paper's phenomenon differs from \ddd{}.

While all models in all experiments interpolate the training data,
we observe that both noisy labels and stronger inductive biases
increase the final training loss of an interpolating model:
Smaller filter size results in a decreasing number of model parameters.
Enforcing invariance to more rotations requires a model to
interpolate more (correlated) training samples.
Thus, in both cases,
increasing inductive bias strength decreases a model's overparameterization
in relation to the number of training samples
--- shifting the setting closer to the interpolation threshold.

We argue that our choice of architecture and hyperparameter tuning
ensures that no model in any experiment is close to the corresponding interpolation threshold.
If that is the case, then \ddd{} predicts that increasing the number of model parameters
has a negligible effect on whether regularization benefits generalization,
and does therefore not explain our observations.

In the following, we first describe how our hyperparameter and model selection procedure
ensures that all models in all experiments are sufficiently overparameterized,
so that \ddd{} predicts negligible effects from increasing the number of parameters.
Then, we provide additional experimental evidence that supports our argument:
We repeat a subset of the experiments in \cref{sec:empirical}
while upscaling the number of parameters in all models.
For a fixed model scale and varying inductive bias,
we observe that all phenomena in \cref{sec:empirical} persist.
For a fixed inductive bias strength,
we further see that the test error of interpolating models saturates
at a value that matches our hypothesis.
In particular, for strong inductive biases,
the gap in test error between interpolating models
and their optimally early-stopped version
--- harmful interpolation --- persists.

\paragraph*{Hyperparameter tuning}
We mitigate differences in model complexity for different inductive bias strengths
by tuning all hyperparameters on worst-case settings,
that is, maximum inductive bias with noisy training samples.
To avoid optimizing on test data,
we tune on dataset seeds and network initializations
that differ from the ones used in actual experiments.
\Cref{fig:loss} displays the final training loss for all
empirical settings in this paper.
While models with a stronger inductive bias exhibit larger training losses,
all values are close to zero,
and the numerical difference is small.
Finally, we want to stress again that this discussion is only about the training \emph{loss};
all models in all experiments have zero training \emph{error}
and perfectly fit the corresponding training data.

\begin{figure}[t]
    \centering
    \begin{subfigure}[l]{\figsixcol}
        \centering
        \begin{subfigure}[t]{\linewidth}
            \centering
            \includegraphics[width=\linewidth]{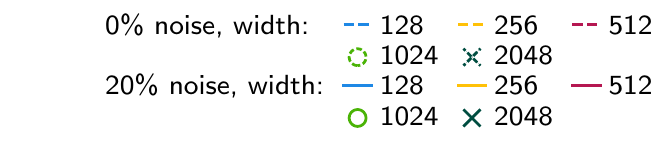}
        \end{subfigure}
        \begin{subfigure}[b]{\linewidth}
            \centering
            \includegraphics[width=\linewidth]{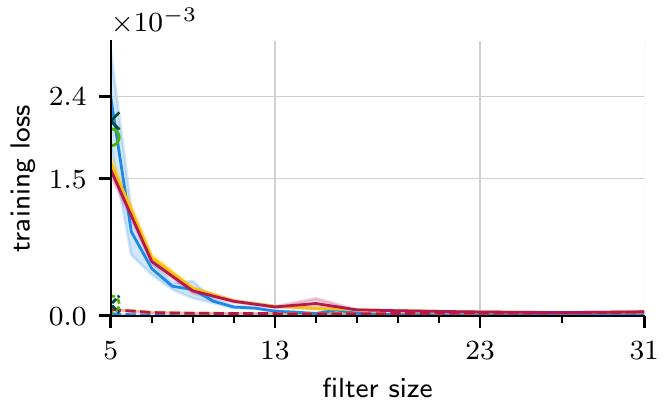}
        \end{subfigure}
        \caption{Filter size}
    \end{subfigure}
    \hfill%
    \begin{subfigure}[r]{\figsixcol}
        \centering
        \begin{subfigure}[t]{\linewidth}
            \centering
            \includegraphics[width=\linewidth]{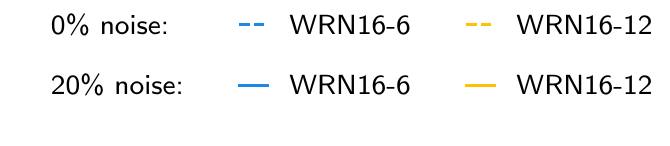}
        \end{subfigure}
        \begin{subfigure}[b]{\linewidth}
            \centering
            \includegraphics[width=\linewidth]{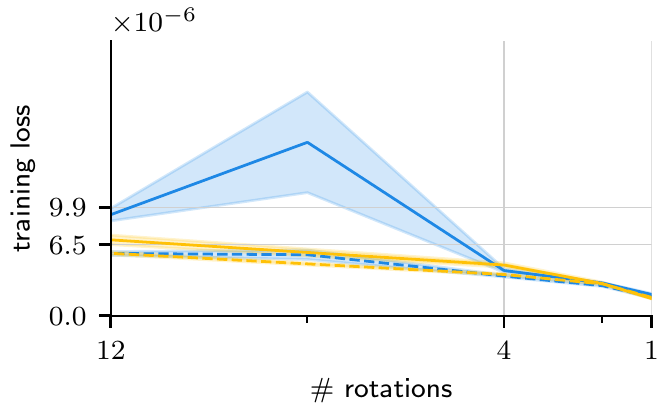}
        \end{subfigure}
        \caption{Rotational invariance}
    \end{subfigure}
    \caption{
        Training losses of the models in this paper's experiments as a function of
        (a) filter size and (b) the number of training set rotations.
        Models with a stronger inductive bias generally exhibit larger losses.
        However, in all instances, the numerical difference is small.
        Lines show the mean loss over $5$ training set samples in (a)
        and $5$ different optimization runs in (b),
        shaded areas the corresponding standard error.
    }
    \label{fig:loss}
\end{figure}

\paragraph*{Increasing model complexity for varying filter size}
As additional evidence,
we repeat the main filter size experiments
from \cref{fig:filters_error} in \cref{ssec:empirical_filters}
using the same setup as before
(see \cref{ssec:experimentdetails_setup_filters}),
but increase the convolutional layer width to $256$, $512$, $1024$, and $2048$.
For computational reasons,
we evaluate a reduced number of filter sizes for widths $256$ and $512$,
and only the smallest filter size $5$ for widths $1024$ and $2048$.
Since we found the original learning rate $0.2$ to be too unstable for the larger model sizes,
we use a decreased peak learning rate
$0.13$ for widths $256$ and $512$,
and $0.1$ for widths $1024$ and $2048$.

\Cref{fig:ablation_filters_error_noisy,fig:ablation_filters_error_noiseless}
show the test errors for $20\%$ and $0\%$ training noise, respectively.
With noisy training data (\cref{fig:ablation_filters_error_noisy}),
larger interpolating models yield a slightly smaller test error,
but the overall trends remain:
the gap in test error between converged and optimally early-stopped models
increases with inductive bias strength,
and the phase transition between harmless and harmful interpolation persists.
In particular, \cref{fig:ablation_filters_error_noisy} shows strong evidence
that the number of model parameters does not influence our phenomenon:
for example, models with filter size $5$ (strong inductive bias)
and width $512$ (red)
have more parameters than models with filter size $27$ (weak inductive bias)
and width $128$ (blue).
Nevertheless, models with filter size $5$ benefit significantly
from early stopping,
while interpolation for models with filter size $27$ is harmless.
In the noiseless case (\cref{fig:ablation_filters_error_noiseless}),
increasing model complexity does neither harm nor improve generalization,
and all models achieve their optimal performance
after interpolating the entire training dataset.
Similarly, \cref{fig:ablation_filters_es} reveals
that the fraction of training noise that optimally early-stopped models fit
stays the same for larger models.
Finally, for a fixed inductive bias strength,
the test errors saturate as model size increases,
making a different trend for models with more than $2048$ filters unlikely.
To increase legibility, we present the numerical results
for the largest two filter sizes in \cref{tab:conv_ablation}.

\begin{figure}[t]
    \centering
    \begin{subfigure}[t]{\figtwelvecol}
        \centering
        \begin{subfigure}[l]{2.\figfourcol}
            \centering
            \includegraphics[width=\linewidth]{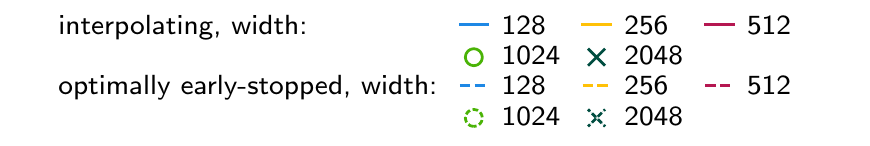}
        \end{subfigure}
            \begin{subfigure}[r]{\figfourcol}
            \centering
            \includegraphics[width=\linewidth]{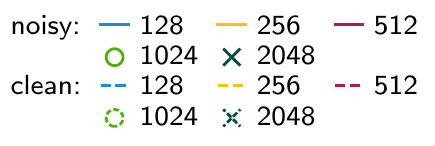}
        \end{subfigure}
    \end{subfigure}
    \begin{subfigure}[b]{\figtwelvecol}
        \centering
        \begin{subfigure}[l]{\figfourcol}
            \centering
            \includegraphics[width=\linewidth]{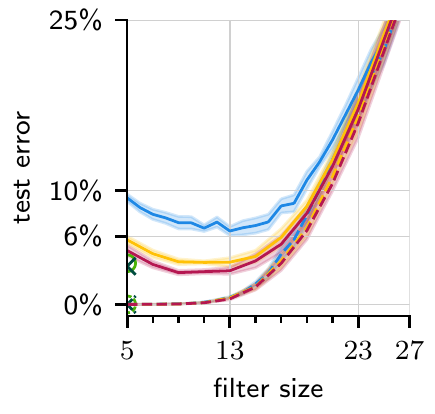}
            \caption{Test error ($20\%$ noise)}
            \label{fig:ablation_filters_error_noisy}
        \end{subfigure}
        \hfill%
        \begin{subfigure}[c]{\figfourcol}
            \centering
            \includegraphics[width=\linewidth]{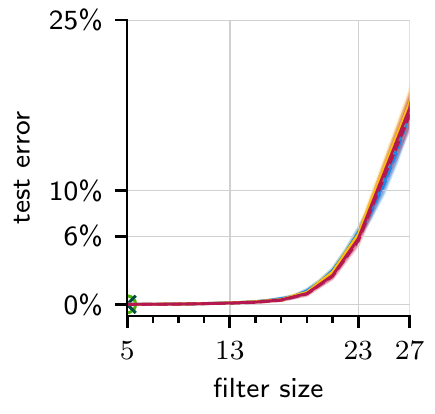}
            \caption{Test error ($0\%$ noise)}
            \label{fig:ablation_filters_error_noiseless}
        \end{subfigure}
        \hfill%
        \begin{subfigure}[r]{\figfourcol}
            \centering
            \includegraphics[width=\linewidth]{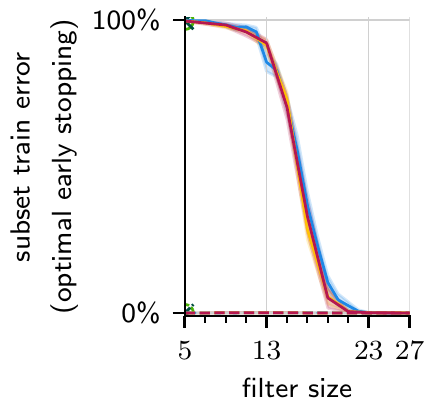}
            \caption{Opt.\ early stopping ($20\%$ noise)}
            \label{fig:ablation_filters_es}
        \end{subfigure}
    \end{subfigure}
    \caption{
        An increase in convolutional layer width by factors $2$ to $16$
        does not significantly alter the behavior of
        (a) the test error when training on $20\%$ label noise,
        (b) the test error when training on $0\%$ label noise,
        and (c) the training error when using optimal early stopping on $20\%$ label noise.
        Despite significantly larger model size,
        the phase transition between harmless and harmful interpolation persists.
        Lines show the mean over five random datasets,
        shaded areas the standard error.
    }
\end{figure}

\begin{table}[ht]
\centering
\caption{
Test errors for filter size $5$ (strongest inductive bias)
and very large width under $20\%$ training noise.
}
\label{tab:conv_ablation}
\begin{tabular}{@{}lll@{}}
\toprule
                         & width $1024$ & width $2048$ \\ \midrule
early-stopped test error & $0.0062\%$   & $0.0049\%$   \\
interpolating test error & $3.6251\%$   & $3.3664\%$   \\
\# parameters            & $289281$     & $578049$     \\ \bottomrule
\end{tabular}
\end{table}

\paragraph*{Increasing model capacity for varying rotational invariance}
For completeness, we also repeat the rotation invariance experiments
from \cref{fig:rotations} in \cref{ssec:empirical_rotations}
with twice as wide \wrn{}s
on a reduced number of rotations.
More precisely, we increase the network widen-factor from $6$ to $12$,
and otherwise use the same setting as the main experiments
(see \cref{ssec:experimentdetails_setup_rotations}).
Note that this corresponds to a parameter increase from
around $6$ million to around $24$ million parameters.

The results in \cref{fig:ablation_rotations} provide additional evidence
that our phenomenon is distinct from double descent:
both the test error
(\cref{fig:ablation_rotations_error_noisy,fig:ablation_rotations_error_noiseless})
and fraction of fitted noise under optimal early stopping
(\cref{fig:ablation_rotations_es})
exhibit the same trend,
despite the significant difference in number of parameters.

\begin{figure}[t]
    \centering
    \begin{subfigure}[t]{\figtwelvecol}
        \centering
        \begin{subfigure}[l]{2.\figfourcol}
            \centering
            \includegraphics[width=\linewidth]{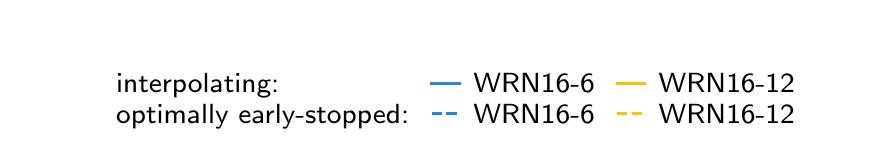}
        \end{subfigure}
            \begin{subfigure}[r]{\figfourcol}
            \centering
            \includegraphics[width=\linewidth]{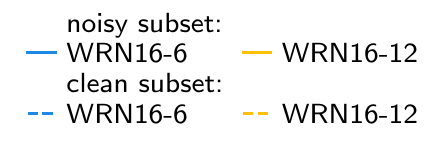}
        \end{subfigure}
    \end{subfigure}
    \begin{subfigure}[b]{\figtwelvecol}
        \centering
        \begin{subfigure}[l]{\figfourcol}
            \centering
            \includegraphics[width=\linewidth]{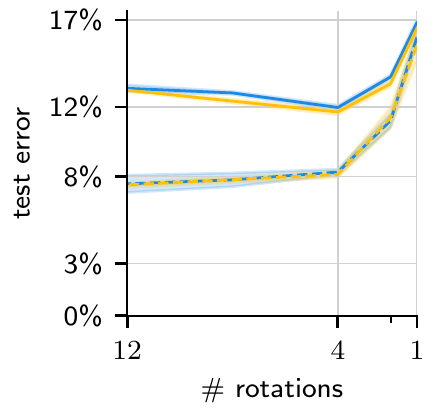}
            \caption{Test error ($20\%$ noise)}
            \label{fig:ablation_rotations_error_noisy}
        \end{subfigure}
        \hfill%
        \begin{subfigure}[c]{\figfourcol}
            \centering
            \includegraphics[width=\linewidth]{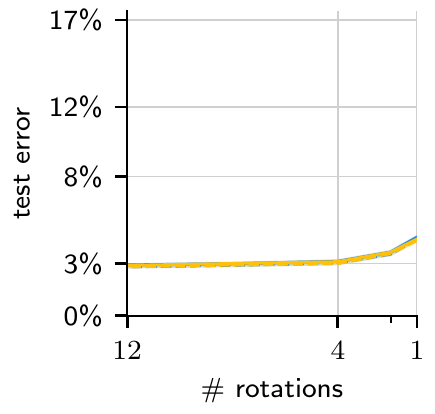}
            \caption{Test error ($0\%$ noise)}
            \label{fig:ablation_rotations_error_noiseless}
        \end{subfigure}
        \hfill%
        \begin{subfigure}[r]{\figfourcol}
            \centering
            \includegraphics[width=\linewidth]{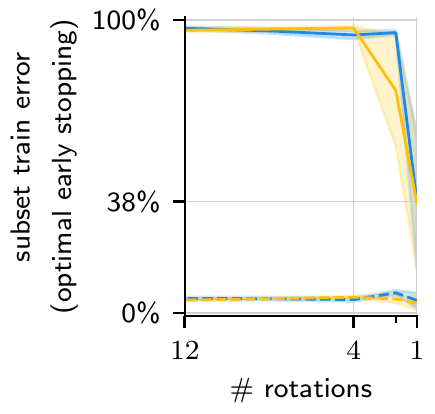}
            \caption{Opt.\ early stopping ($20\%$ noise)}
            \label{fig:ablation_rotations_es}
        \end{subfigure}
    \end{subfigure}
    \caption{
        Doubling \wrn{} width
        does not significantly alter the behavior of
        (a) the test error when training on $20\%$ label noise,
        (b) the test error when training on $0\%$ label noise,
        and (c) the training error when using optimal early stopping on $20\%$ label noise.
        Despite significantly larger model size,
        the phase transition between harmless and harmful interpolation persists.
        Lines show the mean over five random network initializations,
        shaded areas the standard error.
    }
    \label{fig:ablation_rotations}
\end{figure}